\documentclass{article}
\usepackage{PRIMEarxiv}

\usepackage[utf8]{inputenc} % allow utf-8 input
\usepackage[T1]{fontenc}    % use 8-bit T1 fonts
\usepackage{hyperref}       % hyperlinks
\usepackage{nicefrac}       % compact symbols for 1/2, etc.
\usepackage{microtype}      % microtypography
\usepackage{lipsum}
\usepackage{fancyhdr}       % header
\usepackage{graphicx}       % graphics
\graphicspath{{media/}}     % organize your images and other figures under media/ folder

\usepackage{times}  % DO NOT CHANGE THIS
\usepackage{helvet}  % DO NOT CHANGE THIS
\usepackage{courier}  % DO NOT CHANGE THIS
\usepackage{natbib}  % DO NOT CHANGE THIS AND DO NOT ADD ANY OPTIONS TO IT
\usepackage{caption} % DO NOT CHANGE THIS AND DO NOT ADD ANY OPTIONS TO IT
\usepackage{algorithm}
\usepackage{algpseudocode}

\usepackage{newfloat}
\usepackage{listings}
\DeclareCaptionStyle{ruled}{labelfont=normalfont,labelsep=colon,strut=off} % DO NOT CHANGE THIS
\floatstyle{ruled}
\newfloat{listing}{tb}{lst}{}
\floatname{listing}{Listing}
\usepackage{booktabs}       % professional-quality tables
\usepackage{amsfonts}       % blackboard math symbols
\usepackage{nicefrac}       % compact symbols for 1/2, etc.
\usepackage{microtype}      % microtypography
\usepackage{rotating}
\usepackage{tabularx}
\usepackage{pdflscape}
\usepackage{amsmath,amsthm,amssymb, bm}
\usepackage[capitalize,noabbrev]{cleveref}
\usepackage{blkarray}  

\usepackage{amsmath}

\usepackage{subcaption} 
\usepackage{dsfont}
\usepackage{float}
\usepackage{cancel}
\usepackage{enumerate, cases}
\usepackage{thmtools,thm-restate}
\usepackage{mathtools}
\usepackage[textsize=tiny]{todonotes}
\usepackage[none]{hyphenat}
\usepackage{multirow}
\usepackage{makecell}

\usepackage{dblfloatfix}

\usepackage{array}
\usepackage{enumitem}
\usepackage{tikz}
\usetikzlibrary{decorations.pathreplacing}
\usepackage{pgfkeys}
\usetikzlibrary{intersections}
\usepackage[textsize=tiny]{todonotes}
\usepackage{etoolbox} % For better programming logic
\usepackage{pdfpages}

\DeclareMathOperator*{\argmin}{arg\,min}
\DeclareMathOperator*{\argmax}{arg\,max}
\DeclarePairedDelimiter\abs{\lvert}{\rvert}%
\makeatletter
\let\oldabs\abs
\def\abs{\@ifstar{\oldabs}{\oldabs*}}

\usepackage{pifont}
\theoremstyle{plain}
\newtheorem{thm}{Theorem}
\newtheorem{lem}{Lemma}

\newtheorem{cor}{Corollary}

\newtheorem{rem}{Remark}
\newtheorem{defi}{Definition}
\newtheorem{assu}{Assumption}

\newcounter{keylemma}

\newtheorem{keylem}[keylemma]{Lemma}

\newcounter{auxlemma}

\newtheorem{auxlem}[auxlemma]{Lemma}

\newcounter{techlemma}

\newtheorem{techlem}[techlemma]{Lemma}

\usepackage{newfloat}
\usepackage{listings}
\DeclareCaptionStyle{ruled}{labelfont=normalfont,labelsep=colon,strut=off} % DO NOT CHANGE THIS
\floatstyle{ruled}
\newfloat{listing}{tb}{lst}{}
\floatname{listing}{Listing}
\newcommand{\Algo}{{Robust Value Iteration with $f$-Divergence Uncertainty Set (RVI-$f$)}}
\newcommand{\Algoname}{RVI-$f$}
\newcommand{\Algonamechi}{RVI-$\chi^2$}
\newcommand{\AlgonameKL}{RVI-\text{KL}}

\newcommand{\RMDPchi}{RMDP-$\chi^2$}
\newcommand{\RMDPKL}{RMDP-\text{KL}}

\title{ORVIT: Near-Optimal Online Distributionally Robust Reinforcement Learning}

\author{Debamita Ghosh$^1$, George K. Atia$^{1,2}$, Yue Wang$^{1,2}$ \\
$^1$ Department of Electrical and Computer Engineering\\
$^2$ Department of Computer Science\\
University of Central Florida\\
Orlando, FL 32816, USA 
}

\begin{document}
\maketitle

\begin{abstract}
\label{sec:Abstract_Arxiv}
% Reinforcement learning (RL) faces significant challenges in real-world deployments due to the sim-to-real gap, where policies trained in simulators often underperform in practice due to mismatches between training and deployment conditions. Distributionally robust RL addresses this issue by optimizing worst-case performance over an uncertainty set of environments and providing an optimized lower bound on deployment performance. However, existing studies typically assume access to either a generative model or offline datasets with broad coverage of the deployment environment—assumptions that limit their practicality in unknown environments without prior knowledge. In this work, we study the more realistic and challenging setting of online distributionally robust RL, where the agent interacts only with a single unknown training environment while aiming to optimize its worst-case performance. We focus on general $f$-divergence-based uncertainty sets, including $\chi^2$ and KL divergence balls, and propose a  computationally efficient algorithm with sublinear regret guarantees under minimal assumptions. Furthermore, we establish a minimax lower bound on regret of online learning, demonstrating the near-optimality of our approach. Extensive experiments across diverse environments further confirm the robustness and efficiency of our algorithm, validating our theoretical findings.
%%%%%%%%%%%%%%%%%%%%%%%%%%%%%%

We investigate reinforcement learning (RL) in the presence of distributional mismatch between training and deployment, where policies trained in simulators often underperform in practice due to mismatches
between training and deployment conditions, and thereby reliable guarantees on real-world performance are essential. Distributionally robust RL addresses this issue by
optimizing worst-case performance over an uncertainty set of environments and providing an optimized lower bound on deployment performance. However, existing studies typically assume
access to either a generative model or offline datasets with broad coverage of the deployment
environment—assumptions that limit their practicality in unknown environments without prior knowledge. In this work, we study a more practical and challenging setting: online distributionally robust RL, where the agent interacts only with a single unknown training environment while seeking policies that are robust with respect to an uncertainty set around this nominal model. We consider general $f$-divergence-based ambiguity sets, including $\chi^2$ and KL divergence balls, and design a computationally efficient algorithm that achieves sublinear regret for the robust control objective under minimal assumptions, without requiring generative or offline data access. Moreover, we establish a corresponding minimax lower bound on the regret of any online algorithm, demonstrating the near-optimality of our method. Experiments across diverse environments with model misspecification show that our approach consistently improves worst-case performance and aligns with the theoretical guarantees.
\end{abstract}

\section{Introduction}
\label{sec:Introduction_Arxiv}

Reinforcement learning (RL) has emerged as a powerful paradigm for sequential decision-making, with notable successes in simulation-driven domains such as video games \cite{silver2016mastering,zha2021douzero,berner2019dota,vinyals2017starcraft} and large-scale generative modeling \cite{NeurIPS2022_TrainingLLM_Ouyang,cao2023reinforcement,black2023training,uehara2024understanding,zhang2024large,du2023guiding,cao2024survey}. However, deploying RL in safety-critical applications such as autonomous driving \cite{IEEE2021_DRLSurvey_Kiran} and healthcare \cite{ACM2018_SupervisedRLRNN_Wang} remains challenging. The major challenge is due to the infeasibility of direct real-world training under these environments, where RL agents are instead generally trained in simulation and then deployed in the real world. However, real-world deployments are often susceptible to environment uncertainties from, e.g.,  unpredictable noise, unmodeled perturbations, and even adversarial attacks, which cannot be fully captured by the training environments \cite{padakandla2020reinforcement, rajeswaran2017epopt, james2017transferring}, resulting in the {\it sim-to-real gap} and severe performance degradation in practical deployments \cite{Sage2013_RLSurvey_Kober, Arxiv2016_CAD2RL_Sadeghi, IEEE2018_SimToRealTransfer_Peng, IEEE2020_SimToRealTransferDRLRobotics_Zhao}. 

Distributionally robust RL (DRRL) \cite{INFORM2005_RobusDP_Iyengar, PMLR2017_RobustAdvRL_Pinto, Arxiv2022_SimToRealTransferContinuousPartialObs_Hu} provides a principled framework for bridging the sim-to-real gap by training policies that remain reliable under environment shifts. Instead of optimizing performance solely under a nominal model, DRRL specifies an uncertainty set of plausible environments centered
around the training model and then optimizes for the worst-case scenario using a pessimistic objective. When the true environment lies within the prescribed ambiguity set, this leads to a guaranteed lower bound on deployment performance and enhances generalization and robustness to unforeseen conditions \cite{goodfellow2014explaining, vinitsky2020robust, abdullah2019wasserstein, hou2020robust, rajeswaran2017epopt, atkeson2003nonparametric, morimoto2005robust, huang2017adversarial, kos2017delving, lin2017tactics, pattanaik2018robust, mandlekar2017adversarially}.

Despite its extensive studies, most existing methods on DRRL cannot be directly adapted to real-world applications due to their restrictive assumptions on data collection. Most studies either rely on a  {\it generative model } of the training environment \cite{PMLR2022_SampleComplexityRORL_Panaganti,PMLR2023_ImprovedSampleComplexityDRRL_Xu, NeuRIPS2023_CuriousPriceDRRLGenerativeMdel_Shi} that can freely generate data, or assume {\it offline datasets} that comprehensively cover the unknown optimal policy  \cite{NeuRIPS2023_DoublePessimismDROfflineRL_Blanchet, JMLR2024_DROfflineRLNearOptimalSampelComplexity_Shi, Arxiv2024_OfflineRobustRLFDivg_Tang,NeuRIPS2024_UnifiedPessimismOfflineRL_Yue,NeurIPS2024_MinimaxOptimalOfflineRL_Liu,NeurIPS2022_RobustRLOffline_Panaganti,Arxiv2024_SampleComplexityOfflineLinearDRMDP_Wang}, which generally cannot be guaranteed in practice, as the practical environments are unknown and data is often sparse, and the agent needs to explore the environments by itself. 

These considerations have recently motivated the study of DRRL with \emph{online interactions} \cite{Arxiv2024_UpperLowerDRRL_Liu, Arxiv2024_DRORLwithInteractiveData_Lu, PMLR2024_LinearFunctionDRMDP_Liu, ICML2025_OnlineDRMDPSampleComplexity_He}, where the agent strategically explores the unknown training environment and optimizes for the worst-case performance. A key difficulty in this setting is the inherent mismatch between (i) the distribution that generates the training data (the nominal model) and (ii) the worst-case distributions used to evaluate robustness. This induces  online DRRL as an \emph{off-target learning} problem \cite{Arxiv2020_OffDynRL_Eysenbach, NeurIPS2024_MinimaxOptimalOfflineRL_Liu, holla2021off}, leading to a major challenge known
as the information deficit: the states covered by the worst-case or deployment environment may not be visited during
training, yet the agent must still act reliably in these unfamiliar states \cite{ICML2025_OnlineDRMDPSampleComplexity_He}. This lack of exposure can significantly increase the difficulties of online learning.

To cope with this challenge, existing online DRRL approaches typically impose additional structural or coverage conditions. Some works assume special “fail-state” structures or monotone degradation patterns that constrain how the worst-case model can differ from the nominal one \cite{Arxiv2024_UpperLowerDRRL_Liu, PMLR2024_LinearFunctionDRMDP_Liu, Arxiv2024_DRORLwithInteractiveData_Lu}, effectively rendering adversarial shifts more predictable and alleviating the information deficit. More recent results \cite{ICML2025_OnlineDRMDPSampleComplexity_He} relax such structural assumptions but still rely on nontrivial coverage-type conditions that ensure sufficient exploration of critical regions. However, in more realistic settings where deployment dynamics are unknown, these assumptions are impractical or hard to justify, which limits the practical applicability of existing theory. This raises the central question that motivates our work:
\begin{center}
\textbf{\textit{Can we develop an online distributionally robust RL algorithm with no structural assumptions and near-optimal
sample-complexity?}}
\end{center}

\subsection{Contributions}
In this paper, we answer this question by designing algorithms for online DRRL with near-optimal sample complexity, without relying on structural assumptions such as fail-states. Our major contributions are summarized as follows. 
\begin{itemize}
    \item We propose an optimistic model-based meta-algorithm for online DRRL, named  {\Algoname}, with implementations under two important uncertainty set structures: {\Algonamechi} and {\AlgonameKL}. These algorithms use plug-in estimates of the nominal kernel and introduce data-driven penalty terms tailored to the respective uncertainty sets. Notably, our algorithm enjoys better simplicity and more efficient implementation, compared to the ones in \cite{ICML2025_OnlineDRMDPSampleComplexity_He}, which requires an additional optimization oracle. 

    \item We prove that {\Algonamechi} and {\AlgonameKL} are data-efficient: they achieve an $\varepsilon$-optimal robust policy with $\tilde{\mathcal{O}}\big(H^5(1+\sigma)SA/\varepsilon^2\big)$ and $\tilde{\mathcal{O}}\big(H^5\exp(2H^2)SA/(P^\star_{\min}\varepsilon^2)\big)$ samples, respectively. Here, $P^\star_{\min}$ denotes the minimum positive entry of the nominal transition probability induced by the optimal robust policy. Our results do not require any additional assumptions, and enjoy a better complexity compared to previous work \cite{ICML2025_OnlineDRMDPSampleComplexity_He} (Refer to Table \ref{tab:comparison_theoretical_guarantees} for details and Section \ref{sec:Problem_setup_Arxiv} for relevant notation).

    \item We further develop hard instances to derive the minimax lower bound for online DRRL, highlighting the fundamental difficulty of the problem. Our results show that for any online algorithm, there exists a hard instance requiring a sample complexity of at least $\Omega\big(H^5(1+\sigma)SA/\varepsilon^2\big)$ and $\Omega\big(H^5SA/(P^\star_{\min}\varepsilon^2)\big)$ to find an $\varepsilon$-optimal policy under {\RMDPchi} and {\RMDPKL}. This demonstrates the near-optimality of our algorithms, which is minimax-optimal up to logarithmic factors under $\chi^2$ set, and matches the lower bound up to the dependency on $H$ under the KL set.

    \item We validate our methods through extensive experiments on the Gambler’s problem and Frozen Lake, demonstrating strong performance under significant distribution shifts and supporting our theoretical findings. 
\end{itemize}

\section{Related Work}
\label{app:related_work}

We provide a detailed discussion on other related works in this section.

\textbf{Distributionally Robust Reinforcement Learning (DRRL).}
The study of distributionally robust MDPs (RMDPs) originates from
\cite{INFORM2005_RobusDP_Iyengar, INFORMS2005_RobustMDPUncertaintyTransitionMatrix_Nilim, NeurIPS2006_RobustnessMDP_Xu, INFORM2013_RMDP_Wiesemann, INFORMS2016_RMDPkRectangualr_Mannor},
where the uncertainty set is known
and dynamic programming based planning methods are studied. Subsequent research incorporates learning, where the nominal model is unknown and must be inferred from data. Broadly, two data-access models have been investigated:
(i) the \emph{generative model} setting, where the learner may sample any state-action pair under the nominal kernel
\cite{AnnalsStat2022_TheoreticalUnderstandingRMDP_Yang, PMLR2022_SampleComplexityRORL_Panaganti, NeuRIPS2023_CuriousPriceDRRLGenerativeMdel_Shi, INFORMS2023_DRBatchContxBandits_Si, PMLR2023_SampleComplexityDRQLearning_Wang, PMLR2023_ImprovedSampleComplexityDRRL_Xu, JMLR2024_SampleComplexityVarianceReducedDRQLearning_Wang,Arxiv2023_TowardsMinimaxOptimalityRobustRL_Clavier, NeurIPS2021_OnlineRobustRLModelUncertainty_Wang,wang2022policy,wang2023model,wang2024modelfree},
and (ii) the \emph{offline} setting, where one is given a fixed dataset collected under (possibly unknown) behavior policies
\cite{PMLR2021_DROTabularRL_Zhou, JMLR2024_DROfflineRLNearOptimalSampelComplexity_Shi, NeurIPS2022_RobustRLOffline_Panaganti, NeuRIPS2023_DoublePessimismDROfflineRL_Blanchet, NeurIPS2024_MinimaxOptimalOfflineRL_Liu, Arxiv2025_LinearMixtureDRMDP_Liu, Arxiv2024_SampleComplexityOfflineLinearDRMDP_Wang, Arxiv2022_OfflineDRRLLinearFunctionApprox_Ma, NeuRIPS2024_UnifiedPessimismOfflineRL_Yue,zhang2025modelfree}.
Both lines provide sample-complexity guarantees under their respective access models, but they do not address the purely online interactive regime studied in this paper.

More recently, robust RL has also been explored through regularized RMDPs (RRMDPs), also known as penalized
or soft robust MDPs
\cite{Arxiv2023_RMDPModelEstimation_Yang, Arxiv2023_SoftRMDPRiskSensitive_Zhang, Arxiv2024_ModelFreeRobustRL_Panaganti, Arxiv2024_OfflineRobustRLFDivg_Tang},
where robustness is enforced through a divergence-based penalty rather than
hard uncertainty constraints. The algorithm design and analysis therein, however, are distinct from ours.

DRRL is also related to risk-measure RL, such as entropic value-at-risk and Orlicz-heart risk measures
\cite{ahmadi2012entropic, cheridito2009risk}, and risk-sensitive RL, e.g., \cite{osogami2015robust,artzner1999coherent,rockafellar2000optimization,deng2025near,godbout2021acrel,
chow2015risk,
tamar2015optimizing,
chow2018risk,
tamar2015optimizing,
kashima2007risk}. However, most existing risk-sensitive RL formulations related to entropic risk or alternative risk criteria
(e.g., entropic-VaR-based RL and Gini-deviation-based risk-averse RL
\cite{ni2022risk, luo2023alternative})
either assume known or generative models, or focus on asymptotic or static performance, and do not provide online, instance-level regret guarantees.

\textbf{Online Distributionally Robust Reinforcement Learning (Online DRRL).}
Online robust RL, where the agent learns exclusively from its own interaction with an unknown environment, has been comparatively less explored than offline or generative-model settings. Prior works such as
\cite{NeurIPS2021_OnlineRobustRLModelUncertainty_Wang, PMLR2021_RobustRLLeastSquaresPolicyIteration_Badrinath}
study infinite-horizon RMDPs with model uncertainty and design algorithms under $R$-contamination or general uncertainty sets, but their guarantees typically rely on assuming access to sufficiently exploratory behavior, effectively presupposing good coverage. This assumption is difficult to justify in many real-world applications.

Several recent papers move closer to the purely interactive robust setting considered here.
For example,
\cite{PMLR2024_LinearFunctionDRMDP_Liu, Arxiv2024_UpperLowerDRRL_Liu, Arxiv2024_DRORLwithInteractiveData_Lu}
analyze TV-based RMDPs and obtain upper and/or lower bounds under additional structural conditions (such as fail-states or vanishing minimal values) that constrain how adversarial shifts can occur and thereby alleviate the exploration challenge.
The work of
\cite{ICML2025_OnlineDRMDPSampleComplexity_He}
dispenses with such structural assumptions but instead imposes a coverage-type condition (the supremal visitation ratio) to ensure that critical regions are sufficiently explored.
In contrast to these results, we consider finite-horizon RMDPs with general $(s,a)$-rectangular $f$-divergence ambiguity sets, including $\chi^2$ and KL, and we design algorithms that achieve near-optimal regret without relying on fail-state structures, vanishing minimal values, or explicit visitation-ratio assumptions.
Our contributions are therefore complementary: we address a different class of uncertainty sets and remove structural conditions that are central to the guarantees in \cite{Arxiv2024_DRORLwithInteractiveData_Lu} and related work.

\textbf{Connections to Non-Robust Online RL.}
Our work also relates to the extensive literature on non-robust online RL in tabular MDPs
\cite{PMLR2017_MinimxRegretBoundNonRobustRL_Azar, NeurIPS2017_UnifyingPACEpisodicRL_Dann, NeuRIPS2018_QLearningEfficient_Jin, PMLR2019_TighterProblemDependentRegretRL_Zanette, NeurIPS2020_AlmostOptimalModelFreeRL_Zhang, PMLR2021_UCBMomentumQLearning_Menard, PMLR2022_NearOptimalPolicyStableTImeGuarantee_Wu, INFORMS2024_QlearningTIghtSampleComplexity_Li, PMLR2021_RLDifficultThanBandits_Zhang},
which establishes minimax-optimal regret rates (e.g., $\widetilde{O}(\sqrt{H^3SAK})$ for UCB-VI \cite{PMLR2017_MinimxRegretBoundNonRobustRL_Azar}).
A standard MDP can be viewed as a special case of an RMDP where the uncertainty radius is zero; in this sense, robust algorithms should recover classical guarantees when robustness is inactive.
Recent advances further extend online RL to linear and more general function approximation
\cite{PMLR2020_EfficientRLLinearFuncApprox_Jin, PMLR2021_MinimaxOptimalLenarMDP_Zhou, NeurIPS2021_BellmanEluderDim_Jin, Arxiv2022_GECUnifiedFrameworkMDPPOMDPBeyond_Zhong, CoRR2023_MaxObjFusingEst_Liu, PMLR2024_HorizonFreeInstanceDependentRegretBoundFuncApprox_Huang}.
Our results complement this line of work by showing how the robust objective and $f$-divergence uncertainty sets fundamentally modify the exploration requirements and attainable regret bounds, even in the tabular setting.

\textbf{Distinction from Corruption-Robust RL.}
Finally, we distinguish our setting from corruption-robust RL, which studies robustness when the \emph{training data} itself may be corrupted or adversarially perturbed
\cite{PMLR2021_CorruptedEpisodicRL_Lykouris, PMLR2022_ModelSelectionApproachCorruptedRobustRL_Wei, PMLR2022_CorruptedRobustOfflineRL_Zhang, yArxiv2024_RobustRLAdvCorruption_Ye, PMLR2023_CorruptedRobustRLContextualBandits_Ye, NeurIPS2023_CorruptedRobustOfflineRLFuncApprox_Ye}.
DRRL instead aims to guarantee performance under \emph{test-time} shifts in the environment dynamics.
While both paradigms pursue reliable decision-making under uncertainty, they address different failure modes and are not directly interchangeable, and our analysis is confined to the latter.

\section{Preliminaries and Problem Formulation}
\label{sec:Problem_setup_Arxiv}

\subsection{Distributionally Robust Markov Decision Process (RMDPs).}
We begin by recalling the standard episodic finite-horizon RMDP formulation \cite{INFORM2005_RobusDP_Iyengar}.
An RMDP is given by $(\mathcal{S},\mathcal{A},H, \mathcal{P}, r)$, where $\mathcal{S}=\{1,\dots,S\}$ is a finite state space, $\mathcal{A}=\{1,\dots,A\}$ is a finite action space, $H$ is the horizon length, and $r = \{r_h\}_{h=1}^H$ is a collection of reward functions with $r_h: \mathcal{S}\times \mathcal{A}\rightarrow [0, 1]$. The set $\mathcal{P}=\{\mathcal{P}_h\}_{h=1}^H$ denotes an uncertainty set of transition kernels. At step $h$, the agent at state $s_h$ selects an action $a_h$, receives reward $r_h(s_h,a_h)$, and transited to next state $s_{h+1}$ following an arbitrary transition kernel kernel $P_h \in \mathcal{P}_h$.

We consider the standard $(s,a)$-rectangular uncertainty model with a divergence-ball structure \cite{INFORM2013_RMDP_Wiesemann}. Specifically, let $P^{\star} = \{P^{\star}_h\}_{h=1}^H$ be a \emph{nominal} transition kernel with $P^\star_h(\cdot|s,a)\in\Delta(\mathcal{S})$. The ambiguity set centered at $P^\star$ is defined as
\[
\mathcal{P}=\mathcal{U}^{\sigma}(P^{\star})
= \bigotimes_{(h,s,a)\in [H]\times \mathcal{S} \times \mathcal{A}}\mathcal{U}^{\sigma}_h(s,a),
\]
where
\[
\mathcal{U}^{\sigma}_h(s,a)
=\big\{P\in\Delta(\mathcal{S}): D(P,P^\star_h(\cdot|s,a))\leq \sigma\big\}
\]
collects all transition kernels within radius $\sigma\ge 0$ under a chosen probability divergence functions
\cite{INFORM2005_RobusDP_Iyengar, PMLR2022_SampleComplexityRORL_Panaganti, AnnalsStat2022_TheoreticalUnderstandingRMDP_Yang}.
In this work, we focus on the important class of $f$-divergence-based ambiguity sets, recalled below (see also Section~\ref{app:RMDP_f_Divg} for further discussion).

\begin{defi}[$f$-Divergence Uncertainty Set]
\label{def:f_divergence_uncertainty}
For each $(s,a)$ and step $h$, the uncertainty set is
\begin{align}
\label{eq:f_divg_Uncertainty}
\mathcal{U}^{\sigma}_h(s,a)
= \left\{ P \in \Delta(\mathcal{S}): D_f\Big(P,P^{\star}_h(\cdot|s,a)\Big)\leq \sigma \right\},
\end{align}
where
\[
D_f\big(P,P^{\star}_h(\cdot|s,a)\big)
= \sum_{s' \in \mathcal{S}} f\!\left( \frac{ P(s')}{P^{\star}_h(s'|s,a)} \right)
P^{\star}_h(s'|s,a)
\]
is the $f$-divergence \cite{sason2016f}.
\end{defi}
We refer to an RMDP $(\mathcal{S},\mathcal{A},H, \mathcal{P}, r)$ with $\mathcal{P}$ defined via \Cref{def:f_divergence_uncertainty} as an $f$-RMDP.

\subsection{Policy and Robust Value Function.}
The agent's strategy of taking actions is captured by a Markov policy $\pi := \{\pi_h\}_{h=1}^H$, with $\pi_h: \mathcal{S}\rightarrow \Delta(\mathcal{A})$ for each step $h \in [H]$, where $\pi_h(\cdot|s)$ is the probability of taking actions at the state $s$ in step $h$. In RMDPs, performance is evaluated in a worst-case sense over the ambiguity set. For any policy $\pi$ and step $h\in[H]$, the \emph{robust value function} and \emph{robust state-action value function} are
\begin{align}
    V^{\pi,\sigma}_{h}(s)
    &= \inf_{P \in \mathcal{U}^{\sigma}(P^{\star})}
    \mathbb{E}_{\pi,P}\Bigg[\sum_{t=h}^H r_t(s_t,a_t)\,\Big|\, s_h=s\Bigg] \quad \forall (h,s)\in [H]\times \mathcal{S}, \label{eq:robust_Q_fn}\\
    Q^{\pi,\sigma}_{h}(s,a)
    &= \inf_{P \in \mathcal{U}^{\sigma}(P^{\star})}
    \mathbb{E}_{\pi,P}\Bigg[\sum_{t=h}^H r_t(s_t,a_t)\,\Big|\, s_h=s,a_h=a\Bigg], \quad \forall (h,s,a)\in [H]\times \mathcal{S}\times \mathcal{A}\nonumber
\end{align}
where the expectation is taken over trajectories $\{s_h,a_h,r_h\}_{h=1}^H$ induced by $\pi$ and $P$.

The objective is to find an optimal robust policy
\begin{align}
\label{eq:optimal_policy}
\pi^\star \in \argmax_{\pi\in \Pi} V^{\pi,\sigma}_{1}(s_1),
\end{align}
for a (fixed) initial state $s_1\in\mathcal{S}$, where $\Pi$ is the set of admissible policies. Thus $\pi^\star$ maximizes the worst-case expected return over all transition kernels in $\mathcal{U}^\sigma(P^\star)$.

\subsection{RMDP with $f$-Divergence Uncertainty Set}
\label{app:RMDP_f_Divg}

n this subsection we briefly review the formulation of RMDP with $f$-divergence uncertainty sets. 

\paragraph{Rectangular structure.}
Throughout, we assume an $\mathcal{S}\times\mathcal{A}$-rectangular ambiguity set
\cite{INFORM2005_RobusDP_Iyengar}, meaning that for each $(h,s,a)$ the adversary selects a transition kernel independently from a local set $\mathcal{U}_h^{\sigma}(s,a)$.
This standard assumption decouples uncertainty across state-action pairs and ensures that robust dynamic programming is tractable.
For any kernel $P_h$ and value function $V$, we use the shorthand
\[
[\mathbb{P}_h V](s,a) := \mathbb{E}_{s'\sim P_h(\cdot|s,a)}[V(s')]
\]
when the dependence on $P_h$ is clear.

\paragraph{Dual representation of $f$-divergence balls.}
Given a nominal kernel $P_h^\star(\cdot|s,a)$ and radius $\sigma\ge 0$, the local ambiguity set in \Cref{def:f_divergence_uncertainty} is
\[
\mathcal{U}^{\sigma}_h(s,a)
=
\left\{ P \in \Delta(\mathcal{S}): D_f\big(P,P^{\star}_h(\cdot|s,a)\big)\leq \sigma \right\},
\]
where $D_f$ is the $f$-divergence \cite{sason2016f}.
For such sets, the robust next-state value
\[
\mathbb{E}_{\mathcal{U}^{\sigma}_h(s,a)}[V]
:= \inf_{P \in \mathcal{U}^{\sigma}_h(s,a)} [\mathbb{P}V](s,a)
\]
admits an equivalent dual formulation obtained via convex duality. In particular, combining standard arguments (see, e.g.,
\cite{Book2004_ConvexOpt_Boyd, AnnalsStat2022_TheoreticalUnderstandingRMDP_Yang}), one can show that
\begin{align}
\label{eq:f-div-dual-generic}
\mathbb{E}_{\mathcal{U}^{\sigma}_h(s,a)}[V]
=
\sup_{\lambda \ge 0,\ \eta \in \mathbb{R}}
\left\{
-\lambda\sigma + \eta
-\lambda \sum_{s'\in\mathcal{S}} P_h^\star(s'|s,a)\,
f^\star\!\left(\frac{\eta - V(s')}{\lambda}\right)
\right\},
\end{align}
where $f^\star$ is the convex conjugate of $f$ restricted to $[0,\infty)$. We will use this representation as a unifying template for deriving robust bonuses. This formula as given in \eqref{eq:f-div-dual-generic} follows from standard strong duality arguments; see, e.g., \cite[Lemma B.1]{AnnalsStat2022_TheoreticalUnderstandingRMDP_Yang}.

\paragraph{Special cases: TV, $\chi^2$, and KL}
For concreteness, we recall the resulting one-dimensional variational forms for three choices frequently used in robust RL; detailed derivations can be found in
\cite{AnnalsStat2022_TheoreticalUnderstandingRMDP_Yang, ICML2025_OnlineDRMDPSampleComplexity_He}. Under the $\mathcal{S} \times \mathcal{A}$-rectangularity assumption and \eqref{eq:f-div-dual-generic}, the robust expectation for any $V:\mathcal{S}\to[0,H]$ and $P^{\star}_h$ admits the following equivalent forms:

\begin{itemize}
\item \textbf{TV-divergence} ($f(t)=|t-1|$).
In this case, \eqref{eq:f-div-dual-generic} simplifies to
\begin{align}
\label{eq:dual_TV}
\mathbb{E}_{\mathcal{U}^{\sigma}_h(s,a)}[V]
=
\sup_{\eta \in [0,H]}
\left\{
-\big[\mathbb{P}^\star_h (\eta-V)_{+}\big](s,a)
- \frac{\sigma}{2}\big(\eta - \min_{s'}V(s')\big)_{+}
+ \eta
\right\}.
\end{align}

\item \textbf{$\chi^2$-divergence} ($f(t)=(t-1)^2$).
One obtains a variance-sensitive form:
\begin{align}
\label{eq:dual_chi}
\mathbb{E}_{\mathcal{U}^{\sigma}_h(s,a)}[V]
=
\sup_{\eta \in [0,H]}
\left\{
-\sqrt{\sigma\,\mathrm{Var}_{P_h^\star(\cdot|s,a)}\big((\eta-V)_{+}\big)}
+ \big[\mathbb{P}^\star_h (V-\eta)_{+}\big](s,a)
\right\}.
\end{align}

\item \textbf{KL-divergence} ($f(t)=t\log t$).
The robust expectation can be written as
\begin{align}
\label{eq:dual_KL}
\mathbb{E}_{\mathcal{U}^{\sigma}_h(s,a)}[V]
=
\sup_{\eta \in [\underline{\eta},\,H/\sigma]}
\left\{
-\eta \log\!\Big( \big[\mathbb{P}^\star_h(\exp\{-V/\eta\})\big](s,a) \Big)
- \eta\sigma
\right\},
\end{align}
where $\underline{\eta}>0$ is a regularity bound on the optimal dual variable, as commonly assumed in
\cite{NeuRIPS2023_DoublePessimismDROfflineRL_Blanchet, ICML2025_OnlineDRMDPSampleComplexity_He}.
\end{itemize}

\begin{rem}[Comparison between TV, $\chi^2$, and KL sets]
In TV-based RMDPs, information deficits can be exacerbated when rarely visited states under the nominal dynamics become critical in the worst-case model, motivating additional structural assumptions such as fail-states or vanishing minimal values
\cite{PMLR2024_LinearFunctionDRMDP_Liu, Arxiv2024_DRORLwithInteractiveData_Lu}.
For smoother $f$-divergences such as $\chi^2$ and KL, the impact of individual rare states is moderated by the geometry of the divergence, which motivates our focus on these uncertainty sets in the online setting.
\end{rem}

\subsection{Robust Bellman Equation and the Robust Optimal Policy}
For $(s,a)$-rectangular RMDPs, the robust value functions admit a dynamic programming characterization that parallels the standard Bellman equations; see, e.g., \cite{INFORM2005_RobusDP_Iyengar, INFORMS2005_RobustMDPUncertaintyTransitionMatrix_Nilim, INFORM2013_RMDP_Wiesemann}. We briefly summarize the finite-horizon formulation used in this paper.

Given any policy $\pi = \{\pi_h\}_{h=1}^H$ and ambiguity sets $\{\mathcal{U}^{\sigma}_h(s,a)\}$ as defined above, the robust state-action and value functions satisfy, for each step $h$,
\begin{align}
Q^{\pi,\sigma}_{h}(s,a)
&= r_h(s,a)
+ \mathbb{E}_{\mathcal{U}^{\sigma}_h(s,a)}\!\left[V^{\pi,\sigma}_{h+1}\right],
\label{eq:Robust_bellman_Q_fn}\\
V^{\pi,\sigma}_{h}(s)
&= \mathbb{E}_{a \sim \pi_h(\cdot|s)} \left[ Q^{\pi,\sigma}_{h}(s,a) \right],
\label{eq:Robust_bellman_V_fn}
\end{align}
with terminal condition $V^{\pi,\sigma}_{H+1} \equiv 0$.
Here, $\mathbb{E}_{\mathcal{U}^{\sigma}_h(s,a)}[\,\cdot\,]$ denotes the worst-case expectation over all kernels in the local set $\mathcal{U}^{\sigma}_h(s,a)$, as in Section~\ref{app:RMDP_f_Divg}.

The optimal robust value functions arise as the fixed point of the corresponding max–min recursion.
In particular, for any $(h,s,a)$,
\begin{align}
\label{cor:Robust_Bellman_Optimal_eq}
Q^{\star,\sigma}_{h}(s,a)
&= r_h(s,a)
+ \mathbb{E}_{\mathcal{U}^{\sigma}_h(s,a)}\!\left[V^{\star,\sigma}_{h+1}\right], \nonumber \\
V^{\star,\sigma}_{h}(s)
&= \max_{a \in \mathcal{A}} Q^{\star,\sigma}_{h}(s,a),
\end{align}
again with $V^{\star,\sigma}_{H+1}\equiv 0$.
Any policy $\pi^\star$ that is greedy with respect to $\{Q^{\star,\sigma}_{h}\}$, i.e.,
$\pi_h^\star(\cdot|s) \in \arg\max_{a} Q^{\star,\sigma}_{h}(s,a)$ for all $(h,s)$,
achieves the robust optimal value.
These properties follow from standard robust dynamic programming arguments for rectangular ambiguity sets
\cite{INFORM2005_RobusDP_Iyengar, INFORMS2005_RobustMDPUncertaintyTransitionMatrix_Nilim, NeuRIPS2023_DoublePessimismDROfflineRL_Blanchet}, see, e.g., \cite[Proposition 2.3]{NeuRIPS2023_DoublePessimismDROfflineRL_Blanchet}, 
and form the backbone of our algorithmic and regret analysis.

\subsection{Online Distributionally Robust RL.}
We now formalize the online learning setting of interest.
The learner interacts with the (unknown) nominal environment $P^\star$ over $K \in \mathbb{N}$ episodes. At the beginning of episode $k$, the agent observes an initial state $s^k_1$ and selects a policy $\pi^k$ based on its past observations. Executing $\pi^k$ in $P^\star$ generates a trajectory, which is then used to update the policy for subsequent episodes. Our performance metric is the \emph{cumulative robust regret}
\begin{align}
\label{eq:Regret_K}
\mathrm{Regret}(K)
:= \sum_{k=1}^K \Big[V^{\star,\sigma}_{1}(s^k_1) - V^{\pi^k,\sigma}_{1}(s^k_1)\Big],
\end{align}
which measures the total robust value gap between the optimal robust policy $\pi^\star$ and the learner’s sequence $\{\pi^k\}_{k=1}^K$.
This extends the classical notion of regret in standard MDPs \cite{NeuRIPS2008_NearOptimalRegetBoundRL_Auer} to the distributionally robust objective.

We also consider the \emph{sample complexity} of learning an $\varepsilon$-optimal robust policy.
Let $T = KH$ be the total number of time steps.
We say an algorithm succeeds if it outputs a policy $\widehat{\pi}$ such that
\begin{align}
V^{\star,\sigma}_1(s_1) - V^{\widehat{\pi},\sigma}_1(s_1) \leq \varepsilon.
\end{align}
Our goal is to design algorithms with sublinear robust regret and near-optimal sample complexity under the above online interaction model.

\section{Robust Value Iteration (RVI)}
\label{sec:Algorithm_Arxiv}
\begin{algorithm}[!htb]
\caption{{\Algo}}
\label{algo:ORVI}
\begin{algorithmic}[1]
\State \textbf{Input:} uncertainty level $\sigma>0$.
\State \textbf{Initialize:} Dataset \( \mathbb{D} = \emptyset \)
\For{episode \( k = 1, \dots, K \)}
    \Statex {\color{blue}*** {\bf \textsc{  Nominal Transition Estimation }} ***}
    \State Compute the transition kernel estimator $ \widehat{P}_h^k(s,a,s')$ as given in \eqref{eq:transition_estimate}.
    \Statex {\color{blue}*** {\bf \textsc{  Optimistic Robust Estimation }} ***}
    \State Set \( \overline{V}_{H+1}^k(\cdot) = \underline{V}_{H+1}^k(\cdot) = 0 \)
    \For{step \( h = H, \dots, 1 \)}
        \For{$\forall (s,a)\in \mathcal{S}\times \mathcal{A}$}
            \State Update \( \overline{Q}_h^k(s,a) \) as in \eqref{eq:Upper_estimate_Q}.
            \State Update \( \underline{Q}_h^k(s,a) \) as in \eqref{eq:Lower_estimate_Q}.
        \EndFor
        \For{$\forall s\in \mathcal{S}$}
        \State Update  $\pi_h^k(\cdot)$, $\overline{V}_h^k(\cdot)$ and $\underline{V}_h^k(\cdot)$ by \eqref{eq:policy_value_epsiode_k}.
        \EndFor
    \EndFor
    \Statex {\color{blue} *** {\bf \textsc{  Policy Execution and data collection}} ***}
    \State Receive initial state \( s_1^k \in \mathcal{S} \)
    \For{step \( h = 1, \dots, H \)}
        \State Take action \( a_h^k \sim \pi_h^k(\cdot \mid s_h^k) \), observe reward \( r_h(s_h^k, a_h^k) \) and next state \( s_{h+1}^k \).
    \EndFor
    \State Set \( \mathbb{D} = \mathbb{D} \cup \{(s_h^k, a_h^k, s_{h+1}^k)\}_{h=1}^H \).
\EndFor
\State \textbf{Output:} Randomly (uniformly) return a policy from \( \{ \pi^k \}_{k=1}^K \).
\end{algorithmic}
\end{algorithm}

In this section, we introduce {\it \Algo}, a meta-algorithm for episodic finite-horizon RMDPs with interactive data collection under $f$-divergence uncertainty sets, as defined in \Cref{def:f_divergence_uncertainty}. {\Algoname} is flexible and can be applied to various $f$-divergences, with a focus on the $\chi^2$-divergence and KL-divergence. The algorithm, detailed in Algorithm~\ref{algo:ORVI}, balances exploration and exploitation by building confidence intervals directly around the robust value function, avoiding the complexity of modeling full transition dynamics. By leveraging the structure of the $f$-divergence, it uses adaptive bonuses that reflect both uncertainty and robustness. Inspired by UCB-VI \cite{PMLR2017_MinimxRegretBoundNonRobustRL_Azar}, this leads to tighter confidence bounds, less dependence on state space size, and more efficient learning in uncertain environments.

\subsection{Algorithm Design: {\Algoname}}
Our algorithm follows a value iteration framework and integrates optimistic estimation, to derive an optimistic estimation of the robust value function. In each episode $k$, {\Algoname} proceeds in three stages as follows.
% \begin{itemize}
%     \item \textbf{Stage 1:} Estimation of the transition dynamics in the training environment (Lines 5);  
%     \item \textbf{Stage 2:} Optimistic robust planning based on the estimated transitions (Lines 6--16);  
% \item \textbf{Stage 3:} Execution of the policy and data collection within the environment (Lines 17--24).
% \end{itemize}

\subsubsection{Stage 1: Nominal Transition Estimation (Line 4).}
At the beginning of each episode \( k \in [K] \), we maintain an estimate of the transition kernel \( P^\star \) of the training environment by using the historical data \( \mathbb{D} = \{(s_h^\tau, a_h^\tau, s_{h+1}^\tau)\}_{\tau=1,h=1}^{k-1,H} \) collected from the interaction with the training environment. Specifically, {\Algoname} updates the empirical transition kernel for $(h,s,a,s') \in [H]\times \mathcal{S} \times \mathcal{A} \times \mathcal{S}$, as follows
\begin{align}
\label{eq:transition_estimate}
\widehat{P}_h^k(s'|s,a) &= 
\begin{cases}
    \frac{N_h^k(s,a,s') }{N_h^k(s,a)},  &\text{if } N_h^k(s,a)>0\\
    \frac{1}{|\mathcal{S}|} , &\text{if } N_h^k(s,a)=0,
\end{cases}
\end{align}
where the counts \( N_h^k(s,a,s') \) and \( N_h^k(s,a) \) are calculated on the current dataset \( \mathbb{D} \) by
\begin{align}
\label{eq:counts_on_D}
    N^k_h(s,a,s') &= \sum\limits_{\tau=1}^{k-1}\mathbf{1}\{(s^{\tau}_h, a^{\tau}_h, s^{\tau}_{h+1}) = (s,a,s^{\prime})\}, \qquad \text{ and } \qquad N^k_h(s,a)=\sum\limits_{s^{\prime}\in \mathcal{S}}N^k_h(s,a,s^{\prime}).
\end{align}
Notably, our algorithm follows a model-based approach, as it requires explicit model estimation. While this incurs a large memory cost, we emphasize that DRRL is inherently challenging in the model-free setting: the worst-case expectation is non-linear in the nominal transition kernel, making model-free estimation either biased or extremely sample-inefficient \cite{wang2023model,liu2022distributionally,wang2024modelfree,zhang2025modelfree}. 

\subsubsection{Stage 2: Optimistic Robust Estimation (Lines 5--14).}
Given the empirical transition model $\widehat{P}^k$ obtained, {\Algoname} performs optimistic robust planning to construct the policy $\pi^k$ for episode $k$. Specifically, we construct an optimistic estimation of the robust value function for policy execution. Such an approach, known as the Upper-Confidence-Bound (UCB) method, is shown to be effective in online interactive learning in vanilla RL \cite{PMLR2017_MinimxRegretBoundNonRobustRL_Azar, auer2010ucb, dann2019policy, PMLR2019_TighterProblemDependentRegretRL_Zanette, NeurIPS2020_AlmostOptimalModelFreeRL_Zhang, PMLR2021_RLDifficultThanBandits_Zhang, PMLR2021_UCBMomentumQLearning_Menard, li2021breaking, domingues2021episodic, zhang2024settling}. Specifically, interacting based on an optimistic estimation encourages the agent to explore the less visited state-action pairs.

Toward this goal, we update our estimation as follows, to ensure the estimation is optimistic. 
 At each episode \( k \), {\Algoname} maintains a bonus term to account for the difference between the robust value function of the estimated model $\hat{P}$ and the true robust value function. We then add this term to the estimation, which is constructed following the robust Bellman equation,  to ensure its optimism. Namely, for each \( (h, s, a) \in [H] \times \mathcal{S} \times \mathcal{A} \), we update the estimation as 
\begin{align}
\overline{Q}_h^k(s,a) &= \min \left\{ \overline{R}^k_h(s,a) + B^{f}_{k,h}(s,a),\; H \right\}, \label{eq:Upper_estimate_Q}\\
\underline{Q}_h^k(s,a) &= \max \left\{ \underline{R}^k_h(s,a) - B^{f}_{k,h}(s,a),\; 0 \right\}. \label{eq:Lower_estimate_Q}
\end{align}
Each of these estimates \eqref{eq:Upper_estimate_Q} and \eqref{eq:Lower_estimate_Q} consists of two components: an estimated robust Bellman operator $\overline{R}_h^k(s,a)$ or $\underline{R}_h^k(s,a)$, computed as 
\begin{align}
\overline{R}_h^k(s,a) &= r_h(s,a) + \mathbb{E}_{\widehat{\mathcal{U}}^{\sigma}_h}(s,a)[\overline{V}_{h+1}^k], \label{eq:overestimate_R} \\
\underline{R}_h^k(s,a) &= r_h(s,a) + \mathbb{E}_{\widehat{\mathcal{U}}^{\sigma}_h}(s,a)[\underline{V}_{h+1}^k], \label{eq:underestimate_R}
\end{align}
and a bonus term $B^f_{k,h}(s,a) \geq 0$. We denote $     \mathbb{E}_{\mathcal{U}^{\sigma}(s,a)}[V]:= \inf_{P \in \mathcal{U}^{\sigma}(s,a)}\mathbb{E}_{P}[V]$. The bonus term is constructed (we will discuss the construction later) to ensure the estimation becomes a confidence interval of the true robust value function, i.e., $Q^{\star,\sigma}_h\in [\underline{Q}_h^k(s,a), \overline{Q}_h^k(s,a)]$, with high probability. 

With the optimistic estimation $\overline{Q}_h^k$, we then set the execution policy for the $k$-th episode as the greedy policy with respect to the optimistic Q-estimate:
\begin{align}
\pi_h^k(\cdot \mid s) = \arg\max_{a \in \mathcal{A}}\overline{Q}_h^k(s,a), 
% \nonumber\\
% &\overline{V}_h^k(s) = \max_{a \in \mathcal{A}} \overline{Q}_h^k(s,a), \quad 
% \underline{V}_h^k(s) = \max_{a \in \mathcal{A}} \underline{Q}_h^k(s,a). \label{eq:policy_value_epsiode_k}
\end{align}
and update the robust value function estimation of $V^{\star,\sigma}_{h}$ as 
\begin{align}
&\overline{V}_h^k(s) = \max_{a \in \mathcal{A}} \overline{Q}_h^k(s,a), \quad 
\underline{V}_h^k(s) = \max_{a \in \mathcal{A}} \underline{Q}_h^k(s,a). \label{eq:policy_value_epsiode_k}
\end{align}

We remark that although the lower estimation in \eqref{eq:Lower_estimate_Q} does not affect the choice of the execution policy, it is critical for constructing valid exploration bonus terms and for establishing strong theoretical guarantees, and the algorithm leverages both upper and lower bounds to guide exploration. This strategy—optimistic robust planning—enables structured, uncertainty-aware exploration, effectively balancing the competing objectives of exploration, exploitation, and distributional robustness.

\subsubsection{Stage 3: Policy Execution and Data Collection (Lines 15--21).}
After evaluating the policy $\{\pi^k_h\}_{h=1}^H$ for episode $k$, the learner takes action based on $\pi^k_h$ and observes reward $r_h(s^k_h,a^k_h)$ and next state $s^k_{h+1}$, which gets appended to the historical dataset collected till episode $k-1$.

\subsection{Bonus of {\Algoname} under RMDP}
We then instantiate our meta-algorithm for RMDPs under both $\chi^2$-divergence and KL-divergence by explicitly constructing the corresponding bonus terms and estimation procedures.
\textbf{{\RMDPchi}:} We denote $B^{f}_{k,h}(s,a):= B^{\chi^2}_{k,h}(s,a)$, and is given by
    \begin{align}
\label{eq:Bonus_term_chi}
B^{\chi^2}_{k,h}(s,a) &=  \sqrt{\frac{\sigma c_1L\text{Var}_{\widehat{P}_h^k(\cdot|s,a)}\left[
\left( \frac{ \overline{V}_{h+1}^k + \underline{V}_{h+1}^k }{2} \right)\right]}{N_h^k(s,a)\vee 1}} +\frac{2\sqrt{\sigma}\mathbb{E}_{\widehat{P}_h^k(\cdot|s,a)}\left[\overline{V}_{h+1}^k - \underline{V}_{h+1}^k) \right]}{H} \nonumber\\
&\quad + \frac{c_2 \sqrt{\sigma}H^2 S(2L+1)}{\sqrt{N_h^k(s,a)\vee 1}} + \sqrt{\frac{\sigma}{K}},
\end{align}
where \( L = \log\Big(\frac{S^3 A H^2 K^{3/2}}{\delta}\Big) \), and \( c_1, c_2 > 0 \) are absolute constants. The term \( \delta \) is a pre-selected failure probability.

 \textbf{{\RMDPKL}:} We denote $B^{f}_{k,h}(s,a) := B^{\text{KL}}_{k,h}(s,a)$, and is given by
\begin{align}
\label{eq:Bonus_term_KL}
  B^{\text{KL}}_{k,h}(s,a)&= \frac{2c_fH}{\sigma}\sqrt{\frac{L}{\big(N_h^k(s,a)\vee 1\big)\widehat{P}^k_{\min,h}(s,a)}}  + \sqrt{\frac{1}{K}},
\end{align}
where $\widehat{P}^k_{\min,h}(s,a) = \min\limits_{s'\in \mathcal{S}}\{\widehat{P}^k_{h}(s'|s,a): \widehat{P}^k_{h}(s'|s,a)>0\}$, $L = \log\Big(\frac{S^3AH^2K^{3/2}}{\delta}\Big)$, and  $c_f>0$ is an absolute constant.

Under these constructions, \( \overline{Q}_h^k \) and \( \underline{Q}_h^k \) remain valid confidence bounds ( as shown in Lemma ~\ref{lem:Optimistic_pessimism} and \Cref{lem:Optimistic_pessimism_KL} in Appendix). Importantly, we will also show that the carefully designed bonus \eqref{eq:Bonus_term_chi} and \eqref{eq:Bonus_term_KL} ensure the confidence region is tight, resulting in a near-optimal regret bound of our algorithm. 

% are tailored for RMDPs so that its variance terms remain well controlled across time steps. This plays a key role in achieving sharp regret and sample complexity bounds for Algorithm ~\ref{algo:ORVI}.

\subsection{Theoretical Guarantees}
\label{sec:Theoretical_guarantees}
We then develop theoretical analysis of our algorithm, under both $\chi^2$ and KL divergence uncertainty sets. 

\subsubsection{Regret Bound} We first study the regret bound of our method. For the KL-divergence uncertainty set, we adopt the following standard assumption \cite{AnnalsStat2022_TheoreticalUnderstandingRMDP_Yang,NeuRIPS2023_CuriousPriceDRRLGenerativeMdel_Shi,Arxiv2022_OfflineDRRLLinearFunctionApprox_Ma}, which ensures the regularity of the dual formulation of the distributionally robust optimization over the KL-divergence uncertainty set.
\begin{assu}
\label{ass:KL_P_min}
    We assume there exists a constant
$P^{\star}_{\min}>0$, such that for any $(h,s,a,s') \in [H]\times \mathcal{S}\times \mathcal{A} \times \mathcal{S}$, if $P^\star_h(s'|s,a)>0$, then $P^\star_h(s'|s,a)>P^\star_{\min}$. 
%, where $P^\star_{\min}$ only depends on the state-action pairs covered by the optimal robust policy $\pi^\star$ under the nominal model $P^\star$ and is defined as $P^\star_{\min}:= \min\limits_{(h,s,s')\in [H]\times \mathcal{S}\times \mathcal{S}}\Big\{P^\star_h\big(s'|s,\pi^\star_h(s)\big):P^\star_h\big(s'|s,\pi^\star_h(s)\big)>0\Big\}$.
\end{assu}
We then present the regret bound of our algorithm. 
\begin{thm}[Regret Bound of {\Algoname}]
\label{thm:Regret_f_bound}
Consider the $\chi^2$ and KL divergence uncertainty sets. For any $\delta\in (0,1)$ and uncertainty radius $\sigma>0$, with probability at least $1-\delta$, the regret of our {\Algoname} algorithm with corresponding bonus term as  \eqref{eq:Bonus_term_chi} and \eqref{eq:Bonus_term_KL} can be bounded as:
\begin{itemize}
    \item For $\chi^2$ divergence uncertainty set, \begin{align}
        \text{Regret}(K)= \tilde{\mathcal{O}} \left(\sqrt{H^4(1+\sigma)SAK} \right);
    \end{align} 
    \item For KL divergence uncertainty set, under \Cref{ass:KL_P_min}, \begin{align}
        \text{Regret}(K)= \tilde{\mathcal{O}} \left(\sqrt{\frac{H^4\exp\big(2H^2\big)SAK}{P^{\star}_{\min}\sigma^2}} \right).
    \end{align} 
\end{itemize}
where $f(K)=\tilde{\mathcal{O}}(g(K))$ means $f(K)\leq c\cdot g(K)\cdot\textbf{Poly}(\log(K))$ for some constant $c$ and some polynomial of $\log(K)$.
% and some polynomial of $\log(K)$. 
\end{thm}
%     \begin{align}
%     \label{eq:Regret_f_bound}
%     \text{Regret}(K)
% &=  
% \begin{cases}
%     \tilde{\mathcal{O}} \left(\sqrt{H^4(1+\sigma)SAK} \right), & \text{$\chi^2$} \\
% \tilde{\mathcal{O}} \left(\sqrt{\frac{H^4SAK}{P^{\star}_{\min}\sigma^2}} \right).    & \text{KL}
% \end{cases}
% \end{align}

Our results imply that our algorithm achieves a sublinear regret of $\tilde{\mathcal{O}}(\sqrt{K})$ in both uncertainty sets, ensuring efficient robust policy learning from interactive data. Notably, the exponential term in the KL case is standard due to the complicated structure of the uncertainty set, e.g.,  \cite{NeuRIPS2023_DoublePessimismDROfflineRL_Blanchet,PMLR2022_SampleComplexityRORL_Panaganti}.

\subsubsection{Sample Complexity}
As a direct consequence, we derive the sample complexity to learn an $\varepsilon$-optimal policy through {\Algonamechi} and {\AlgonameKL}. Using a standard online-to-batch conversion \cite{NeuRIPS2001_OnlineLearningAlgo_Cesa}, we have the following results. 
\begin{cor}[Sample Complexity of {\Algoname}]
\label{cor:Sample_Complexity_bound_chi}
Under the same setup in \Cref{thm:Regret_f_bound}, with probability at least $1 - \delta$, {\Algoname} obtains an $\varepsilon$-optimal policy with
\begin{align}
\label{eq:Sample_Complexity_bound}
T = KH = 
\begin{cases}
\tilde{\mathcal{O}}\left(\frac{H^5(1+\sigma)SA}{\varepsilon^2}\right), & \text{{\RMDPchi}} \\
\tilde{\mathcal{O}}\left(\frac{H^5\exp\big(2H^2\big) SA}{P^{\star}_{\min} \sigma^2 \varepsilon^2}\right), & \text{{\RMDPKL}}
\end{cases}
\end{align}
number of samples. 
\end{cor}

The sample complexity bound for {\RMDPchi} shows linear dependence on the uncertainty radius $\sigma$, consistent with prior generative model results \citep{NeuRIPS2023_CuriousPriceDRRLGenerativeMdel_Shi, AnnalsStat2022_TheoreticalUnderstandingRMDP_Yang}. In particular, \citep[Theorem 3]{NeuRIPS2023_CuriousPriceDRRLGenerativeMdel_Shi} proved a near-optimal sample complexity of order $\mathcal{O}\left(\frac{SA(1+\sigma)H_\gamma^4}{\varepsilon^2}\right)$ for infinite-horizon stationary {\RMDPchi} with the effective horizon $H_\gamma = \frac{1}{1-\gamma}$, thus our bound aligns with this result while requiring no access to a generative model (note that the non-stationary nature of finite horizon MDPs generally requires an additional $H$ in complexity \cite{JMLR2024_DROfflineRLNearOptimalSampelComplexity_Shi}), showing the tightness of our algorithm. 
% Furthermore, comparing with the non-robust case $(\sigma=0)$, the bound simplifies to $\mathcal{O}\left(\frac{H^5 SA}{\varepsilon^2}\right)$, matching the minimax lower bound for online MDPs \citep{PMLR2017_MinimxRegretBoundNonRobustRL_Azar}. This demonstrates that our algorithm naturally extends to the classical MDP setting.

Similarly, our sample complexity bounds for {\RMDPKL} are comparable to existing results under both generative models \cite{AnnalsStat2022_TheoreticalUnderstandingRMDP_Yang, PMLR2022_SampleComplexityRORL_Panaganti} and offline settings \cite{NeuRIPS2024_UnifiedPessimismOfflineRL_Yue, JMLR2024_DROfflineRLNearOptimalSampelComplexity_Shi, NeuRIPS2023_DoublePessimismDROfflineRL_Blanchet, AnnalsStat2022_TheoreticalUnderstandingRMDP_Yang}.  Notably, in the offline setting, the sample complexity for the KL set is shown to be  $\tilde{\mathcal{O}}\left(\frac{S^2C^\star H^4e^H}{\sigma^2 \varepsilon^2}\right)$ \cite{NeuRIPS2023_DoublePessimismDROfflineRL_Blanchet}, where $C^\star$ is the converge coefficient. Our result is hence comparable to theirs. 

The appearance of exponential (or $P_{\min}^\star$-dependent) factors in KL-based bounds is not an artifact of our analysis but a direct consequence of the geometry of KL ambiguity sets. Through the dual representation of KL balls \eqref{eq:dual_KL}, robust expectations involve log-moment generating functions of $V$ under the nominal kernel, which naturally yield dependencies that can be re-expressed in terms of the minimal transition probability $P_{\min}^\star$, leading to factors of order $(P_{\min}^\star)^{-2}$ \cite{NeuRIPS2023_DoublePessimismDROfflineRL_Blanchet}. These terms are unavoidable: KL divergence permits adversarial mass shifts toward states with tiny nominal probability, and such states are intrinsically hard to estimate robustly.

Furthermore, our bounds are stated uniformly for all $\sigma>0$. When $\sigma$ is very small, the KL-ambiguity set is nearly degenerate around the nominal model, so robustness is essentially unnecessary and unified robust bounds may appear loose relative to the non-robust case. When $P_{\min}^\star$ is extremely small, rare transitions are observed too infrequently to constrain worst-case KL perturbations; any DRRL algorithm must pay a sample-complexity price to guard against adversarial behavior on such transitions \cite{INFORMS2023_DRBatchContxBandits_Si,wang2023model,NeuRIPS2023_DoublePessimismDROfflineRL_Blanchet,JMLR2024_DROfflineRLNearOptimalSampelComplexity_Shi}. Thus the $\exp(\cdot)$-type and $1/(P_{\min}^\star\sigma^2)$ dependencies in KL-robust bounds capture fundamental hardness rather than proof artifacts, and our {\RMDPKL} guarantees are effectively tight: they align with known behavior in generative and offline models while showing that interactive online learning attains comparable robustness without extra structural assumptions beyond those inherent to the KL ambiguity.

\section{Lower Bound for Online RMDP}
\label{sec:Lower_Bound_Arxiv}
\begin{table*}[t]
\centering
\renewcommand{\arraystretch}{1.3}
\caption{Comparison between {\Algoname} and prior results on online RMDP}
\label{tab:comparison_theoretical_guarantees}
\begin{tabular}{|
>{\centering\arraybackslash}p{3.4cm}|
>{\centering\arraybackslash}p{2cm}|
>{\centering\arraybackslash}p{7cm}|
>{\centering\arraybackslash}p{3.2cm}|
}
\hline
\textbf{Model Assump.} & \textbf{Algorithm} & \textbf{Regret} & \textbf{Lower Bound} \\
\hline
Vanishing minimal value \cite{Arxiv2024_DRORLwithInteractiveData_Lu} & TV-OPROVI & 
$\tilde{\mathcal{O}}\big(\sqrt{\min\{H, \sigma^{-1}\}H^2 S A K} \big)$ & 
N/A \\
\hline
\multirow{3}{=}{Supremal visitation ratio $C_{vr}$ \cite{ICML2025_OnlineDRMDPSampleComplexity_He}} 
& TV-ORBIT & 
$\tilde{\mathcal{O}}\big(C_{vr} S^2 A H^2 + \sqrt{C_{vr}H^4 S^3 A K} \big)$ & 
\multirow{3}{*}{$\Omega\Big(\sqrt{C_{vr} K}\Big)$} \\
\cline{2-3}
& $\chi^2$-ORBIT & 
$\tilde{\mathcal{O}}\big(C_{vr} S^2 A H^2 + \sqrt{C_{vr}H^4 S^3 A K} \big)$ & \\
\cline{2-3}
& KL-ORBIT & 
$\tilde{\mathcal{O}}\left(\left(1 + \frac{H \sqrt{S}}{\sigma P^\star_{\min}}\right)(C_{vr} S A H + \sqrt{C_{vr} H^2 S A K})\right)$ & \\
\hline
\multirow{2}{*}{No assump. \textbf{(our work)}} 
& {\Algonamechi} & 
$\mathcal{O}\big(\sqrt{H^4(1+\sigma)SAK}\big)$ &  $\Omega\big(\sqrt{H^4(1+\sigma)SAK}\big)$ \\
\cline{2-4}
& {\AlgonameKL} & 
$\tilde{\mathcal{O}}\left(\sqrt{\frac{H^4\exp(2H^2)S A K}{P^\star_{\min} \sigma^2}} \right)$ & 
$\Omega\left(\sqrt{\frac{H^4S A K}{P^\star_{\min} \sigma^2} }\right)$ \\
\hline
\end{tabular}
\end{table*}

To further understand the hardness of online DRRL and assess the tightness of the upper bounds presented in \cref{thm:Regret_f_bound}, we now study corresponding minimax lower bounds for {\RMDPchi} and {\RMDPKL}.

\begin{thm}[Minimax Lower Bound of Online DRRL]
\label{thm:regret_lower_bound}
For any learning algorithm $\xi$, there exist an $f$-RMDP $\mathcal{M}$ with the following regret bound with $\xi$, as long as $K\geq A$:
% $\mathcal{M} \in \boldsymbol{\Xi}$ having the following regret bound with $\xi$: 
    \begin{align}
        \label{eq:regret_lower_bound}
        % &\sup_{\mathcal{M}_i \in {\bf \Xi}} \mathbb{E}[\text{Regret}_{\mathcal{M}_i}(\xi, K)]\nonumber\\
      \mathbb{E}[\text{Regret}_{\mathcal{M}}(\xi, K)]
        &=
        \begin{cases}
            \Omega\Big(\sqrt{H^4(1+\sigma)SAK}\Big), &\text{{\RMDPchi}}\\
            \Omega\Big(\sqrt{\frac{H^4SAK}{(1-P^\star_{\min})\sigma^2}}\Big), &\text{{\RMDPKL}}
        \end{cases}
    \end{align}
    where $f(K)=\Omega(g(K))$ means $\limsup_{K\to\infty}\frac{f(K)}{g(K)}> 0$. 
\end{thm}
 Noting that the previous results are all for generative model settings \cite{NeuRIPS2023_CuriousPriceDRRLGenerativeMdel_Shi} or offline setting \cite{JMLR2024_DROfflineRLNearOptimalSampelComplexity_Shi,NeurIPS2024_MinimaxOptimalOfflineRL_Liu}, our results stand for the first minimax lower bound for the more involved online setting. We also note that our upper bounds in \Cref{thm:Regret_f_bound} match with the lower bound in \Cref{thm:regret_lower_bound} in parameters \( K, S, A \), up to some logarithmic factors. This implies that our {\Algoname} is nearly minimax optimal. Our algorithm is hence the first online DRRL algorithm to achieve near-optimal sample complexity without structural assumptions.

  \textbf{Proof Sketch.} We then briefly discuss our proof of the minimax lower bound in \Cref{thm:regret_lower_bound}. The proof is inspired by techniques from \cite{Book2020_Bandit_Lattimore} and builds on the lower bound framework of \cite{ICML2025_OnlineDRMDPSampleComplexity_He}. 

 \textbf{Construction of hard instances.} We construct a family of hard RMDPs ${\bf \Xi} := \{\mathcal{M}_1, \dots, \mathcal{M}_S\}$, where each $\mathcal{M}_i$ has 3 states, $A$ actions, and horizon $H$, with $s_1$ as the initial state. All MDPs share the same nominal transition: $s_1$ transitions to $s_2$ with probability $p$ and to $s_3$ with probability $1-p$. At $s_2$, the agent receives a Gaussian reward (variance 1) whose mean depends on the action; at $s_3$, it receives a fixed reward of 1. Each $\mathcal{M}_i$ differs only in the reward associated with a single action $a_{c_i}$ at $s_2$, where $c_1 = 1$, and $c_i \neq c_j$ for $i \neq j$. In $\mathcal{M}_1$, action $a_1$ yields mean reward $\mu^\star$; in $\mathcal{M}_i$ for $i \geq 2$, action $a_{c_i}(\neq a_1)$ yields a higher mean $\sqrt{S}\mu^\star$, making it uniquely optimal.

\textbf{Worst-transition kernel.} The worst-case transition kernel $P^\omega$ is chosen that minimizes the expected value while satisfying an $f$-divergence constraint. Since $V^{\pi,\sigma}_h(s_2) < V^{\pi,\sigma}_h(s_3)$, $P^{\omega}$ for step $h=1$ is determined by the supremum value $\tilde{p} = \sup \{ p' : D_f(P^\omega_1, P^\star_1) \leq \sigma \}$ for \( f \)-divergence uncertainty set. For steps $h=2,\dots,H$, $P^{\omega}_h=1$ for state $s_2$ or $s_3$.

\textbf{Expected regret between RMDPs.} Define $\mathbb{P}^0_1$ and $\mathbb{P}^0_i$ as the probability distribution over all $K$ episodes under $\mathcal{M}_1$ and $\mathcal{M}_i$, respectively. By applying the Bretagnolle–Huber inequality (\Cref{lem:Bretagnolle_Huber_inequality}), we can bound the sum of expected regret under  $\mathcal{M}_1$ and $\mathcal{M}_i$ as
\begin{align*}
    &\mathbb{E}[\text{Regret}_{\mathcal{M}_1} + \text{Regret}_{\mathcal{M}_i}] 
    = \Omega\left(\tilde{p}(H-1)KSH\mu^\star \cdot e^{-\text{KL}(\mathbb{P}^0_1, \mathbb{P}^0_i)} \right).
\end{align*}
By substituting $\mu^\star = \sqrt{c'A/(SKp)}$, we get that
\begin{align*}
    \sup_{\mathcal{M}_i \in {\bf \Xi}} \mathbb{E}[\text{Regret}_{\mathcal{M}_i}(\xi,K)] 
    = \Omega\left( \frac{\tilde{p}}{\sqrt{p}} \sqrt{H^4KSA} \right).
\end{align*}

\textbf{Lower bounds for {\RMDPchi} and {\RMDPKL}.} For {\RMDPchi}, applying $\tilde{p} = p + \sqrt{\sigma p(1-p)}$, and setting $p=1/2$, we get the bound in \eqref{eq:regret_lower_bound} for {\RMDPchi}.
% \begin{align*}
%     \sup_{\mathcal{M}_i \in {\bf \Xi}} \mathbb{E}[\text{Regret}_{\mathcal{M}_i}(\xi,K)] 
%     = \Omega\left( \sqrt{(1+\sigma)H^4 K S A} \right).
% \end{align*}

For {\RMDPKL}, we set $p = 1 - P^\star_{\min}$, and define $\beta = \frac{1}{2}\log(1/P^\star_{\min})$, such that $\tilde{p} \geq \frac{1}{\beta}$. Under appropriate range of $\sigma$,
% as in \eqref{eq:sigma_range},  
we get the bound in \eqref{eq:regret_lower_bound} for {\RMDPKL}.\qedhere

\section{Comparisons with Prior Works}
\label{sec:Comparison_Prior_Work_Arxiv}
 We then compare our results with two most related prior studies on online DRRL \cite{Arxiv2024_DRORLwithInteractiveData_Lu,ICML2025_OnlineDRMDPSampleComplexity_He}. Other related works will be discussed in \Cref{app:related_work}. Compared to the existing works \cite{Arxiv2024_DRORLwithInteractiveData_Lu,ICML2025_OnlineDRMDPSampleComplexity_He}, our results enjoy two major advantages: {\it better applicability}, and {\it better sample efficiency}.

The work \cite{Arxiv2024_DRORLwithInteractiveData_Lu} studies online learning under the TV-divergence uncertainty set.
% \textcolor{red}{Want to add to the table that they only look at TV set?} % \textcolor{blue}{The key differentiator is the assumption they require—specifically, the vanishing minimal state visitation—which we have already highlighted in the "Assumptions" column of the table. Adding a separate column for the uncertainty set might push the table beyond the page margins. However, I’m happy to include it if you feel it would improve clarity or completeness.} 
Their results and methods heavily rely on the additional assumption of the fail-state condition and vanishing minimal states, which effectively addresses the information deficit in their case. However, such simplifications do not extend to general $f$-divergences, such as $\chi^2$ or KL, whose worst-case solutions become more difficult \cite{INFORM2005_RobusDP_Iyengar}. A more recent work \cite{ICML2025_OnlineDRMDPSampleComplexity_He} studies all three uncertainty sets, however, their studies also rely on an assumption of the coverage of the worst-case kernel by the nominal environment. Specifically, they assume a supremal visitation ratio between the visitation distributions under the nominal and the worst-case kernels which is polynomial in $S,A,H$: $C_{vr}=\textbf{Poly}(S,A,H)$. Such an assumption similarly bypasses the information deficit as in \cite{Arxiv2024_DRORLwithInteractiveData_Lu}, inspired by the offline RL literature \cite{Anals2024_SettlingSampleComplexityModelbasedRL_Li, NeuRIPS2023_CuriousPriceDRRLGenerativeMdel_Shi}. However, both assumptions can be infeasible in practice, as there is no such prior knowledge on the distributionally RMDPs. Moreover, their implementations require an additional optimization oracle.  In contrast, our method makes no additional assumptions. We design confidence-aware updates that are fully data-driven and based on the dual representation of uncertainty sets, without any oracle. As a result, our analysis applies broadly to RMDPs with general $f$-divergences, providing robust learning guarantees while addressing the information deficit.

Moreover, our method has a better data efficiency (also see Table \ref{tab:comparison_theoretical_guarantees}). Specifically, ORBIT incurs a regret of $\tilde{\mathcal{O}}\big(\sqrt{C_{vr} S^3 A H^4 K}\big)$ for $\chi^2$ and $\tilde{\mathcal{O}}\left(\left(1 + \frac{H\sqrt{S}}{\sigma P^{\star}_{\min}}\right)\left(\sqrt{C_{vr} S A H^2 K}\right)\right)$ for KL, which are largely sub-optimal. In contrast, our regret bound does not depend on $C_{vr}$, providing more general results, and improve scaling in parameters, leading to tighter guarantees. % \textcolor{red}{They don't have an exponential term in $H$. Is that due to their visitation assumption?} \textcolor{blue}{Mainly due to their Q-updation strategy. We apply Q-upper and Q-lower as confidence region, which they don't. This complicated our analysis.} 

On the other hand, our studies on minimax lower bound provides an assumption-free result, specifying the dependence on $S,A,H,\sigma$, which presents a more detailed and accurate study on distributionally robust online learning compared to \cite{ICML2025_OnlineDRMDPSampleComplexity_He}. Moreover, our lower bounds indicate the near-optimality of our method.  

% By avoiding pessimistic assumptions tied to worst-case state visitation, 

Our method hence leads to both sharper theoretical bounds and more applicable and efficient learning in practice.

\section{Numerical Experiments}
\label{app:Numerical_Experiments_Arxiv}
\begin{figure*}[!htb]
\centering
\begin{subfigure}[b]{0.24\textwidth}
  \includegraphics[width=\linewidth]{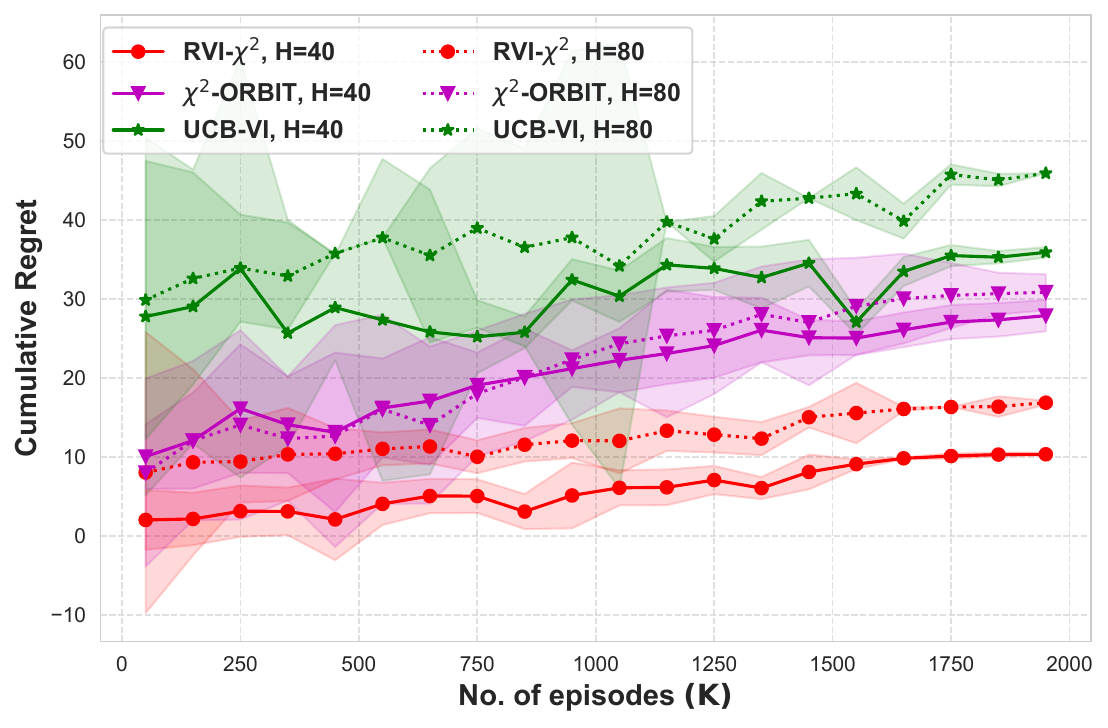}
  \caption{Regret vs. $K$ for {\RMDPchi}, $S=20$}
  \label{fig:Gambler_Chi_regret_vs_K_S20}
\end{subfigure}
\hfill
\begin{subfigure}[b]{0.24\textwidth}
  \includegraphics[width=\linewidth]{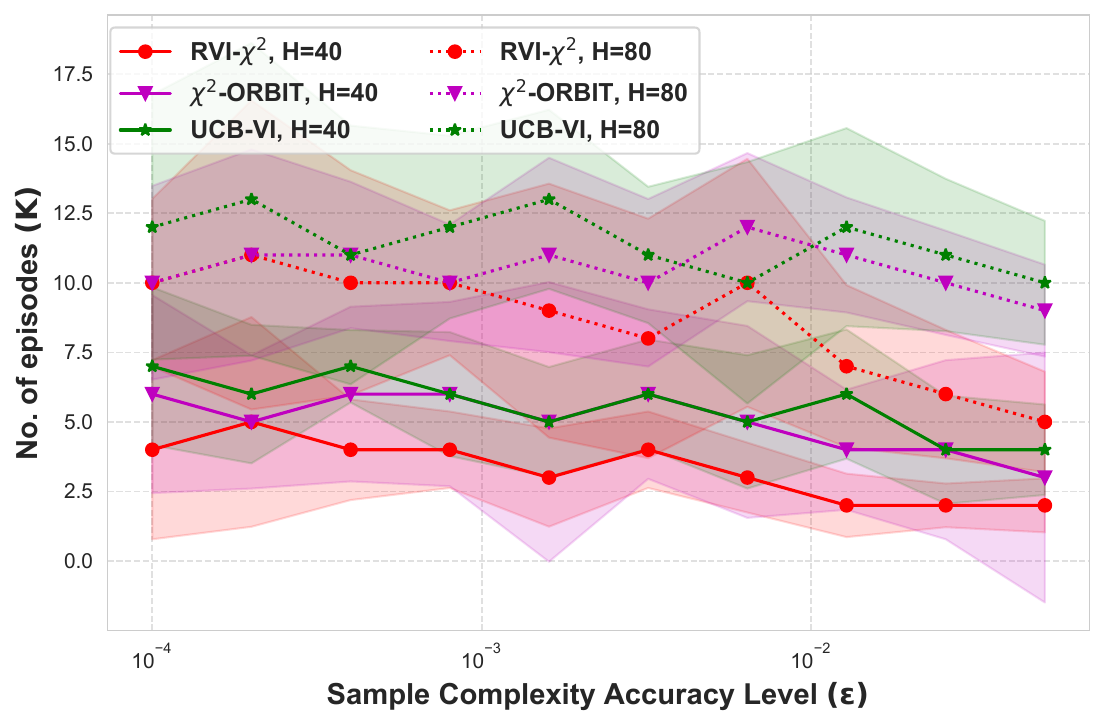}
  \caption{Accuracy level $\varepsilon$ vs. $K$ for {\RMDPchi}, $S=100$}
  \label{fig:Gambler_Chi_samplecomplexity_vs_K_S100}
\end{subfigure}
\hfill
\begin{subfigure}[b]{0.24\textwidth}
  \includegraphics[width=\linewidth]{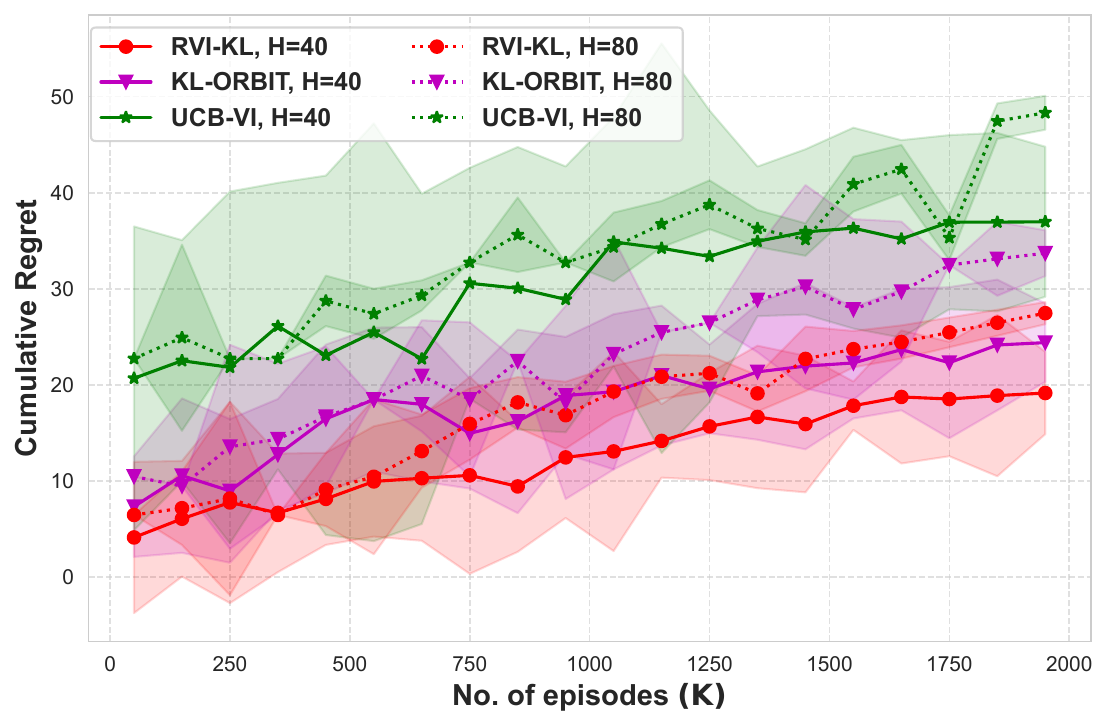}
  \caption{Regret vs. $K$ for {\RMDPKL}, $S=20$}
  \label{fig:Gambler_KL_regret_vs_K_S20}
\end{subfigure}
\hfill
\begin{subfigure}[b]{0.24\textwidth}
  \includegraphics[width=\linewidth]{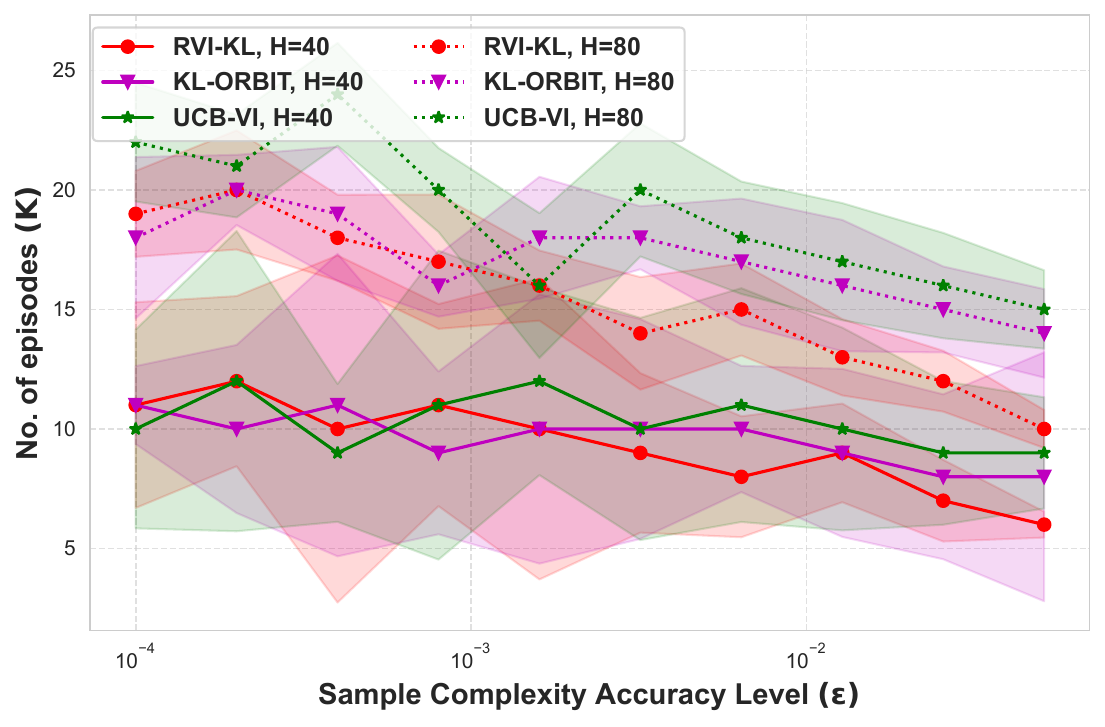}
  \caption{Accuracy level $\varepsilon$ vs. $K$ for {\RMDPKL}, $S=100$}
  \label{fig:Gambler_KL_samplecomplexity_vs_K_S100}
\end{subfigure}
\caption{Performance comparisons for the Gambler’s problem under {\RMDPchi} ($\sigma=0.05)$ and {\RMDPKL} ($\sigma=0.1)$.}
\label{fig:Gambler_comparison_all}
\end{figure*}

\begin{figure*}[!htb]
\centering
\begin{subfigure}[b]{0.24\textwidth}
  \includegraphics[width=\linewidth]{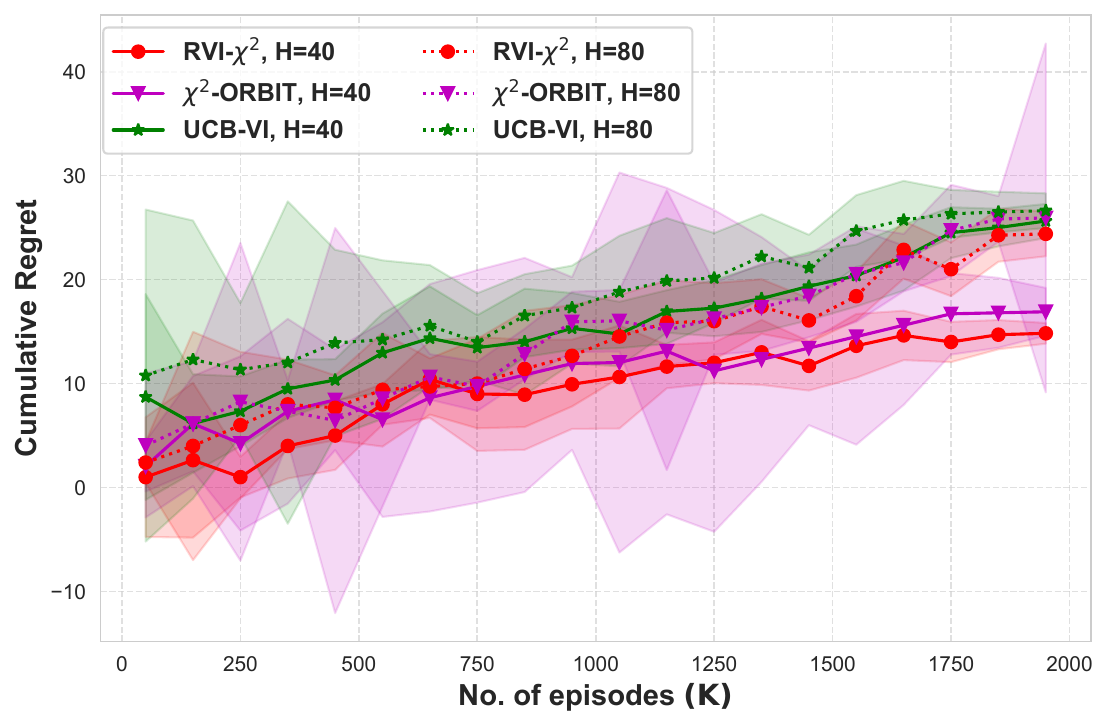}
  \caption{Regret vs. $K$ for {\RMDPchi}, Grid size= $4 \times 4$}
  \label{fig:Frozen_Chi_regret_vs_K_gridsize4}
\end{subfigure}
\hfill
\begin{subfigure}[b]{0.24\textwidth}
  \includegraphics[width=\linewidth]{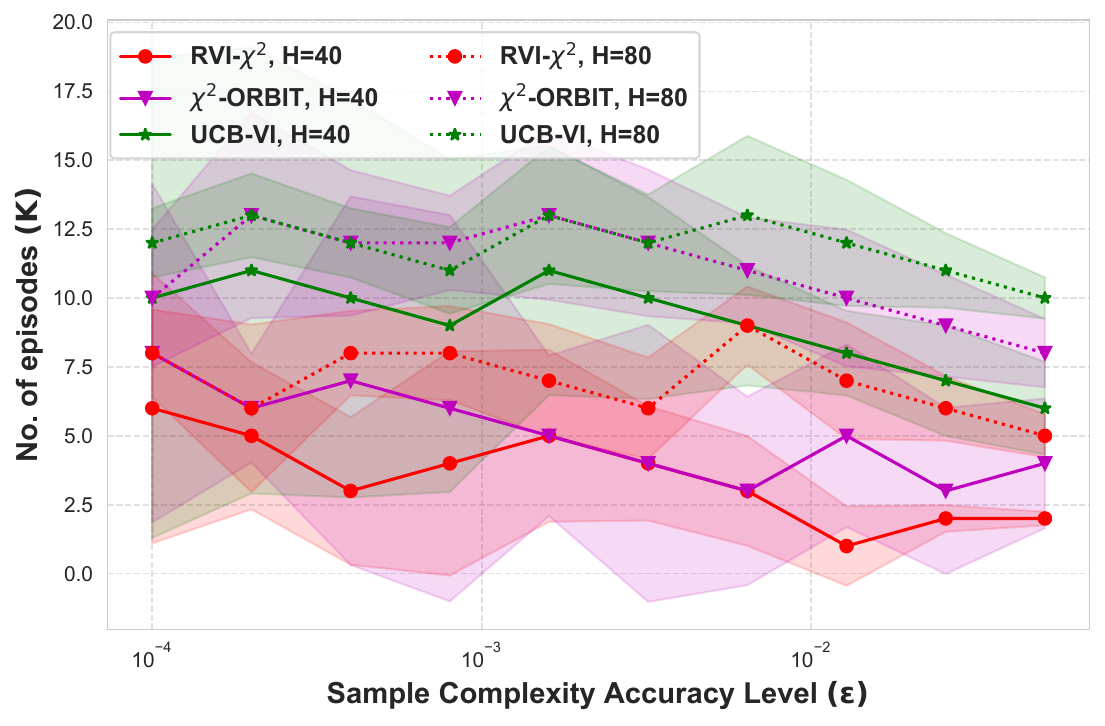}
  \caption{Accuracy level $\varepsilon$ vs. $K$ for {\RMDPchi}, Grid size= $20 \times 20$}
  \label{fig:Frozen_Chi_samplecomplexity_vs_K_gridsize20}
\end{subfigure}
\hfill
\begin{subfigure}[b]{0.24\textwidth}
  \includegraphics[width=\linewidth]{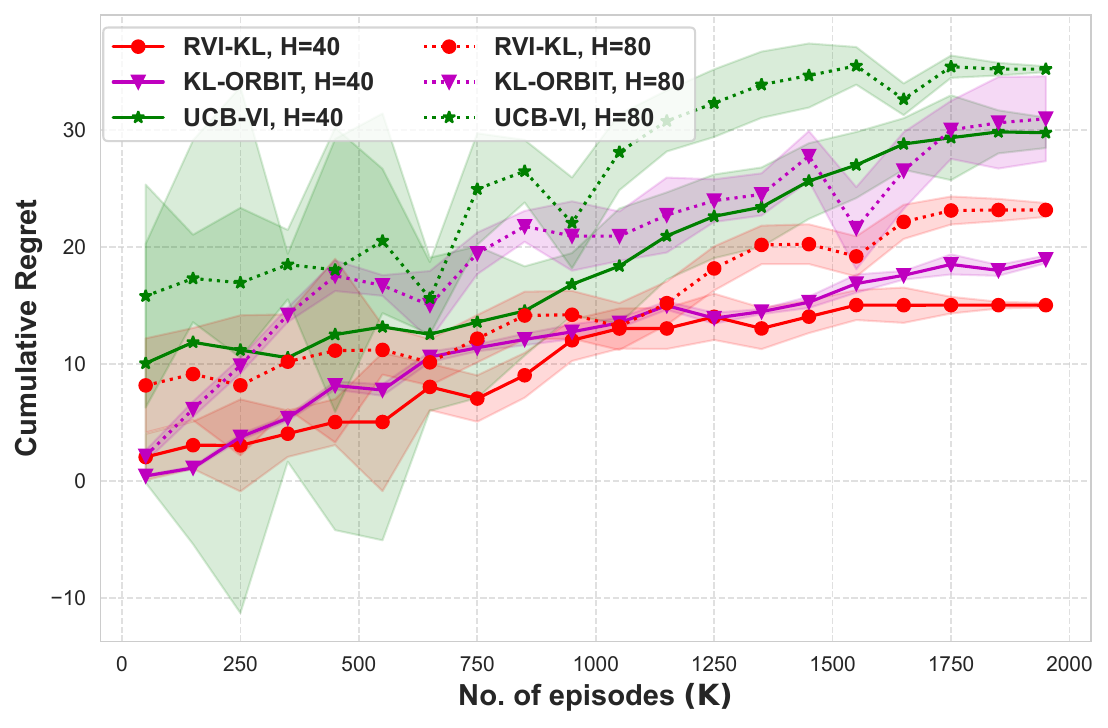}
  \caption{Regret vs. $K$ for {\RMDPKL}, Grid size = $4 \times 4$}
  \label{fig:Frozen_KL_regret_vs_K_gridsize4}
\end{subfigure}
\hfill
\begin{subfigure}[b]{0.24\textwidth}
  \includegraphics[width=\linewidth]{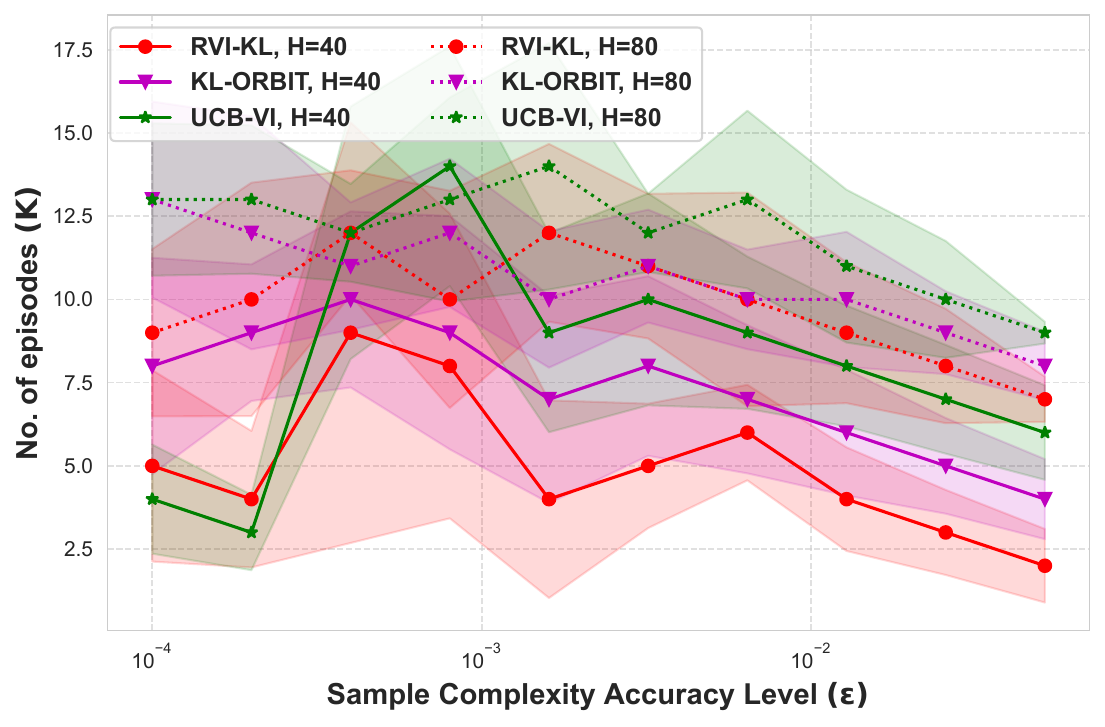}
  \caption{Accuracy level $\varepsilon$ vs. $K$ for {\RMDPKL}, Grid size = $20 \times 20$}
  \label{fig:Frozen_KL_samplecomplexity_vs_K_gridsize20}
\end{subfigure}
\caption{Performance comparisons for the Frozen Lake under {\RMDPchi} ($\sigma=0.05)$ and {\RMDPKL} ($\sigma=0.1$).}
\label{fig:Frozen_comparison_all}
\end{figure*}

In this section, we evaluate the effectiveness of {\Algoname} using two challenging environments: the Gambler’s problem \cite{Book1998_RL_Sutton, JMLR2024_DROfflineRLNearOptimalSampelComplexity_Shi, PMLR2021_DROTabularRL_Zhou} and the Frozen Lake environment \cite{Arxiv2016_openai_Brockman}. 

\subsection{Learning the Environments}
\label{app:Construction_games}
Here, we give the detailed environment setup of the Gambler’s problem  and the Frozen Lake environment, as explained below: 

\subsubsection{Gambler's Game Problem}
In the {\it Gambler’s game} \cite{Book1998_RL_Sutton, JMLR2024_DROfflineRLNearOptimalSampelComplexity_Shi, PMLR2021_DROTabularRL_Zhou}, a gambler places bets on a sequence of biased coin flips. At each step, the gambler wins the amount staked if the coin lands heads and loses it otherwise. Starting from some initial balance, the game terminates either when the gambler’s balance reaches the maximum target amount $S$, or drops to $0$, or the time horizon $H$ is reached. This setup defines an episodic MDP with state space $\mathcal{S} = \{0, 1, \dots, S\}$ and action space $\mathcal{A}(s) = \{0, 1, \dots, \min(s, S-s)\}$ at each state $s$. The transition dynamics are governed by a fixed coin bias, with the probability of heads set to $p_{\text{head}} = 0.6$ for all time steps $h \in [H]$. The reward function assigns a reward of $1$ when the state reaches $S$, and $0$ for all other cases. We evaluate this problem under two episode lengths as $H = 40$ and $H = 80$.

\subsubsection{Frozen Lake Problem}
In the {\it Frozen Lake} environment \cite{Arxiv2016_openai_Brockman}, an agent aims to cross a frozen lake of grid size $S \times S$ from a designated Start location to a Goal, while avoiding terminal Hole states. The environment is stochastic due to the slippery surface. The agent’s movement may deviate from the intended direction. The state space corresponds to the positions on the grid, and the action space includes the four cardinal directions. The agent receives a reward of $1$ when the agent reaches Goal and $0$ otherwise. Similar to the Gambler’s game, we consider two finite horizons lengths as $H = 40$ and $H = 80$.

Both environments are modeled as episodic finite-horizon RMDPs. We implement our algorithms for {\RMDPchi} and {\RMDPKL} and compare them with two baselines: (i) the standard non-robust UCB-VI algorithm \cite{PMLR2017_MinimxRegretBoundNonRobustRL_Azar}, and (ii) the robust ORBIT algorithm (using $\chi^2$-ORBIT for {\RMDPchi} and KL-ORBIT for {\RMDPKL}) from \cite{ICML2025_OnlineDRMDPSampleComplexity_He}. For all methods, we set $\sigma = 0.05$ in {\RMDPchi} and $\sigma = 0.1$ in {\RMDPKL}, obtain their output policies, and evaluate them under the corresponding RMDPs. We focus on two performance criteria: robust regret as a function of episodes $K$ and the sample complexity to achieve an $\varepsilon$-optimal robust policy. All results are averaged over 10 independent runs with confidence intervals.

% The uncertainty level is fixed at $\sigma = 0.1$ for both environments.
\subsection{Comparisons of {\Algoname} with the Benchmark Algorithms under {\RMDPchi} and {\RMDPKL}}

\subsubsection{Robust Regret}
We first evaluate and compare the robust regrets of all algorithms. For both the {$\chi^2$-divergence} and KL-divergence uncertainty models, our proposed algorithms—{\Algonamechi} and {\AlgonameKL}—consistently demonstrate superior performance in terms of robust regret. As shown in Figures \ref{fig:Gambler_Chi_regret_vs_K_S20} and \ref{fig:Frozen_Chi_regret_vs_K_gridsize4} (for {\RMDPchi}), and Figures \ref{fig:Gambler_KL_regret_vs_K_S20} and \ref{fig:Frozen_KL_regret_vs_K_gridsize4} (for {\RMDPKL}), the cumulative regret of all considered algorithms grows sub-linearly with the number of episodes $K$ for a fixed horizon $H$. Compared to the non-robust UCB-VI, {\Algonamechi} and {\AlgonameKL} consistently achieve lower regret across all configurations. The performance gap between our algorithms and UCB-VI becomes even more pronounced as the horizon increases, highlighting the effectiveness of our distributionally robust approach in managing model uncertainty. This strong and consistent performance makes {\Algonamechi} and {\AlgonameKL} more efficient and effective choices for complex and uncertain environments, verifying our theoretical results.

\subsubsection{Sample Complexity}
The sample complexity behavior further reinforces the advantages of our proposed algorithms. Figures \ref{fig:Gambler_Chi_samplecomplexity_vs_K_S100} and \ref{fig:Frozen_Chi_samplecomplexity_vs_K_gridsize20} (for {\RMDPchi}), along with Figures \ref{fig:Gambler_KL_regret_vs_K_S20} and \ref{fig:Frozen_KL_samplecomplexity_vs_K_gridsize20} (for {\RMDPKL}), illustrate the number of episodes $K$ required to achieve a given accuracy level $\varepsilon$. In both the Gambler and Frozen-Lake environments, {\Algonamechi} and {\AlgonameKL} consistently require fewer episodes to reach the same level of accuracy compared to the benchmark algorithms. This advantage is especially noticeable at higher accuracy levels, i.e., for larger values of $\varepsilon$.
This sample-efficiency advantage becomes more significant as the horizon $H$ increases, indicating that the robust formulation of both {\Algonamechi} and {\AlgonameKL} is more sample-efficient and scales better in deeper decision-making scenarios. Overall, these results highlight the superior sample complexity of {\Algonamechi} and {\AlgonameKL}, reinforcing their effectiveness in efficiently learning robust policies under distributional uncertainty, under both $\chi^2$- or KL-divergence uncertainty sets.

\section{Conclusion}
\label{sec:Conclusion_Arxiv}

In this paper, we studied online DRRL in the RMDP framework under general $f$-divergence uncertainty sets, including $\chi^2$ and KL divergences. Unlike prior work, our approach avoids strong structural assumptions, such as vanishing minimal values or bounded supremal visitation ratios, but can still achieve convergence guarantees. We proposed a computationally efficient algorithm, $f$-ORVIT, that achieve a sub-linear robust regret bound. Moreover, we derived the minimax regret lower for online distributionally robust learning, verifying the near-optimality of our algorithm. We hence provided the first tight performance guarantees for online DRRL under these uncertainty sets. Extensive experiments on diverse environments validate our theoretical results and demonstrate the practical robustness and efficiency of our method. This work opens several promising directions, including online robust learning for large-scale problems with function approximation.

% where significant theoretical and algorithmic gaps remain.

\section*{Acknowledgments}
This work was supported in part by DARPA under Agreement No. HR0011-24-9-0427. The authors thank the anonymous reviewers for their constructive feedback. %We also gratefully acknowledge our colleagues at the University of Central Florida for their valuable discussions.

%Bibliography
\bibliographystyle{unsrt}  
\bibliography{Arxviv_Version_Online/reference}

\newpage
\appendix
\onecolumn
%\section{Appendix}
\label{appendix}
\section{Proofs of Theoretical Results of {\Algoname}}
We present the proofs of the regret bound for {\Algonamechi} and {\AlgonameKL}.

\subsection{Proof of regret bound of {\Algonamechi}}
\label{app:thm:Regret_chi_bound}

\subsubsection{Define the event \( \mathcal{E}_{\chi^2} \) for {\RMDPchi}:}
Before presenting all lemmas, we define the typical event \( \mathcal{E}_{\chi^2} \)  as

%V^{\star,\sigma}_{h+1} ()recheck with Wang)
\begin{align}
\label{eq:Event}
    \mathcal{E}_{\chi^2} =\bigg\{&\abs{\sqrt{\text{Var}_{P^{\star}_h(\cdot|s,a)}(\eta-V_{h+1})_{+}}- \sqrt{\text{Var}_{\widehat{P}^{k}_h(\cdot|s,a)}(\eta-V_{h+1})_{+}}}\leq \sqrt{\frac{c_1L\text{Var}_{\widehat{P}^k_h(\cdot|s,a)}(V_{h+1})}{\{N^k_h(s,a)\vee 1\}}} + \frac{c_2HL}{\{N^k_h(s,a)\vee 1\}},\nonumber\\
    & \abs{\left[ \bigg(\mathbb{P}^\star_h - \widehat{\mathbb{P}}^k_h\bigg)\bigg(\eta-V_{h+1} \bigg)\right](s,a)}\leq \sqrt{\frac{\sigma c_1L\text{Var}_{\widehat{P}^k_h(\cdot|s,a)}(V_{h+1})}{\{N^k_h(s,a)\vee 1\}}} + \frac{c_2\sqrt{\sigma}HL}{\{N^k_h(s,a)\vee 1\}},\nonumber\\
&\abs{P^{\star}_h(s^{\prime}|s,a)-\widehat{P}^k_h(s^{\prime}|s,a)}\leq \sqrt{\frac{c_1L\min\{P^{\star}_h(s^{\prime}|s,a),\widehat{P}^k_h(s^{\prime}|s,a)\}}{\{N^k_h(s,a)\vee 1\}}} + \frac{c_2L}{\{N^k_h(s,a)\vee 1\}},\nonumber\\
    &\forall (h, s,a,s^{\prime},k) \in [H]\times \mathcal{S}\times \mathcal{A}\times \mathcal{S}\times[K], \forall \eta \in \mathcal{N}_{1/S\sqrt{K}}([0,H]) \bigg\},
\end{align}
where $L=\log(S^3AH^2K^{3/2}/\delta)$ and $c_1, c_2 > 0$ are two constants and $\eta \in \mathcal{N}_{1/S\sqrt{K}}([0,H])$, where $\mathcal{N}_{1/S\sqrt{K}}([0,H])$ denotes an $1/S\sqrt{K}$-cover of the interval $[0,H]$.

\begin{lem}[Bound of event $\mathcal{E}_{\chi^2}$]
\label{lem:confidence_event}
    For the typical event $\mathcal{E}_{\chi^2}$ defined in \eqref{eq:Event}, it holds that $\Pr(\mathcal{E}_{\chi^2}) \geq 1 - \delta$.
\end{lem}
\begin{proof}
This result follows directly from \Cref{lem:self_bound_variance}, which is a version of Bernstein's inequality and its empirical counterpart from \cite{Arxiv2009_EmpBernsteinBounds_Maurer}. To extend the bound uniformly, we apply a union bound over all tuples $
(h, s, a, s', k, \eta) \in [H] \times \mathcal{S} \times \mathcal{A} \times \mathcal{S} \times [K] \times \mathcal{N}_{1/(S\sqrt{K})}\big([0, H]\big)$.
Here, the size of $\mathcal{N}_{1/(S\sqrt{K})}\big([0, H]\big)$ is of order $\mathcal{O}(SH\sqrt{K})$.
\end{proof}

In the following, we condition on the event $\mathcal{E}_{\chi^2}$, which, according to \Cref{lem:confidence_event}, holds with probability at least $1 - \delta$.

\subsubsection{Proof of \Cref{thm:Regret_f_bound}({\RMDPchi} Setting)}
\label{app:proof_thm_Regret_chi_bound}
\begin{proof}
With \Cref{lem:Optimistic_pessimism}, we can upper bound the regret as
\begin{align}
    \text{Regret}(K) = \sum_{k=1}^{K} V_{1}^{\star,\sigma}(s^k_1) - V_{1}^{\pi^k,\sigma}(s^k_1) 
    \leq \sum_{k=1}^{K} \overline{V}_1^k(s^k_1) - \underline{V}_1^k(s^k_1). \label{eq:Regret_step1}
\end{align}

In the following, we break our proof into three steps.
\begin{itemize}
\item \textbf{Step 1: Upper bound of \eqref{eq:Regret_step1}.} By the choice of $\overline{Q}_h^k$, $\underline{Q}_h^k$, $\overline{V}_h^k$, $\underline{V}_h^k$ as given in \eqref{eq:Upper_estimate_Q}, \eqref{eq:Lower_estimate_Q} and \eqref{eq:policy_value_epsiode_k}, and by the choice of bonus term $B^{\chi^2}_{k,h}(s,a)$ given in \eqref{eq:Bonus_term_chi} for any $(h,k) \in [H] \times [K]$ and $(s,a) \in \mathcal{S} \times \mathcal{A}$,
\begin{align}
    \overline{Q}_h^k(s,a) - \underline{Q}_h^k(s,a)
    &= \min \left\{ r_h(s,a) + \mathbb{E}_{\widehat{\mathcal{U}^{\sigma}_h}(s,a)}\left[ \overline{V}_{h+1}^k \right] + B^{\chi^2}_{k,h}(s,a), H\right\} \\
    &\quad - \max \left\{ r_h(s,a) + \mathbb{E}_{\widehat{\mathcal{U}^{\sigma}_h}(s,a)}\left[ \underline{V}_{h+1}^k \right] - B^{\chi^2}_{k,h}(s,a), 0 \right\} \\
    &\leq \mathbb{E}_{\widehat{\mathcal{U}^{\sigma}_h}(s,a)}\left[ \overline{V}_{h+1}^k \right] 
    - \mathbb{E}_{\widehat{\mathcal{U}^{\sigma}_h}(s,a)}\left[ \underline{V}_{h+1}^k \right] + 2B^{\chi^2}_{k,h}(s,a).\label{eq:Regret_step2}
\end{align}
We denote
\begin{align}
    A:= &\mathbb{E}_{\widehat{\mathcal{U}^{\sigma}_h}(s,a)}\left[ \overline{V}_{h+1}^k \right]  - \mathbb{E}_{\mathcal{U}^{\sigma}_h(s,a)}\left[ \overline{V}_{h+1}^k \right] + \mathbb{E}_{\mathcal{U}^{\sigma}_h(s,a)}\left[ \underline{V}_{h+1}^k \right] - \mathbb{E}_{\widehat{\mathcal{U}^{\sigma}_h}(s,a)}\left[ \underline{V}_{h+1}^k \right] \label{eq:Regret_A}\\
    B := &\mathbb{E}_{\mathcal{U}^{\sigma}_h(s,a)}\left[ \overline{V}_{h+1}^k \right] - \mathbb{E}_{\mathcal{U}^{\sigma}_h(s,a)}\left[ \underline{V}_{h+1}^k \right]\label{eq:Regret_B}
\end{align}
Applying \eqref{eq:Regret_A} and \eqref{eq:Regret_B} in \eqref{eq:Regret_step2}, we get
\begin{align}
   \overline{Q}_h^k(s,a) - \underline{Q}_h^k(s,a)   &\leq A + B + 2B^{\chi^2}_{k,h}(s,a). \label{eq:Regret_step3}
\end{align}

\begin{enumerate}[label=(\roman*)]
    \item \textbf{Upper bound $A$.} By applying a Bernstein-style concentration inequality tailored to {\RMDPchi}, as shown in \Cref{lem:Proper_bouns_optimism_pessimism_bound}, we can upper bound term \(A\) by the corresponding bonus term, as given by
    \begin{align}
        A \leq 2B^{\chi^2}_{k,h}(s,a). \label{eq:Regret_bound_A}
    \end{align}

    \item \textbf{Upper bound $B$.}  By the definition of $B$ as given in \eqref{eq:Regret_B}, and by \Cref{lem:Optimistic_pessimism}, we have
\begin{align}
 B &= \mathbb{E}_{\mathcal{U}^{\sigma}_h(s,a)}\left[ \overline{V}_{h+1}^k \right] - \mathbb{E}_{\mathcal{U}^{\sigma}_h(s,a)}\left[ \underline{V}_{h+1}^k \right]\leq \mathbb{E}_{\mathcal{U}^{\sigma}_h(s,a)}\left[ \overline{V}_{h+1}^k \right].\label{eq:Regret_step5}
 \end{align}
By the definition of $\mathbb{E}_{\mathcal{U}^{\sigma}_h(s,a)}[V]$ as given in \eqref{eq:dual_chi}
\begin{align}
\mathbb{E}_{\mathcal{U}^{\sigma}_h(s,a)}\left[ \overline{V}_{h+1}^k \right] &= \sup_{\eta \in [0,H]} \left\{-\sqrt{\sigma\text{Var}_{P^{\star}_h(\cdot|s,a)}(\eta - \overline{V}^k_{h+1})_{+}} + \bigg[\mathbb{P}^\star_h\Big(\overline{V}^k_{h+1}-\eta\Big)_+ \bigg](s,a)\right\},\nonumber\\
&\overset{(i)}{\leq} \sup_{\eta \in [0,H]} \left\{\bigg[\mathbb{P}^\star_h\Big(\overline{V}^k_{h+1}-\eta\Big)_+ \bigg](s,a)\right\}, \nonumber\\
&\leq \left[\mathbb{P}^\star_h\Big(\overline{V}^k_{h+1}\Big)\right](s,a),\label{eq:Regret_step6}
\end{align}
where (i) is due to the fact $\text{Var}_{P^{\star}_h(\cdot|s,a)}(\eta - \overline{V}^k_{h+1})_{+}\geq 0$. Therefore, by applying \eqref{eq:Regret_step6} in \eqref{eq:Regret_step5}, we get
\begin{align}
\label{eq:Regret_bound_B}
    B \leq \left[\mathbb{P}^\star_h\Big(\overline{V}^k_{h+1}\Big)\right](s,a).
\end{align}
\end{enumerate}

By applying \eqref{eq:Regret_bound_A} and \eqref{eq:Regret_bound_B} in \eqref{eq:Regret_step3}, we get
\begin{align}
     \overline{Q}_h^k(s,a) - \underline{Q}_h^k(s,a)  &\leq \mathbb{P}^\star_h\Big(\overline{V}^k_{h+1}\Big) + 4B^{\chi^2}_{k,h}(s,a). \label{eq:Regret_step9}
\end{align}

We recall the bound of bonus term as given in \Cref{lem:Control_Bonus}, as 
\begin{align}
\label{eq:Regret_step10}
    B^{\chi^2}_{k,h}(s,a) &\leq \sqrt{\frac{\sigma c_1L\text{Var}_{P_h^\star(\cdot|s,a)}\left[ V_{h+1}^{\pi^k,\sigma} \right]}{\{N_h^k(s,a) \vee 1\}}}+ \frac{4 \sqrt{\sigma}\left[ \mathbb{P}^\star_h\big(\overline{V}_{h+1}^k\big) \right](s,a)}{H} + \frac{c_2\sqrt{\sigma} H^2 S(2L+1)}{\sqrt{\{N_h^k(s,a)\vee 1\}}}
+ \sqrt{\frac{\sigma}{K}}.
\end{align}

By applying \eqref{eq:Regret_step10} in \eqref{eq:Regret_step9}, and after rearranging terms we further obtain that
\begin{align}
   \overline{Q}_h^k(s,a) - \underline{Q}_h^k(s,a) &\leq \left(1 + \frac{16\sqrt{\sigma}}{H}\right) \left[\mathbb{P}^\star_h\big(\overline{V}^k_{h+1} \big)\right](s,a)+ 4 \sqrt{ \frac{ \sigma c_1L\text{Var}_{P_h^{\star}(\cdot|s,a)} \left[ V_{h+1}^{\pi^k,\sigma} \right]}{\{N_h^k(s,a) \vee 1\}} } \nonumber\\
   &\qquad + \frac{4 c_2\sqrt{\sigma} H^2 S(2L+1)}{\sqrt{\{N_h^k(s,a) \vee 1\}}} + 4\sqrt{\frac{\sigma}{K}},\label{eq:Regret_step11}
\end{align}
where $c^{\sigma}_{H}, c_1, c_2 > 0$ are two absolute constants. 

For the sake of brevity, we now introduce the  following notations of differences, for any $(h,k) \in [H] \times [K]$, as given by
\begin{align}
    \Delta^k_h &:= \overline{V}^k_{h}(s^k_h) -\underline{V}^k_{h}(s^k_h),\label{eq:Delta_k_h}\\
    \zeta_h^k &:= \Delta_h^k - \left( \overline{Q}_h^k(s_h^k, a_h^k) - \underline{Q}_h^k(s_h^k, a_h^k) \right), \label{eq:zeta_h_k}\\
     \xi_h^k &:= \left[\mathbb{P}^\star_h\Big(\overline{V}^k_{h+1}\Big)\right](s^k_h,a^k_h) - \Delta_{h+1}^k.\label{eq:xi_h_k}
    %\xi_h^k &:= \max\bigg\{\sqrt{ \abs{\left[\mathbb{P}^\star_h\big(\overline{V}^k_{h+1} -\underline{V}^k_{h+1}\big) \right](s,a)}},\left[\mathbb{P}^\star_h\big(\overline{V}^k_{h+1} -\underline{V}^k_{h+1}\big)\right](s,a) \bigg\} - \Delta_{h+1}^k. 
\end{align}
We now define the filtration $\{ \mathcal{F}_{h,k} \}_{(h,k) \in [H] \times [K]}$ as
\begin{align*}
    \mathcal{F}_{h,k} := \sigma \bigg( \Big\{\big(s_i^\tau, a_i^\tau \big)\Big\}_{(i,\tau) \in [H] \times [k-1]} \bigcup \Big\{\big(s_i^k, a_i^k\big)\Big\}_{i \in [h-1]} \bigcup \Big\{s_h^k\Big\} \bigg).
\end{align*}
Considering the filtration $\{ \mathcal{F}_{h,k} \}_{(h,k) \in [H] \times [K]}$, we can find that $\{\zeta_h^k\}_{(h,k) \in [H] \times [K]}$ is a martingale difference sequence with respect to $\{ \mathcal{F}_{h,k} \}_{(h,k) \in [H] \times [K]}$ and $\{\xi_h^k\}_{(h,k) \in [H] \times [K]}$ is a submartingale difference sequence with respect to $\{\mathcal{F}_{h,k} \cup \{a_h^k\}\}_{(h,k) \in [H] \times [K]}$. Furthermore, applying \eqref{eq:Regret_step11} in \eqref{eq:zeta_h_k}, we have
\begin{align}
\Delta_h^k &= \zeta_h^k + \left( \overline{Q}_h^k(s_h^k, a_h^k) - \underline{Q}_h^k(s_h^k, a_h^k) \right) \nonumber \\
&\le \zeta_h^k + \left(1 + \frac{16\sqrt{\sigma}}{H}\right) \left[\mathbb{P}^\star_h\big(\overline{V}^k_{h+1} \big)\right](s,a) + 4 \sqrt{ \frac{ \sigma c_1 L\text{Var}_{P_h^{\star}(\cdot|s,a)} \left[ V_{h+1}^{\pi^k,\sigma} \right]}{\{N_h^k(s_h^k, a_h^k) \vee 1\}} }  + \frac{4 c_2\sqrt{\sigma} H^2 S(2L+1)}{\sqrt{\{N_h^k(s_h^k, a_h^k) \vee 1\}}} + 4\sqrt{\frac{\sigma}{K}} \nonumber \\
&= \zeta_h^k + \left(1 + \frac{16\sqrt{\sigma}}{H}\right) \xi_h^k + \left(1 + \frac{16\sqrt{\sigma}}{H}\right)\Delta_{h+1}^k + 4 \sqrt{ \frac{ \sigma c_1 L\text{Var}_{P_h^{\star}(\cdot|s,a)} \left[ V_{h+1}^{\pi^k,\sigma} \right]}{\{N_h^k(s_h^k, a_h^k) \vee 1\}} } + \frac{4 c_2\sqrt{\sigma} H^2 S(2L+1)}{\sqrt{\{N_h^k(s_h^k, a_h^k) \vee 1\}}} + 4\sqrt{\frac{\sigma}{K}}. \label{eq:Regret_step12}
\end{align}
Recursively applying \eqref{eq:Regret_step12} and using the fact that $1 \leq \left(1 + \frac{16\sqrt{\sigma}}{H}\right)^h \leq \left(1 + \frac{16\sqrt{\sigma}}{H}\right)^H :=d^{\sigma}_{H}$ for some constant $d^{\sigma}_{H} > 0$, we can upper bound the right hand side of \eqref{eq:Regret_step1} as
\begin{align}
\text{Regret}_{\bf \Phi}(K) \leq \sum_{k=1}^K \Delta_1^k &\leq C \cdot \sum_{k=1}^K \sum_{h=1}^H \Bigg\{\left( \zeta_h^k + \xi_h^k \right) + \sqrt{ \frac{L \text{Var}_{P_h^{\star}(\cdot|s,a)} \left[ V_{h+1}^{\pi^k,\sigma} \right]}{\{N_h^k(s_h^k, a_h^k) \vee 1\}} } + \frac{H^2 S(2L+1)}{\sqrt{\{N_h^k(s_h^k, a_h^k) \vee 1\}}} + \sqrt{\frac{1}{K}}\Bigg\}. \label{eq:Regret_step13}
\end{align}
where $C > 0$ is an constant.

\item \textbf{Step 2: Upper bound on the summation of variance terms.} To make progress, it suffices to upper bound the right-hand side of \eqref{eq:Regret_step13}. The main difficulty lies in handling the sum of the variance terms, which we now analyze carefully. Applying the Cauchy–Schwarz inequality to this summation, we get
\begin{align}
\label{eq:Regret_step14}
&\sum_{k=1}^K \sum_{h=1}^H \sqrt{ \frac{ \text{Var}_{P_h^{\star}(\cdot \mid s_h^k, a_h^k)} \left[ V_{h+1, {\bf P}^{\star}, {\bf \Phi}}^{\pi^k} \right] }{\{N_h^k(s_h^k, a_h^k) \vee 1\}}}\le \sqrt{\left( \sum_{k=1}^K \sum_{h=1}^H \text{Var}_{P_h^{\star}(\cdot \mid s_h^k, a_h^k)}\left[ V_{h+1}^{\pi^k,\sigma} \right] \right)\cdot 
\left( \sum_{k=1}^K \sum_{h=1}^H \frac{1}{\{N_h^k(s_h^k, a_h^k) \vee 1\}} \right)}.
\end{align}

According to \Cref{lem:inverse_count_bound}, we have
\begin{align}
\label{eq:Regret_step15}
     \sum_{k=1}^K \sum_{h=1}^H \frac{1}{\{N_h^k(s_h^k, a_h^k) \vee 1\}} \leq c_3HSA\log(K) \leq c_3HSAL,
\end{align}
where $c_3>0$ is an absolute constant, and $L=\log(S^2AH^2K^{3/2}/\delta)$.

Moreover, according to \Cref{lem:total_variance_bound}, with probability at least $1-\delta$, we have
\begin{align}
\label{eq:Regret_step16}
    \sum_{k=1}^{K} \sum_{h=1}^{H} \text{Var}_{P^\star_h(\cdot \mid s_h^k, a_h^k)} \left[ V_{h+1}^{\pi^k,\sigma} \right] 
\leq c_4 \cdot \left( H^3L + H^3(1+\sigma)K \right),
\end{align}
where $c_4$ is the absolute constant.

Combining \eqref{eq:Regret_step15} and \eqref{eq:Regret_step16} in \eqref{eq:Regret_step14}, we get
\begin{align}
\label{eq:Regret_step17}
   \sum_{k=1}^K \sum_{h=1}^H \sqrt{ \frac{ \text{Var}_{P_h^{\star}(\cdot \mid s_h^k, a_h^k)} \left[ V_{h+1}^{\pi^k,\sigma} \right] }{\{N_h^k(s_h^k, a_h^k) \vee 1\}}} \leq c_5\sqrt{H^4L^2SA + H^4L(1+\sigma)SAK},
\end{align}
where \( c_5 > 0 \) being another absolute constant.

\item \textbf{Step 3: Conclusion the proof.} Note that according to the definition in \eqref{eq:zeta_h_k} and \eqref{eq:xi_h_k}, both $\zeta^k_h$ and $\xi^k_h$ are bounded in the range $[0,2H]$. As a result, using Azuma-Hoeffding inequality in \Cref{lem:Azuma-Hoeffding}, with probability at least \(1-\delta\),
\begin{align}
\label{eq:Regret_step18}
    \sum_{k=1}^{K} \sum_{h=1}^{H} (\zeta_h^k + \xi_h^k)
\leq c_6 \sqrt{H^3KL},
\end{align}
where \( c_6 > 0 \) is an absolute constant. Therefore, applying \eqref{eq:Regret_step18}, \eqref{eq:Regret_step17}, and \eqref{eq:Regret_step15} in \eqref{eq:Regret_step13}, with probability at least \(1 - 3\delta\), we have
\begin{align}
\label{eq:Regret_step19}
    \text{Regret}(K)
&\leq C^{\prime\prime} \cdot \bigg(\sqrt{H^3KL} + \sqrt{H^4L^3SA + H^4L^2(1+\sigma)SAK}  + \sqrt{H^3S^3L^2A} + \sqrt{ H^3S^3LA} + \sqrt{H^2 K}\bigg) \nonumber\\
&= \mathcal{O} \left(\sqrt{H^4(1+\sigma)SAK\upsilon} \right),
\end{align}
where $C^{\prime\prime}$ is any constant and $\upsilon = \bigg(\log\Big(\frac{SAHK}{\delta}\Big)\bigg)^2$. 
\end{itemize}
This completes the proof of \Cref{thm:Regret_f_bound}. \qedhere
\end{proof}

%%%%%%%%%%%%%%%%%%%% Key Lemma %%%%%%%%%%%%%%%%%%%%%%%%%%

\subsection{Key Lemmas for {\RMDPchi}}
\label{subsubsec:Key_Lemma_chi}
\begin{keylem}[Optimistic and pessimistic estimation of the robust values for {\RMDPchi}]
\label{lem:Optimistic_pessimism}
\textit{By setting the bonus \( B^{\chi^2}_{k,h} \) as in \eqref{eq:Bonus_term_chi}, then under the typical event \( \mathcal{E}_{\chi^2} \), it holds that}
\begin{align}
\label{eq:Optimistic_pessimism_ineq}
\underline{Q}_h^k(s,a) \leq Q_{h}^{\pi^k,\sigma}(s,a) &\leq Q_{h}^{\star,\sigma}(s,a) \leq \overline{Q}_h^k(s,a),\nonumber\\
\underline{V}_h^k(s) \leq V_{h}^{\pi^k,\sigma}(s) &\leq V_{h}^{\star,\sigma}(s) \leq \overline{V}_h^k(s),
\end{align}
\textit{for any \( (s, a, h, k) \in \mathcal{S} \times \mathcal{A} \times [H] \times [K] \).}
\end{keylem}

\begin{proof}
We will prove \Cref{lem:Optimistic_pessimism} by induction and in three cases, as follows:
\begin{itemize}
    \item \textbf{Ineq. 1:} To prove $Q_{h}^{\star,\sigma}(s,a) \leq \overline{Q}_h^k(s,a)$.
    \item \textbf{Ineq.  2:} To prove $Q_{h}^{\pi^k,\sigma}(s,a) \leq Q_{h}^{\star,\sigma}(s,a)$.
    \item \textbf{Ineq. 3:} To prove $\underline{Q}_h^k(s,a) \leq Q_{h}^{\pi^k,\sigma}(s,a)$.
\end{itemize}
Let us consider that \eqref{eq:Optimistic_pessimism_ineq} holds at step $h+1$. 
\begin{itemize}
    \item \textbf{Proof of Ineq. 1:}
For step $h$, we will first consider the robust $Q$ function part. Specifically, by using the robust Bellman optimal equations (Eq. ~\eqref{eq:Robust_bellman_Q_fn} and \eqref{eq:Robust_bellman_V_fn}) and \eqref{eq:Upper_estimate_Q}, we have that
\begin{align}
Q_{h}^{\star,\sigma}(s,a) - \overline{Q}_h^k(s,a) &=  \max \left\{ \mathbb{E}_{\mathcal{U}^{\sigma}_h(s,a)} \left[ V_{h+1}^{\star,\sigma} \right]
- \mathbb{E}_{\widehat{\mathcal{U}^{\sigma}_h}(s,a)} \left[\overline{V}^k_{h+1} \right] - B^{\chi^2}_{k,h}(s,a), \,
Q_{h}^{\star,\sigma}(s,a) - H \right\} \nonumber\\
&\leq \max \left\{\mathbb{E}_{\mathcal{U}^{\sigma}_h(s,a)} \left[ V_{h+1}^{\star,\sigma} \right]
- \mathbb{E}_{\widehat{\mathcal{U}^{\sigma}_h}(s,a)} \left[V_{h+1}^{\star,\sigma} \right] - B^{\chi^2}_{k,h}(s,a), 0 \right\},\label{eq:optimism_pessimism_ineq_step1}
\end{align}
where the second inequality follows from the induction of $V_{h+1}^{\star,\sigma} \leq \overline{V}_{h+1}^k$ at step $h+1$ and the fact that $Q_{h}^{\star,\sigma} \leq H$. By Lemma \ref{lem:Bernstein_Bound_Chi_optimal_policy}, we have that
\begin{align}
\label{eq:optimism_pessimism_ineq_step2}
\mathbb{E}_{\mathcal{U}^{\sigma}_h(s,a)} \left[ V_{h+1}^{\star,\sigma} \right]
- &\mathbb{E}_{\widehat{\mathcal{U}^{\sigma}_h}(s,a)} \left[V_{h+1}^{\star,\sigma} \right]
\leq \sqrt{ \frac{ \sigma c_1L\text{Var}_{\hat{P}_h^k(\cdot|s,a)} \left[V_{h+1}^{\star,\sigma} \right] }{ \{N_h^k(s,a) \vee 1\}} }+ \frac{c_2 \sqrt{\sigma}HL}{\{N_h^k(s,a) \vee 1\}} + \sqrt{\frac{\sigma}{K}}.
\end{align}

Now by further applying Lemma ~\ref{lem:variance_analysis_1} to the variance term in the above inequality \eqref{eq:optimism_pessimism_ineq_step2}, we can obtain that
\begin{align}
\mathbb{E}_{\mathcal{U}^{\sigma}_h(s,a)} \left[ V_{h+1}^{\star,\sigma} \right]
- \mathbb{E}_{\widehat{\mathcal{U}^{\sigma}_h}(s,a)} \left[V_{h+1}^{\star,\sigma} \right]&\overset{(i)}{\leq} \sqrt{
\sigma c_1L\Bigg(\frac{ \text{Var}_{\hat{P}_h^k(\cdot|s,a)} \left[ (\overline{V}_{h+1}^k + \underline{V}_{h+1}^k)/2 \right] + 4H\left[\widehat{\mathbb{P}}^k_h \big(\overline{V}_{h+1}^k - \underline{V}_{h+1}^k\big) \right](s,a)}{\{N_h^k(s,a) \vee 1\}}\Bigg)}\nonumber\\
&\qquad \qquad \qquad + \frac{c_2 \sqrt{\sigma}HL}{\{N_h^k(s,a) \vee 1\}} + \sqrt{\frac{\sigma}{K}} \nonumber\\
&\overset{(ii)}{\leq} \sqrt{ \frac{ \sigma c_1L\text{Var}_{\hat{P}_h^k(\cdot|s,a)} \left[ (\overline{V}_{h+1}^k + \underline{V}_{h+1}^k)/2 \right]}{\{N_h^k(s,a) \vee 1\}} }
+ \frac{ \sqrt{\sigma}\left[\widehat{\mathbb{P}}^k_h \big(\overline{V}_{h+1}^k - \underline{V}_{h+1}^k\big) \right](s,a) }{H}\nonumber\\
&\qquad \qquad \qquad + \frac{c_2' \sqrt{\sigma}H^2L}{\{N_h^k(s,a) \vee 1\}} + \sqrt{\frac{\sigma}{K}}, \label{eq:optimism_pessimism_ineq_step3}
\end{align}
where (i) is due to Lemma ~\ref{lem:variance_analysis_1}, and the second inequality (ii) is due to the facts $\sqrt{a+b}\leq \sqrt{a}+\sqrt{b}$ and $\sqrt{ab} \leq a + b$. Note that $c_2' > 0$ is an absolute constant. Now recollect the choice of $ B^{\chi^2}_{k,h}$ as given in \eqref{eq:Bonus_term_chi}. Therefore, combining \eqref{eq:optimism_pessimism_ineq_step2}, \eqref{eq:optimism_pessimism_ineq_step3},  \eqref{eq:Bonus_term_chi} and the fact $S\geq 1$, we can conclude that
\begin{align}
\label{eq:optimism_pessimism_ineq_bound_case1}
Q_{h}^{\star,\sigma}(s,a) \leq \overline{Q}_h^k(s,a).
\end{align}

\item \textbf{Proof of Ineq. 2:} By the definition of $Q_{h}^{\star,\sigma}(s,a)$, the
\begin{align}
\label{eq:optimism_pessimism_ineq_bound_case2}
    Q_{h}^{\pi^k,\sigma}(s,a) \leq Q_{h}^{\star,\sigma}(s,a).
\end{align}
is trivial.

\item \textbf{Proof of Ineq. 3:} By using the robust Bellman equation (Eq. ~\eqref{eq:Robust_bellman_Q_fn} and \eqref{eq:Robust_bellman_V_fn}) and \eqref{eq:Lower_estimate_Q}, we have that
\begin{align}
\underline{Q}_h^k(s,a) - Q_{h}^{\pi^k,\sigma}(s,a) &= \max \left\{
\mathbb{E}_{\widehat{\mathcal{U}_h^{\sigma}}(s,a)} \left[ \underline{V}_{h+1}^k \right]
- \mathbb{E}_{\mathcal{U}_h^{\sigma}(s,a)} \left[ V_{h+1}^{\pi^k,\sigma} \right]- B^{\chi^2}_{k,h}(s,a), \,0 - Q_{h}^{\pi^k,\sigma}(s,a)
\right\} \notag\\
&\leq \max \left\{
\mathbb{E}_{\widehat{\mathcal{U}_h^{\sigma}}(s,a)} \left[ V_{h+1}^{\pi^k} \right]
- \mathbb{E}_{\mathcal{U}_h^{\sigma}(s,a)} \left[ V_{h+1}^{\pi^k,\sigma} \right]- B^{\chi^2}_{k,h}(s,a), \, 0\right\}, \label{eq:optimism_pessimism_ineq_step4}
\end{align}
where the second inequality follows from the induction of $\underline{V}_{h+1}^k \leq V_{h+1}^{\pi^k,\sigma}$ at step $h+1$ and the fact that $Q_{h}^{\pi^k,\sigma} \geq 0$. By Lemma ~\ref{lem:Bernstein_Bound_Chi_policy_k}, we get
\begin{align}
\mathbb{E}_{\widehat{\mathcal{U}_h^{\sigma}}(s,a)} \left[ V_{h+1}^{\pi^k,\sigma} \right]
- \mathbb{E}_{\mathcal{U}_h^{\sigma}(s,a)} \left[ V_{h+1}^{\pi^k,\sigma} \right] \leq \sqrt{ \frac{ \sigma c_1L\text{Var}_{\hat{P}_h^k(\cdot|s,a)} \left[ V_{h+1}^{\star,\sigma} \right]}{\{N_h^k(s,a) \vee 1\}} }+ \frac{c_2\sqrt{\sigma} H(2L+1)}{\sqrt{\{N_h^k(s,a) \vee 1\}}} + \sqrt{\frac{\sigma}{K}}.\label{eq:optimism_pessimism_ineq_step5}
\end{align}

Now by applying Lemma ~\ref{lem:variance_analysis_1} to the variance term in \eqref{eq:optimism_pessimism_ineq_step5}, with an argument similar to \eqref{eq:optimism_pessimism_ineq_step3}, we can obtain that
\begin{align}
\mathbb{E}_{\widehat{\mathcal{U}_h^{\sigma}}(s,a)} \left[ V_{h+1}^{\pi^k,\sigma} \right]
- \mathbb{E}_{\mathcal{U}_h^{\sigma}(s,a)} \left[ V_{h+1}^{\pi^k,\sigma} \right] &\leq \sqrt{ \sigma c_1L\frac{ \text{Var}_{\hat{P}_h^k(\cdot|s,a)} \left[ (\overline{V}_{h+1}^k + \underline{V}_{h+1}^k)/2 \right]}{\{N_h^k(s,a) \vee 1\}} } + \frac{ \sqrt{\sigma}\left[\widehat{\mathbb{P}}^k_h \big(\overline{V}_{h+1}^k - \underline{V}_{h+1}^k\big) \right](s,a) }{H}\nonumber\\
&\qquad \qquad + \frac{c_2' \sqrt{\sigma} H^2(2L+1)}{\sqrt{\{N_h^k(s,a) \vee 1\}}} + \sqrt{\frac{\sigma}{K}}. \label{eq:optimism_pessimism_ineq_step6}
\end{align}
Thus by combining \eqref{eq:optimism_pessimism_ineq_step4}, \eqref{eq:optimism_pessimism_ineq_step6}, the choice of $B^{\chi^2}_{k,h}(s,a)$ in \eqref{eq:Bonus_term_chi}, and $S\geq 1$, we get
\begin{align}
\label{eq:optimism_pessimism_ineq_bound_case3}
    Q_h^k(s,a) \leq Q_{h}^{\pi^k,\sigma}(s,a).
\end{align}
\end{itemize}
Therefore, by \eqref{eq:optimism_pessimism_ineq_bound_case1}, \eqref{eq:optimism_pessimism_ineq_bound_case2} and \eqref{eq:optimism_pessimism_ineq_bound_case3}, we have proved that at step $h$, it holds that
\begin{align}
\label{eq:optimism_pessimism_ineq_Qbound_final}
    \underline{Q}_h^k(s,a) \leq Q_{h}^{\pi^k,\sigma}(s,a) \leq Q_{h}^{\star,\sigma}(s,a) \leq \overline{Q}_h^k(s,a).
\end{align}
Finally for the robust $V$ function part, consider that by the robust Bellman equation (Eq. ~\eqref{eq:Robust_bellman_Q_fn} and \eqref{eq:Robust_bellman_V_fn}) and \eqref{eq:policy_value_epsiode_k}, 
\begin{align}
\label{eq:optimism_pessimism_ineq_step7}
\underline{V}_h^k(s) = \max_{a \in \mathcal{A}}  \underline{Q}_h^k(s,\cdot) \leq \max_{a \in \mathcal{A}} Q_{h}^{\pi^k,\sigma}(s,\cdot) = V_{h}^{\pi^k,\sigma}(s),
\end{align}
and that by the robust Bellman optimality (Eq. ~\eqref{cor:Robust_Bellman_Optimal_eq}), the choice of $\pi^k$, $\overline{V}^k_h$ and $\underline{V}^k_h$ defined in \eqref{eq:policy_value_epsiode_k},
\begin{align}
\label{eq:optimism_pessimism_ineq_step8}
    V_{h}^{\star,\sigma}(s) = \max_{a \in \mathcal{A}} Q_{h}^{\star,\sigma}(s,a) \leq \max_{a \in \mathcal{A}} \overline{Q}_h^k(s,a) = \overline{V}_h^k(s),
\end{align}
which proves that
\begin{align}
\label{eq:optimism_pessimism_ineq_Vbound_final}
    \underline{V}_h^k(s) \leq V_{h}^{\pi^k,\sigma}(s) \leq V_{h}^{\star,\sigma}(s) \leq \overline{V}_h^k(s).
\end{align}
Since the conclusion \eqref{eq:Optimistic_pessimism_ineq} holds for the $V$ function part at step $H+1$, an induction proves Lemma ~\ref{lem:Optimistic_pessimism}. \qedhere
\end{proof}
%%%%%%%%%%%%%%%%%%%%%%%

\begin{keylem}[Proper bonus for {\RMDPchi} and optimistic and pessimistic value estimators] 
\label{lem:Proper_bouns_optimism_pessimism_bound}
By setting the bonus $ B^{\chi^2}_{k,h}$ as in \eqref{eq:Bonus_term_chi}, then under the typical event $\mathcal{E}_{\chi^2}$, it holds that
\begin{align}
\label{eq:Proper_bouns_optimism_pessimism_bound}
\mathbb{E}_{\widehat{\mathcal{U}_h^\sigma}(s,a)}\left[\overline{V}_{h+1}^k\right] -\mathbb{E}_{\mathcal{U}_h^\sigma(s,a)}\left[\overline{V}_{h+1}^k\right] + \mathbb{E}_{\mathcal{U}_h^\sigma(s,a)}\left[\underline{V}_{h+1}^k\right] -\mathbb{E}_{\widehat{\mathcal{U}_h^\sigma}(s,a)}\left[\underline{V}_{h+1}^k\right] \leq 2B^{\chi^2}_{k,h}(s,a).
\end{align}
\end{keylem}

\begin{proof}
    Let us denote
    \begin{align}
    \label{eq:A}
        A:= \mathbb{E}_{\widehat{\mathcal{U}_h^\sigma}(s,a)}\left[\overline{V}_{h+1}^k\right] -\mathbb{E}_{\mathcal{U}_h^\sigma(s,a)}\left[\overline{V}_{h+1}^k\right] + \mathbb{E}_{\mathcal{U}_h^\sigma(s,a)}\left[\underline{V}_{h+1}^k\right] -\mathbb{E}_{\widehat{\mathcal{U}_h^\sigma}(s,a)}\left[\underline{V}_{h+1}^k\right]
    \end{align}

We upper bound $A$ by using the concentration inequality given in \Cref{lem:Bernstein_Bound_Chi_optimism_pessimism},
\begin{align}
A \leq 2 \sqrt{\frac{\sigma c_1L\text{Var}_{\widehat{P}_h^k(\cdot|s,a)}\left[ V_{h+1}^{\star,\sigma}\right]}{\{N_h^k(s,a) \vee 1\}}} + \frac{2c_2 \sqrt{\sigma}H(2L+1)}{\sqrt{\{N_h^k(s,a) \vee 1\}}}
+ 2\sqrt{\frac{\sigma}{K}}, \label{eq:Proper_bouns_optimism_pessimism_bound_step1}
\end{align}
where \( c_1, c_2 > 0 \) are absolute constants. Then applying \Cref{lem:variance_analysis_1} to the variance term in \eqref{eq:Proper_bouns_optimism_pessimism_bound_step1}, with an argument the same as \eqref{eq:optimism_pessimism_ineq_step3} in the proof of \Cref{lem:Optimistic_pessimism}, we can obtain that
\begin{align}
A &\leq 2 \sqrt{\frac{\sigma c_1L\text{Var}_{\widehat{P}_h^k(\cdot|s,a)}\left[\left( \overline{V}^k_{h+1} + \underline{V}^k_{h+1} \right)/2 \right]}{\{N_h^k(s,a) \vee 1\}}} + \frac{4\sqrt{\sigma}\left[\widehat{\mathbb{P}}^k_h \big(\overline{V}_{h+1}^k - \underline{V}_{h+1}^k\big)\right](s,a)}{H} + \frac{2c_2\sqrt{\sigma}H^2S(2L+1)}{\sqrt{\{N_h^k(s,a) \vee 1\}}}
+ 2\sqrt{\frac{\sigma}{K}}.\label{eq:Proper_bouns_optimism_pessimism_bound_step2}
\end{align}

Therefore, by the choice of \(B^{\chi^2}_{k,h}(s,a)\) in \eqref{eq:Bonus_term_chi}, we get \eqref{eq:Proper_bouns_optimism_pessimism_bound}. This concludes the proof of \Cref{lem:Proper_bouns_optimism_pessimism_bound}.
\end{proof}
%%%%%%%%%%%%%%%%%%%%%%%

\begin{keylem}[Control of the bonus term for {\RMDPchi}]
\label{lem:Control_Bonus}
\textit{Under the typical event \( \mathcal{E}_{\chi^2} \), the bonus term \(  B^{\chi^2}_{k,h} \) in \eqref{eq:Bonus_term_chi} is bounded by}
\begin{align}
    B^{\chi^2}_{k,h}(s,a) \leq &\sqrt{\frac{\sigma c_1L
\text{Var}_{P_h^\star(\cdot|s,a)}\left[ V_{h+1}^{\pi^k,\sigma} \right]}{\{N_h^k(s,a) \vee 1\}}} + \frac{4 \sqrt{\sigma}\left[ \mathbb{P}^{\star}_h\big(\overline{V}_{h+1}^k - \underline{V}_{h+1}^k\big) \right](s,a)}{H} + \frac{c_2\sqrt{\sigma} H^2 S(2L+1)}{\sqrt{\{N_h^k(s,a) \vee 1\}}}
+ \sqrt{\frac{\sigma}{K}},
\end{align}
where \( L = \log(S^3 A H^2 K^{3/2} / \delta) \) and \( c_1, c_2 > 0 \) are constants.
\end{keylem}
\begin{proof}
Incorporating \Cref{lem:non_robust_conc} and \Cref{lem:variance_analysis_2} into the proof framework of Lemma E.4 from \cite{Arxiv2024_DRORLwithInteractiveData_Lu}, we derive the required bound.
\end{proof}

\begin{keylem}[Total variance law for {\RMDPchi}]
\label{lem:total_variance_bound}
    With probability at least \(1 - \delta\), the following inequality holds
\begin{align}
\label{eq:total_variance_bound}
    \sum_{k=1}^{K} \sum_{h=1}^{H} \text{Var}_{P^\star_h(\cdot \mid s_h^k, a_h^k)} \left[ V_{h+1}^{\pi^k,\sigma} \right] 
\leq c_3 \cdot \left( H^3L + H^3(1+\sigma)K \right).
\end{align}
where \(L = \log \left( \frac{S^3 A H^2 K^{3/2}}{\delta} \right)\) and \(c_3 > 0\) is an absolute constant.
\end{keylem}
\begin{proof}
We adapted the proof lines of  \cite[Lemma E.5]{Arxiv2024_DRORLwithInteractiveData_Lu}. For any policy $\pi$ and any step $h$, the robust value function of $\pi$ holds that
\begin{align}
    \max_{s \in \mathcal{S}} V_{h}^{\pi,\sigma}(s) \leq H,
\end{align}
which we will apply in the following proof-lines. We now define
\begin{align}
    \widetilde{T}_h^k(\cdot \mid s, a) = \arg\min_{P \in \mathcal{U}_h^{\sigma}(s,a)} \left[ \mathbb{P}V_{h+1}^{\pi^k,\sigma} \right](s,a), \quad \forall (h,s,a) \in [H]\times \mathcal{S} \times \mathcal{A},
\end{align}
and set $\widetilde{T}^k = \{ \widetilde{T}_h^k \}_{h=1}^H$, which is the most adversarial transition for the true robust value function of $\pi^k$. Now consider the following decomposition of our objective,
\begin{align}
    \sum_{k=1}^{K} \sum_{h=1}^{H} \text{Var}_{P_h^\star(\cdot | s_h^k, a_h^k)} \left[ V_{h+1}^{\pi^k,\sigma} \right] 
    =& \sum_{k=1}^{K} \sum_{h=1}^{H} \text{Var}_{P_h^\star(\cdot | s_h^k, a_h^k)} \left[ V_{h+1}^{\pi^k,\sigma} \right] 
    - \mathbb{E}_{(s_h^k, a_h^k) \sim (P^\star, \pi^k)} \left[ \sum_{h=1}^{H} \text{Var}_{P_h^\star(\cdot | s_h^k, a_h^k)} \left[ V_{h+1}^{\pi^k,\sigma} \right] \right] \nonumber\\
    &\quad +  \mathbb{E}_{(s_h^k, a_h^k) \sim (P^\star, \pi^k)} \left[ \sum_{h=1}^{H} \text{Var}_{P_h^\star(\cdot | s_h^k, a_h^k)} \left[ V_{h+1}^{\pi^k,\sigma} \right] \right]  \nonumber\\
    =& \text{Term (i) } + \text{ Term (ii) } + \text{ Term (iii)},\label{eq:total_variance_bound_step1}
\end{align}
where we denote
\begin{align}
    \text{Term (i)} &:=  \sum_{k=1}^{K} \sum_{h=1}^{H} \text{Var}_{P_h^\star(\cdot | s_h^k, a_h^k)} \left[ V_{h+1}^{\pi^k,\sigma} \right] 
    - \mathbb{E}_{(s_h^k, a_h^k) \sim (P^\star, \pi^k)} \left[ \sum_{h=1}^{H} \text{Var}_{P_h^\star(\cdot | s_h^k, a_h^k)} \left[ V_{h+1}^{\pi^k,\sigma} \right] \right] \label{eq:total_variance_bound_term1}\\
    \text{Term (ii)} &:= \sum_{k=1}^{K} \mathbb{E}_{(s_h^k, a_h^k) \sim (\widetilde{T}^k, \pi^k)} \left[ \sum_{h=1}^{H} \text{Var}_{\widetilde{T}_h^k(\cdot | s_h^k, a_h^k)} \left[ V_{h+1}^{\pi^k,\sigma} \right] \right]\label{eq:total_variance_bound_term2}\\
    \text{Term (iii)} &:= \sum_{k=1}^{K} \Bigg( \mathbb{E}_{(s_h^k, a_h^k) \sim (P^\star, \pi^k)} \left[ \sum_{h=1}^{H} \text{Var}_{P_h^\star(\cdot | s_h^k, a_h^k)} \left[ V_{h+1}^{\pi^k,\sigma} \right]\right]- \mathbb{E}_{(s_h^k, a_h^k) \sim (\widetilde{T}^k, \pi^k)} \left[ \sum_{h=1}^{H} \text{Var}_{\widetilde{T}_h^k(\cdot | s_h^k, a_h^k)} \left[ V_{h+1}^{\pi^k,\sigma} \right] \right] \Bigg)\label{eq:total_variance_bound_term3}
\end{align}
We will now upper bound each term separately.
\begin{itemize}
    \item \textbf{Bound of Term (i):} Note that Term (i) is the summation of the martingale difference with respect to the filtration $\mathcal{G}_k= \sigma\Big(\Big\{\big(s^{\tau}_h, a^{\tau}_h\big)\Big\}_{(h,\tau)\in [H]\times[K]}\Big)$. Therefore, by applying the Azuma-Hoeffding's inequality as given in \Cref{lem:Azuma-Hoeffding}, with probability of at least $1-\delta$, we get
    \begin{align}
    \label{eq:total_variance_bound_term1_bound}
        \text{Term (i)} \leq g_1 H^3\sqrt{KL},
    \end{align}
    where $g_1>0$ is an absolute constant.

    \item \textbf{Bound of Term (ii):} By applying \Cref{lem:TOtal_variance_law} for policy $\pi^k$ for $k \in [K]$ (which are deterministic policies), we get
    \begin{align}
        \label{eq:total_variance_bound_term2_bound}
        \text{Term (ii)} \leq 2H^2K.
    \end{align}

     \item \textbf{Bound of Term (iii):} By the definition of $\widetilde{T}^k_h$, it holds that
     \begin{align}
     \label{eq:total_variance_bound_term3_step1}
         \widetilde{T}^k_h \in \mathcal{U}_h^{\sigma}(s,a) \implies D_{\chi^2}\big(\widetilde{T}^k_h (\cdot|s,a)||P^{\star}_h(\cdot|s,a)\big)\leq \sigma,
     \end{align}
and, we have the fact
\begin{align}
\label{eq:total_variance_bound_term3_step2}
\text{Var}_{P_h^\star(\cdot | s_h^k, a_h^k)} \left[ V_{h+1}^{\pi^k,\sigma} \right]\leq H^2, \forall (h,s,a) \in [H]\times \mathcal{S}\times \mathcal{A}.
\end{align}

Using \eqref{eq:total_variance_bound_term3_step1}, \eqref{eq:total_variance_bound_term3_step2}, and following the same lines of proof (E.26-E.30) in \cite{Arxiv2024_DRORLwithInteractiveData_Lu}, we get
\begin{align}
\label{eq:total_variance_bound_term3_bound}
     \text{Term (iii)} \leq 3\sigma H^3K. %\leq 3HG^2K, \quad \text{ since $\sigma \leq 1$.}
\end{align}
\end{itemize}
Therefore, combining the upper bounds of Term (i), Term (ii) and Term (iii) as given in \eqref{eq:total_variance_bound_term1_bound}, \eqref{eq:total_variance_bound_term2_bound} and \eqref{eq:total_variance_bound_term3_bound} in \eqref{eq:total_variance_bound_step1}, respectively, we get
\begin{align*}
    \sum_{k=1}^{K} \sum_{h=1}^{H} \text{Var}_{P_h^\star(\cdot | s_h^k, a_h^k)} \left[ V_{h+1}^{\pi^k,\sigma} \right] &\leq g_1H^3\sqrt{KL} + 2H^2K + 3\sigma H^3K\nonumber\\
    &\leq g_1H^3L +g_2H^3(1+\sigma)K,
\end{align*}
where the last inequality is obtained by using $\sqrt{ab} \leq a+b$ for any $a, b > 0$, and $H\geq 1$. This concludes the proof of \Cref{lem:total_variance_bound}.
\end{proof}

\subsection{Auxiliary Lemmas for {\RMDPchi}}
\label{subsubsec:aux_Lemma_chi}
%%%%%%%%%%%%%%%%%%%%%%%%%%%%%

\begin{auxlem}[Bernstein Bound for {\RMDPchi} and Optimal Robust Value function]
\label{lem:Bernstein_Bound_Chi_optimal_policy}
Under event \( \mathcal{E}_{\chi^2} \) defined in \eqref{eq:Event}, with probability at least $1-\delta$, it holds that
\begin{align}
\label{eq:Bernstein_Bound_Chi_optimal_policy}
\abs{ \mathbb{E}_{\widehat{\mathcal{U}_h^\sigma}(s,a)}\left[V_{h+1}^{\star,\sigma}\right] -\mathbb{E}_{\mathcal{U}_h^\sigma(s,a)}\left[V_{h+1}^{\star,\sigma}\right]}\leq \sqrt{\frac{\sigma c'_1L\text{Var}_{\widehat{P}^k_h(\cdot|s,a)}(V^{\star,\sigma}_{h+1})}{\{N^k_h(s,a)\vee 1\}}} + \frac{c'_2\sqrt{\sigma}HL}{\{N^k_h(s,a)\vee 1\}} + \sqrt{\frac{\sigma}{K}},
\end{align}
where \(L = \log(S^3 A H^2 K^{3/2} / \delta) \) and \( c'_1, c'_2 \) are absolute constants.
\end{auxlem}

\begin{proof} 
By our definition of the operator \(\mathbb{E}_{\widehat{\mathcal{U}_h^\sigma}(s,a)}\left[V_{h+1}^{\star,\sigma}\right]\) in \eqref{eq:dual_chi}, we can arrive at
\begin{align}
\abs{ \mathbb{E}_{\widehat{\mathcal{U}_h^\sigma}(s,a)}\left[V_{h+1}^{\star,\sigma}\right] 
- \mathbb{E}_{\mathcal{U}_h^\sigma(s,a)}\left[V_{h+1}^{\star,\sigma}\right]}&= \bigg| 
\sup_{\eta \in [0,H]} \left\{ 
- \sqrt{\sigma\text{Var}_{\widehat{P}_h^k(\cdot | s,a)}\left[(\eta - V_{h+1}^{\star,\sigma})_+\right]}
+ \left[\widehat{\mathbb{P}}^k_h\Big(\eta - V_{h+1}^{\star,\sigma} \Big) \right](s,a) + \eta \right\}  \nonumber\\
&\qquad - \sup_{\eta \in [0,H]} \left\{ 
- \sqrt{\sigma\text{Var}_{P_h^\star(\cdot | s,a)}\left[(\eta - V_{h+1}^{\star,\sigma})_+\right]} + \left[{\mathbb{P}}^\star_h\Big(\eta - V_{h+1}^{\star,\sigma} \Big) \right](s,a)
+ \eta \right\}\bigg| \nonumber \\
&\leq 
\sqrt{\sigma}\sup_{\eta \in [0,H]} \abs{
\sqrt{\text{Var}_{\widehat{P}_h^k(\cdot | s,a)}\left[(\eta - V_{h+1}^{\star,\sigma})_+\right]} - \sqrt{\text{Var}_{P_h^\star(\cdot | s,a)}\left[(\eta - V_{h+1}^{\star,\sigma})_+\right]}} \nonumber\\
&\qquad + \sup_{\eta\in [0,H]}\abs{\left[\Big(\widehat{\mathbb{P}}^k_h - \mathbb{P}^\star_h\Big)\Big(\eta - V_{h+1}^{\star,\sigma} \Big) \right](s,a) }. 
\end{align}

By the definition of $\mathcal{E}_{\chi^2}$ as defined in \eqref{eq:event_I} and \cite[Lemma 7]{NeuRIPS2024_UnifiedPessimismOfflineRL_Yue}, we have 
\begin{align}
&\abs{
\sqrt{\sigma \text{Var}_{\widehat{P}_h^k(\cdot | s,a)}\left[(\eta - V_{h+1}^{\star,\sigma})_+\right]} - \sqrt{\sigma \text{Var}_{P_h^\star(\cdot | s,a)}\left[(\eta - V_{h+1}^{\star,\sigma})_+\right]}} + \abs{\left[\Big(\widehat{\mathbb{P}}^k_h - \mathbb{P}^\star_h\Big)\Big(\eta - V_{h+1}^{\star,\sigma} \Big) \right](s,a) }\nonumber\\
% &\quad \stackrel{(\mathrm{a})}{\leq} \sqrt{\frac{c_1L\text{Var}_{\widehat{P}^k_h(\cdot|s,a)}(\eta - V^{\star,\sigma}_{h+1})_+}{\{N^k_h(s,a)\vee 1\}}} + \frac{c_2HL}{\{N^k_h(s,a)\vee 1\}}  \nonumber \\
&\qquad \qquad \leq  \sqrt{\frac{\sigma c'_1L\text{Var}_{\widehat{P}^k_h(\cdot|s,a)}(V^{\star,\sigma}_{h+1})}{\{N^k_h(s,a)\vee 1\}}} + \frac{c'_2\sqrt{\sigma}HL}{\{N^k_h(s,a)\vee 1\}}, \label{eq:Wang_2}
\end{align}
for any $\eta \in \mathcal{N}_{1/S\sqrt{K}}([0,H])$, and $c'_1 = 4c_1$ and $c'_2 = 2c_2$ are absolute constants. %The first inequality (a) is obtained by applying Lemma 7 of \cite{NeuRIPS2024_UnifiedPessimismOfflineRL_Yue},and the inequality (b) is obtained by $\text{Var}[(a-X)_+] \leq \text{Var}(X)$. 
Therefore, by
a covering argument, for any $\eta \in [0,H]$, we get \eqref{eq:Bernstein_Bound_Chi_optimal_policy}. This concludes the proof of Lemma \ref{lem:Bernstein_Bound_Chi_optimal_policy}.
\end{proof}
%%%%%%%%%%%%%%%%%%%%%%%%%

\begin{auxlem}[Bernstein bound for {\RMDPchi} and the robust value function of \( \pi^k \)]
\label{lem:Bernstein_Bound_Chi_policy_k}
Under event \( \mathcal{E}_{\chi^2} \) in \eqref{eq:Event}, suppose that the optimism and pessimism ineq. \eqref{eq:Optimistic_pessimism_ineq} hold at \( (h+1, k) \), then it holds that
\begin{align}
\abs{
\mathbb{E}_{\widehat{\mathcal{U}_h^\sigma}(s,a)} 
\left[ V_{h+1}^{\pi^k,\sigma} \right] - \mathbb{E}_{\mathcal{U}_h^\sigma(s,a)} 
\left[ V_{h+1}^{\pi^k,\sigma} \right]}\leq \sqrt{\frac{\sigma c'_1L\text{Var}_{\widehat{P}^k_h(\cdot|s,a)}(V^{\star,\sigma}_{h+1})}{\{N^k_h(s,a)\vee 1\}}}+ \frac{c_2'\sqrt{\sigma}H(2L+1)}{\sqrt{\{N^k_h(s,a)\vee 1\}}} + \sqrt{\frac{\sigma}{K}},
\end{align}
where \(L = \log(S^3 A H^2 K^{3/2} / \delta) \) and \( c'_1, c_2' \) are absolute constants.
\end{auxlem}

\begin{proof}
By the definition of \( \mathbb{E}_{\mathcal{U}_h^\sigma(s,a)}[V_{h+1}^{\pi^k,\sigma}] \) in~\eqref{eq:dual_chi}, we can arrive at:
\begin{align}
A:= \abs{ 
\mathbb{E}_{\widehat{\mathcal{U}_h^\sigma}(s,a)} \left[ V_{h+1}^{\pi^k,\sigma} \right] 
- \mathbb{E}_{\mathcal{U}_h^\sigma(s,a)} \left[ V_{h+1}^{\pi^k,\sigma} \right]}&\leq \sqrt{\sigma}\sup_{\eta \in [0,H]} \abs{
\sqrt{\text{Var}_{\widehat{P}_h^k(\cdot | s,a)}\left[(\eta - V_{h+1}^{\pi^k,\sigma})_+\right]} - \sqrt{\text{Var}_{P_h^\star(\cdot | s,a)}\left[(\eta - V_{h+1}^{\pi^k,\sigma})_+\right]}}\nonumber\\
&\qquad + \sup_{\eta\in [0,H]}\abs{\left[\Big(\widehat{\mathbb{P}}^k_h - \mathbb{P}^\star_h\Big)\Big((\eta - V_{h+1}^{\pi^k,\sigma})_+\Big) \right](s,a) }. 
%&\leq \sqrt{\sigma}\bigg\{\text{ Term (i) } + \text{ Term (ii)}\bigg\}, 
\label{eq:Bernstein_bound_step1}
\end{align}
By the definition of $\mathcal{E}_{\chi^2}$ as defined in \eqref{eq:event_I} and Lemma ~\ref{lem:Bernstein_Bound_Chi_optimal_policy}, we can upper bound \eqref{eq:Bernstein_bound_step1} as
\begin{align}
A \leq &\sqrt{\frac{\sigma c'_1L\text{Var}_{\widehat{P}^k_h(\cdot|s,a)}(V^{\pi^k,\sigma}_{h+1})}{\{N^k_h(s,a)\vee 1\}}} + \frac{c'_2\sqrt{\sigma}HL}{\{N^k_h(s,a)\vee 1\}} + \sqrt{\frac{\sigma}{K}}\nonumber\\
&=\sqrt{\frac{\sigma c'_1L\bigg(\text{Var}_{\widehat{P}^k_h(\cdot|s,a)}(V^{\star,\sigma}_{h+1}) + \text{Var}_{\widehat{P}^k_h(\cdot|s,a)}(V^{\pi^k,\sigma}_{h+1})-\text{Var}_{\widehat{P}^k_h(\cdot|s,a)}(V^{\star,\sigma}_{h+1}) \bigg)}{\{N^k_h(s,a)\vee 1\}}} + \frac{c'_2\sqrt{\sigma}HL}{\{N^k_h(s,a)\vee 1\}} + \sqrt{\frac{\sigma}{K}}\nonumber\\
&\stackrel{(\mathrm{a})}{\leq}  \sqrt{\frac{\sigma c'_1L\text{Var}_{\widehat{P}^k_h(\cdot|s,a)}(V^{\star,\sigma}_{h+1})}{\{N^k_h(s,a)\vee 1\}}} + \sqrt{\frac{\sigma c'_1L\abs{\text{Var}_{\widehat{P}^k_h(\cdot|s,a)}(V^{\pi^k,\sigma}_{h+1})-\text{Var}_{\widehat{P}^k_h(\cdot|s,a)}(V^{\star,\sigma}_{h+1})}}{\{N^k_h(s,a)\vee 1\}}} + \frac{c'_2\sqrt{\sigma}HL}{\{N^k_h(s,a)\vee 1\}} + \sqrt{\frac{\sigma}{K}},\nonumber\\
&\stackrel{(\mathrm{b})}{\leq}  \sqrt{\frac{\sigma c'_1L\text{Var}_{\widehat{P}^k_h(\cdot|s,a)}(V^{\star,\sigma}_{h+1})}{\{N^k_h(s,a)\vee 1\}}} + \sqrt{\frac{4\sigma c'_1H^2L}{\{N^k_h(s,a)\vee 1\}}}+ \frac{c'_2\sqrt{\sigma}HL}{\{N^k_h(s,a)\vee 1\}} + \sqrt{\frac{\sigma}{K}},\label{eq:Term_bound}
\end{align}
where the inequality (a)  is obtained by using the fact $\sqrt{a+b}\leq \sqrt{a} + \sqrt{b}$, and the last inequality (b) is from
\begin{align*}
    \abs{\text{Var}_q(V_1)- \text{Var}_q(V_2)} &= \abs{\mathbb{E}\Big[(V_1+V_2)(V_1-V_2)\Big]} + \abs{\Big(\mathbb{E}[V_1] + \mathbb{E}[V_2]\Big)\Big(\mathbb{E}[V_1] - \mathbb{E}[V_2]\Big)]}\nonumber\\
    &\leq 4H\abs{\mathbb{E}\big[V_1-V_2\big]}\leq 4H^2.
\end{align*}
Finally, by applying the fact that $\sqrt{x} \leq x$ for $x>=1$ and $\sqrt{ab}\leq a+b$ in \eqref{eq:Term_bound}, we conclude the proof of \Cref{lem:Bernstein_Bound_Chi_policy_k}.\qedhere
\end{proof}

%%%%%%%%%%%%%%%%%%%%%%%%

\begin{auxlem}[Bernstein bounds for {\RMDPchi} and optimistic and pessimistic robust value estimators]
\label{lem:Bernstein_Bound_Chi_optimism_pessimism}
Under event $\mathcal{E}^{\chi^2}$ in \eqref{eq:Event}, suppose that the optimism and pessimism ineq. \eqref{eq:Optimistic_pessimism_ineq} holds at $(h+1,k)$, it holds that
\begin{align}
&\max \Bigg\{\left|\mathbb{E}_{\widehat{\mathcal{U}_h^\sigma}(s,a)}\left[\overline{V}_{h+1}^k\right] - \mathbb{E}_{\mathcal{U}_h^\sigma(s,a)}\left[\overline{V}_{h+1}^k\right] \right|,
\left| \mathbb{E}_{\widehat{\mathcal{U}_h^\sigma}(s,a)}\left[\underline{V}_{h+1}^k\right] - \mathbb{E}_{\mathcal{U}_h^\sigma(s,a)}\left[\underline{V}_{h+1}^k\right] \right|
\Bigg\}\nonumber\\
&\qquad \qquad \leq \sqrt{\frac{\sigma c_1L\text{Var}_{\widehat{P}^k_h(\cdot|s,a)}(V^{\star,\sigma}_{h+1})}{\{N^k_h(s,a)\vee 1\}}} + \frac{c_2'\sqrt{\sigma}H(2L+1)}{\sqrt{\{N^k_h(s,a)\vee 1\}}} + \sqrt{\frac{\sigma}{K}},
\end{align}
where $L= \log(S^3 A H^2 K^{3/2} / \delta)$ and $c_1, c_2'$ are absolute constants.
\end{auxlem}

\begin{proof}
    We follow the same proof lines as Lemma ~\ref{lem:Bernstein_Bound_Chi_policy_k}, and thereby we omit it.
\end{proof}
%%%%%%%%%%%%%%%%%%%%%%%

\begin{auxlem}[Variance analysis 1]
\label{lem:variance_analysis_1}
Suppose that the optimism and pessimism ineq. \eqref{eq:Optimistic_pessimism_ineq} holds at \( (h+1, k) \), 
\textit{then the following inequality holds,}
\begin{align}
\abs{ \text{Var}_{\widehat{P}_h^k(\cdot|s,a)} 
\left[ \left( \frac{ \overline{V}_{h+1}^k + \underline{V}_{h+1}^k }{2} \right) \right]-\text{Var}_{\widehat{P}_h^k(\cdot|s,a)} 
\left[ V_{h+1}^{\star,\sigma} \right]}\leq 4H\left[\widehat{\mathbb{P}}_h^k\left( \overline{V}_{h+1}^k - \underline{V}_{h+1}^k \right)\right](s,a).
\end{align}
\end{auxlem}

\begin{proof}
    Refer to the proof-lines of Lemma E.11 in \cite{Arxiv2024_DRORLwithInteractiveData_Lu}.
\end{proof}
%%%%%%%%%%%%%%%%%%%%%%%%

\begin{auxlem}[Variance analysis 2]
\label{lem:variance_analysis_2}
\textit{Under event \( \mathcal{E}_{\chi^2} \) in \eqref{eq:Event}, suppose that optimism and pessimism ineq. \eqref{eq:Optimistic_pessimism_ineq} hold at \( (h+1, k) \), then it holds that}
\begin{align}
\abs{\text{Var}_{\widehat{P}_h^k(\cdot|s,a)} 
\left[ \left( \frac{ \overline{V}_{h+1}^k + \underline{V}_{h+1}^k}{2} \right) \right]-\text{Var}_{P_h^\star(\cdot|s,a)} \left[ V_{h+1}^{\pi^k,\sigma} \right]}
&\leq 4H \left[\mathbb{P}^\star_h\big(\overline{V}_{h+1}^k - V_{h+1}^k\big) \right](s,a) + \frac{c_2' H^4 SL}{\{N_h^k(s,a) \vee 1\}} + 1.
\end{align}
\end{auxlem}

\begin{proof}
    Refer to the proof-lines of Lemma E.12 of \cite{Arxiv2024_DRORLwithInteractiveData_Lu}.
\end{proof}

\begin{auxlem}[Total Variance law]
    \label{lem:TOtal_variance_law}
    For any deterministic policy $\pi$ that satisfies the setup of \Cref{thm:Regret_f_bound} for {\RMDPchi}, we define \begin{align}
    \label{eq:min_transitionprob_uncertainty}
            \widetilde{T}_h(\cdot \mid s, a) := \arg\min_{P\in \mathcal{U}_h^{\sigma}(s,a)} \left[\mathbb{P} V_{h+1}^{\pi,\sigma} \right](s,a), \quad \forall (s,a,h) \in \mathcal{S} \times \mathcal{A} \times [H],
\end{align}
and set $\widetilde{\bf T} = \{ \widetilde{T}_h \}_{h=1}^H$, which is the most adversarial transition for the true robust value function of $\pi$.   Therefore,  we have
\begin{align}
    \mathbb{E}_{(s, a) \sim (\widetilde{\bf T}, \pi)} \left[ \sum_{h=1}^{H} \text{Var}_{\widetilde{T}_h(\cdot |s,a)} \left[ V_{h+1}^{\pi,\sigma} \right]\right] \leq 2H^2.
\end{align}
\end{auxlem}
\begin{proof}
    For $\max_{s \in \mathcal{S}} V_{h}^{\pi,\sigma}(s) \leq H$. Using this fact and applying in the proof of Lemma E.6 in \cite{Arxiv2024_DRORLwithInteractiveData_Lu}, we conclude the proof of \Cref{lem:TOtal_variance_law}.\qedhere
\end{proof}
%%%%%%%%%%%%%%%%%%%%%%%%

\begin{auxlem}[Non-robust concentration]
\label{lem:non_robust_conc}
Under event $\mathcal{E}_{\chi^2}$ in \eqref{eq:Event}, suppose that the optimism and pessimism \eqref{eq:Optimistic_pessimism_ineq} holds at $(h+1,k)$, then it holds that
\begin{align*}
\abs{\left[  \left( \widehat{\mathbb{P}}^k_h - \mathbb{P}^\star_h \right)\left(\overline{V}_{h+1}^k - \underline{V}_{h+1}^k \right)\right](s,a)}
\leq \frac{1}{H} \cdot \left[\mathbb{P}^\star_h\big( \overline{V}_{h+1}^k - \underline{V}_{h+1}^k\big)\right](s,a) + \frac{c_2' H^2 SL}{\{N_h^k(s,a) \vee 1\}},
\end{align*}
where $c_2'>0$ is an absolute constant.
\end{auxlem}

\begin{proof}
According to the second inequality of event $\mathcal{E}_{\chi^2}$, we have that
\begin{align*}
\abs{\left[  \left( \widehat{\mathbb{P}}^k_h - \mathbb{P}^\star_h \right)\left(\overline{V}_{h+1}^k - \underline{V}_{h+1}^k \right)\right](s,a)} &\leq \sum_{s' \in \mathcal{S}} \left( \sqrt{ \frac{P^\star_h(s'|s,a) c_1 L}{\{N_h^k(s,a) \vee 1\}} } + \frac{c_2L}{\{N_h^k(s,a) \vee 1\}} \right)
\cdot \left( \overline{V}_{h+1}^k(s')- \underline{V}_{h+1}^k(s')\right),
\end{align*}
where we also apply \eqref{eq:Optimistic_pessimism_ineq} that $\overline{V}_{h+1}^k(s') \geq \underline{V}_{h+1}^k(s')$. Now using the argument that $\sqrt{ab} \leq a+b$, we can arrive at
\begin{align*}
\abs{\left[  \left( \widehat{\mathbb{P}}^k_h - \mathbb{P}^\star_h \right)\left(\overline{V}_{h+1}^k - \underline{V}_{h+1}^k \right)\right](s,a)}
\leq \frac{1}{H}\left[\mathbb{P}^\star_h\big( \overline{V}_{h+1}^k - \underline{V}_{h+1}^k\big)\right](s,a) + \frac{c_2' H^2 SL}{\{N_h^k(s,a) \vee 1\}},
\end{align*}
which finishes the proof of Lemma ~\ref{lem:non_robust_conc}.
\end{proof}

\subsection{Proof of regret bound of {\AlgonameKL}}
\label{app:thm:Regret_KL_bound}
Before starting, we introduce some notations that will be useful in the analysis.
\begin{align}
    &\widehat{P}^k_{\min,h}(s,a) := \min_{s'\in \mathcal{S}}\{\widehat{P}^k_{h}(s'|s,a): \widehat{P}^k_{h}(s'|s,a)>0\}, \label{eq:hat_p_min_KL}\\
    &P^\star_{\min,h}(s,a) := \min_{s'\in \mathcal{S}} \{P^\star_h(s'|s,a):P^\star_h(s'|s,a)>0\},\label{eq:p_star_min_h_KL}\\
    &P^\star_{\min} := \min_{(h,s)\in [H]\times \mathcal{S}} P^\star_{\min,h}(s,\pi^\star_h(s)),\label{eq:p_star_min_KL}
\end{align}
where $P^\star_h(s'|s,a)\geq P^\star_{\min,h}(s,\pi^\star_h(s))\geq P^\star_{\min}$ which satisfies \Cref{ass:KL_P_min}.\\

\subsubsection{Define the event \( \mathcal{E}_{KL} \) for {\RMDPKL}:}
Before presenting all lemmas, we define the typical event \( \mathcal{E}_{KL} \)  as
\begin{align}
\label{eq:Event_KL}
    \mathcal{E}_{KL} =\bigg\{&\abs{\log\bigg(\mathbb{E}_{\widehat{P}_h^k(\cdot|s,a)}\bigg[\exp\bigg\{-\frac{V_{h+1}}{\eta}\bigg\}\bigg]\bigg)  - \log\bigg(\mathbb{E}_{P^{\star}_h(\cdot|s,a)}\bigg[\exp\bigg\{-\frac{V_{h+1}}{\eta}\bigg\}\bigg]\bigg)} \leq c_1\sqrt{\frac{L}{\{N^k_h(s,a)\vee 1\}\widehat{P}^k_{\min,h}(s,a)}},\nonumber\\
   % &\abs{P^{\star}_h(s^{\prime}|s,a)-\widehat{P}_h(s^{\prime}|s,a)}\leq \sqrt{\frac{c_1L\min\{P^{\star}_h(s^{\prime}|s,a),\widehat{P}_h(s^{\prime}|s,a)\}}{\{N^k_h(s,a)\vee 1\}}} + \frac{c_2L}{\{N^k_h(s,a)\vee 1\}},\nonumber\\
   &\qquad  \forall (h,s,a,s^{\prime},k) \in [H]\times \mathcal{S}\times \mathcal{A}\times \mathcal{S}\times[K], \forall \eta \in \mathcal{N}_{\frac{1}{\sigma S\sqrt{K}}}\Big([0,H/\sigma]\Big) \bigg\},
\end{align}
where $\widehat{P}^k_{\min,h}(s,a)$ is defined in \eqref{eq:hat_p_min_KL}, $L=\log(S^3AH^2K^{3/2}/\delta)$, $c_1> 0$ is an absolute constant and $\eta \in \mathcal{N}_{\frac{1}{\sigma S\sqrt{K}}}\Big([0,H/\sigma]\Big)$, where $\mathcal{N}_{\frac{1}{\sigma S\sqrt{K}}}\Big([0,H/\sigma]\Big)$ denotes an $\frac{1}{\sigma S\sqrt{K}}$-cover of the interval $[0,H/\sigma]$.

\begin{lem}[Bound of event $\mathcal{E}_{KL}$]
    For the typical event $\mathcal{E}_{KL}$ defined in \eqref{eq:Event_KL}, it holds that $\Pr(\mathcal{E}_{KL}) \geq 1 - \delta$.
\end{lem}
\begin{proof}
    This is a direct application of \cite[Lemma 16]{NeuRIPS2024_UnifiedPessimismOfflineRL_Yue}, together with a union bound over $(h, s, a, s', k, \eta) \in [H]\times \mathcal{S}\times \mathcal{A} \times \mathcal{S}\times [K]\times\mathcal{N}_{\frac{1}{\sigma S\sqrt{K}}}\Big([0,H/\sigma]\Big)$. Note that the size of $\mathcal{N}_{\frac{1}{\sigma S\sqrt{K}}}\Big([0,H/\sigma]\Big)$ is of order $SH\sqrt{K}$.
\end{proof}

\subsubsection{Proof of \Cref{thm:Regret_f_bound} ({\RMDPKL} Setting)}
\label{app:proof_thm_Regret_KL_bound}
\begin{proof}
With \Cref{lem:Optimistic_pessimism_KL}, we can upper bound the regret as
\begin{align}
    \text{Regret}(K) = \sum_{k=1}^{K} V_{1}^{\star,\sigma}(s^k_1) - V_{1}^{\pi^k,\sigma}(s^k_1) 
    \leq \sum_{k=1}^{K} \overline{V}_1^k(s^k_1) - \underline{V}_1^k(s^k_1). \label{eq:Regret_KL_step1}
\end{align}

In the following, we break our proof into three steps.
\begin{itemize}
\item \textbf{Step 1: Upper bound \eqref{eq:Regret_KL_step1}.} By the choice of $\overline{Q}_h^k$, $\underline{Q}_h^k$, $\overline{V}_h^k$, $\underline{V}_h^k$ as given in \eqref{eq:Upper_estimate_Q}, \eqref{eq:Lower_estimate_Q} and \eqref{eq:policy_value_epsiode_k}, and by the choice of bonus term $B^{KL}_{k,h}(s,a)$ given in \eqref{eq:Bonus_term_KL} for any $(h,k) \in [H] \times [K]$ and $(s,a) \in \mathcal{S} \times \mathcal{A}$,
\begin{align}
    \overline{Q}_h^k(s,a) - \underline{Q}_h^k(s,a) &= \min \left\{ r_h(s,a) + \mathbb{E}_{\widehat{\mathcal{U}^{\sigma}_h}(s,a)}\left[ \overline{V}_{h+1}^k \right] + B^{KL}_{k,h}(s,a), H\right\} \\
    &\quad - \max \left\{ r_h(s,a) + \mathbb{E}_{\widehat{\mathcal{U}^{\sigma}_h}(s,a)}\left[ \underline{V}_{h+1}^k \right] - B^{KL}_{k,h}(s,a), 0 \right\} \\
    &\leq \mathbb{E}_{\widehat{\mathcal{U}^{\sigma}_h}(s,a)}\left[ \overline{V}_{h+1}^k \right] 
    - \mathbb{E}_{\widehat{\mathcal{U}^{\sigma}_h}(s,a)}\left[ \underline{V}_{h+1}^k \right] + 2B^{KL}_{k,h}(s,a).\label{eq:Regret_KL_step2}
\end{align}
We denote
\begin{align}
    A:= &\mathbb{E}_{\widehat{\mathcal{U}^{\sigma}_h}(s,a)}\left[ \overline{V}_{h+1}^k \right]  - \mathbb{E}_{\mathcal{U}^{\sigma}_h(s,a)}\left[ \overline{V}_{h+1}^k \right] + \mathbb{E}_{\mathcal{U}^{\sigma}_h(s,a)}\left[ \underline{V}_{h+1}^k \right] - \mathbb{E}_{\widehat{\mathcal{U}^{\sigma}_h}(s,a)}\left[ \underline{V}_{h+1}^k \right] \label{eq:Regret_KL_A}\\
    B := &\mathbb{E}_{\mathcal{U}^{\sigma}_h(s,a)}\left[ \overline{V}_{h+1}^k \right] - \mathbb{E}_{\mathcal{U}^{\sigma}_h(s,a)}\left[ \underline{V}_{h+1}^k \right]\label{eq:Regret_KL_B}
\end{align}
Applying \eqref{eq:Regret_KL_A} and \eqref{eq:Regret_KL_B} in \eqref{eq:Regret_KL_step2}, we get
\begin{align}
   \overline{Q}_h^k(s,a) - \underline{Q}_h^k(s,a)   &\leq A + B + 2B^{KL}_{k,h}(s,a). \label{eq:Regret_KL_step3}
\end{align}

\begin{enumerate}[label=(\roman*)]
    \item \textbf{Upper bound $A$.} By using a  concentration bound argument customized for KL robust expectations in \Cref{lem:Proper_bound_optimism_pessimism_bound_KL}, we can bound term $A$ by the bonus, as given by
    \begin{align}
        A \leq 2B^{KL}_{k,h}(s,a). \label{eq:Regret_bound_KL_A}
    \end{align}

    \item \textbf{Upper bound $B$.}  By the definition of $\mathbb{E}_{\mathcal{U}_h^\sigma(s,a)}[V]$ in \eqref{eq:dual_KL}, we have
    \begin{align}
        B &= \sup_{\eta \in \left[0,\frac{H}{\sigma}\right]} \bigg\{-\eta\log\bigg(\mathbb{E}_{P^{\star}_h(\cdot|s,a)}\bigg[\exp\bigg\{-\frac{\overline{V}^k_{h+1}}{\eta}\bigg\}\bigg]\bigg) - \eta\sigma\bigg\}  \nonumber\\
        &\qquad \qquad \quad - \sup_{\eta \in \left[0,\frac{H}{\sigma}\right]} \bigg\{-\eta\log\bigg(\mathbb{E}_{P^{\star}_h(\cdot|s,a)}\bigg[\exp\bigg\{-\frac{\underline{V}^k_{h+1}}{\eta}\bigg\}\bigg]\bigg) - \eta\sigma\bigg\}\nonumber\\
        &\leq \sup_{\eta \in [0,H/\sigma]} \eta \Bigg\{ \log\bigg(\mathbb{E}_{P^{\star}_h(\cdot|s,a)}\bigg[\exp\bigg\{-\frac{\underline{V}^k_{h+1}}{\eta}\bigg\}\bigg]\bigg) - \log\bigg(\mathbb{E}_{P^{\star}_h(\cdot|s,a)}\bigg[\exp\bigg\{-\frac{\overline{V}^k_{h+1}}{\eta}\bigg\}\bigg]\bigg)\Bigg\} \nonumber\\
        &= \sup_{\eta \in [0,H/\sigma]} \eta \log\Bigg( \frac{\mathbb{E}_{P^{\star}_h(\cdot|s,a)}\bigg[\exp\bigg\{-\frac{\underline{V}^k_{h+1}}{\eta}\bigg\}\bigg]}{\mathbb{E}_{P^{\star}_h(\cdot|s,a)}\bigg[\exp\bigg\{-\frac{\overline{V}^k_{h+1}}{\eta}\bigg\}\bigg]}\Bigg)\nonumber\\
        &=\sup_{\eta \in [0,H/\sigma]} \eta \log\Bigg(1 +  \frac{\mathbb{E}_{P^{\star}_h(\cdot|s,a)}\bigg[\exp\bigg\{-\frac{\underline{V}^k_{h+1}}{\eta}\bigg\}- \exp\bigg\{-\frac{\overline{V}^k_{h+1}}{\eta}\bigg\}\bigg]}{\mathbb{E}_{P^{\star}_h(\cdot|s,a)}\bigg[\exp\bigg\{-\frac{\overline{V}^k_{h+1}}{\eta}\bigg\}\bigg]}\Bigg)\nonumber
            \end{align}
Therefore, 
    \begin{align}
       B &\overset{\text{(a)}}{\leq} \sup_{\eta \in [0,H/\sigma]} \eta \frac{\mathbb{E}_{P^{\star}_h(\cdot|s,a)}\bigg[\exp\bigg\{-\frac{\underline{V}^k_{h+1}}{\eta}\bigg\}- \exp\bigg\{-\frac{\overline{V}^k_{h+1}}{\eta}\bigg\}\bigg]}{\mathbb{E}_{P^{\star}_h(\cdot|s,a)}\bigg[\exp\bigg\{-\frac{\overline{V}^k_{h+1}}{\eta}\bigg\}\bigg]}\nonumber\\
       &\overset{\text{(b)}}{\leq} \sup_{\eta \in [\underline{\eta},H/\sigma]} \eta \exp\bigg\{\frac{H}{\underline{\eta}}\bigg\}\mathbb{E}_{P^{\star}_h(\cdot|s,a)}\bigg[\exp\bigg\{-\frac{\underline{V}^k_{h+1}}{\eta}\bigg\}- \exp\bigg\{-\frac{\overline{V}^k_{h+1}}{\eta}\bigg\}\bigg] \nonumber\\
        &\overset{\text{(c)}}{\leq}  \exp\bigg\{\frac{H}{\underline{\eta}}\bigg\}\bigg[\mathbb{P}^{\star}_h\big(\overline{V}^k_{h+1} - \underline{V}^k_{h+1}\big)\bigg](s,a), \label{eq:Regret_bound_KL_B}
    \end{align}
where in the inequality (a) we use the fact of $\log(1 + x) \leq x$, and in the inequality (b) is due to the fact that $0 \leq \overline{V}^k_{h+1}\leq H$ and $\eta \in [\underline{\eta}, H/\sigma]$ by the regularity bound of KL-divergence \cite{NeuRIPS2023_DoublePessimismDROfflineRL_Blanchet, ICML2025_OnlineDRMDPSampleComplexity_He}. Lastly, the inequality (c) is due to the $\frac{1}{\eta}$-Lipschitz continuity of $\phi_{\eta}(x) = \exp\big\{-\frac{x}{\eta}\big\}$ for $x\geq 0$, and $\underline{V}^k_{h+1} \leq \overline{V}^k_{h+1}$ by \Cref{lem:Optimistic_pessimism_KL}.
\end{enumerate}

Therefore, by applying \eqref{eq:Regret_bound_KL_A} and \eqref{eq:Regret_bound_KL_B} in \eqref{eq:Regret_KL_step3}, we get
\begin{align}
      \overline{Q}_h^k(s,a) - \underline{Q}_h^k(s,a)   &\leq  \exp\bigg\{\frac{H}{\underline{\eta}}\bigg\}\bigg[\mathbb{P}^{\star}_h\big(\overline{V}^k_{h+1} - \underline{V}^k_{h+1}\big)\bigg](s,a) + 4B^{KL}_{k,h}(s,a) .\label{eq:Regret_KL_step4}
\end{align}
We recall the bound of bonus term as given in \Cref{lem:Control_Bonus_KL}, 
\begin{align}
\label{eq:Regret_KL_step5}
    B^{KL}_{k,h}(s,a) \leq \frac{2c_fH}{\sigma}\sqrt{\frac{L^2}{\{N^k_h(s,a)\vee 1\}P^{\star}_{\min}}} + \sqrt{\frac{1}{K}},
\end{align}
where $L = \log\bigg(\frac{S^3AH^2K^{3/2}}{\delta}\bigg)$, \( c_f>0\) is an absolute constant, and $P^{\star}_{\min}$ is defined in \eqref{eq:p_star_min_KL} and satisfies \Cref{ass:KL_P_min}.

By applying \eqref{eq:Regret_KL_step5} in \eqref{eq:Regret_KL_step4}, and after rearranging terms we further obtain that
\begin{align}
   \overline{Q}_h^k(s,a) - \underline{Q}_h^k(s,a) &\leq \exp\bigg\{\frac{H}{\underline{\eta}}\bigg\}\bigg[\mathbb{P}^{\star}_h\big(\overline{V}^k_{h+1} - \underline{V}^k_{h+1}\big)\bigg](s,a) + \frac{c_1H}{\sigma}\sqrt{\frac{L^2}{\{N^k_h(s,a)\vee 1\}P^{\star}_{\min}}} + 4\sqrt{\frac{1}{K}},\label{eq:Regret_KL_step6}
\end{align}
where $c_1 > 0$ is an absolute constant. 

For the sake of brevity, we now introduce the  following notations of differences, for any $(h,k) \in [H] \times [K]$, as given by
\begin{align}
    \Delta^k_h &:= \overline{V}^k_{h}(s^k_h) -\underline{V}^k_{h}(s^k_h),\label{eq:Delta_k_h_KL}\\
    \zeta_h^k &:= \Delta_h^k - \left( \overline{Q}_h^k(s_h^k, a_h^k) - Q_h^k(s_h^k, a_h^k) \right), \label{eq:zeta_h_k_KL}\\
    \xi_h^k &:= \bigg[\mathbb{P}^{\star}_h\big(\overline{V}^k_{h+1} - \underline{V}^k_{h+1}\big)\bigg](s,a)  - \Delta_{h+1}^k. \label{eq:xi_h_k_KL}
\end{align}
We now define the filtration $\{ \mathcal{F}_{h,k} \}_{(h,k) \in [H] \times [K]}$ as
\begin{align*}
    \mathcal{F}_{h,k} := \sigma \bigg( \Big\{\big(s_i^\tau, a_i^\tau \big)\Big\}_{(i,\tau) \in [H] \times [k-1]} \bigcup \Big\{\big(s_i^k, a_i^k\big)\Big\}_{i \in [h-1]} \bigcup \Big\{s_h^k\Big\} \bigg).
\end{align*}
Considering the filtration $\{ \mathcal{F}_{h,k} \}_{(h,k) \in [H] \times [K]}$, we can find that $\{\zeta_h^k\}_{(h,k) \in [H] \times [K]}$ is a martingale difference sequence with respect to $\{ \mathcal{F}_{h,k} \}_{(h,k) \in [H] \times [K]}$ and $\{\xi_h^k\}_{(h,k) \in [H] \times [K]}$ is a martingale difference sequence with respect to $\{\mathcal{F}_{h,k} \cup \{a_h^k\}\}_{(h,k) \in [H] \times [K]}$. Furthermore, applying \eqref{eq:Regret_KL_step6} in \eqref{eq:zeta_h_k_KL}, we have
\begin{align}
\Delta_h^k &= \zeta_h^k + \left( \overline{Q}_h^k(s_h^k, a_h^k) - \underline{Q}_h^k(s_h^k, a_h^k) \right) \nonumber \\
&\le \zeta_h^k + \exp\bigg\{\frac{H}{\underline{\eta}}\bigg\}\bigg[\mathbb{P}^{\star}_h\big(\overline{V}^k_{h+1} - \underline{V}^k_{h+1}\big)\bigg](s,a) + \frac{c_1H}{\sigma}\sqrt{\frac{L^2}{\{N^k_h(s,a)\vee 1\}P^{\star}_{\min}}} + 4\sqrt{\frac{1}{K}}\nonumber \\
&= \zeta_h^k + \exp\bigg\{\frac{H}{\underline{\eta}}\bigg\} \xi_h^k + \exp\bigg\{\frac{H}{\underline{\eta}}\bigg\}\Delta_{h+1}^k + \frac{c_1H}{\sigma}\sqrt{\frac{L^2}{\{N^k_h(s,a)\vee 1\}P^{\star}_{\min}}} + 4\sqrt{\frac{1}{K}}. \label{eq:Regret_KL_step7}
\end{align}
Recursively applying \eqref{eq:Regret_KL_step7} and using the fact that $1 \leq \left(\exp\Big\{\frac{H}{\underline{\eta}}\Big\}\right)^h \leq \left(\exp\Big\{\frac{H}{\underline{\eta}}\Big\}\right)^H := d_H$ for some constant $d_{H} > 0$, we can upper bound the right hand side of \eqref{eq:Regret_KL_step1} as
\begin{align}
\text{Regret}(K) \leq \sum_{k=1}^K \Delta_1^k &\leq C'd_H \cdot \sum_{k=1}^K \sum_{h=1}^H \Bigg\{\left( \zeta_h^k + \xi_h^k \right) + \frac{H}{\sigma}\sqrt{\frac{L^2}{\{N^k_h(s,a)\vee 1\}P^{\star}_{\min}}} + \sqrt{\frac{1}{K}}\Bigg\}\nonumber\\
&=C'd_H \cdot \bigg\{\text{Term (i) } + \text{Term (ii) } + \text{Term (iiii) } \bigg\}, \label{eq:Regret_KL_step8}
\end{align}
where $C' > 0$ is an absolute constant, and 
\begin{align}
    \text{Term (i) } &:=\sum_{k=1}^K \sum_{h=1}^H \Bigg\{\left( \zeta_h^k + \xi_h^k \right) \Bigg\}.\label{eq:Term_i_KL}\\
    \text{Term (ii) } &:= \sum_{k=1}^K \sum_{h=1}^H \Bigg\{\frac{H}{\sigma}\sqrt{\frac{L^2}{\{N^k_h(s,a)\vee 1\}P^{\star}_{\min}}}\Bigg\}.\label{eq:Term_ii_KL}\\
    \text{Term (iii) } &:=  \sum_{k=1}^K \sum_{h=1}^H \Bigg\{\sqrt{\frac{1}{K}}\Bigg\}= H\sqrt{K}.\label{eq:Term_iii_KL}
\end{align}

\item \textbf{Step 2: Upper bound on Term (i).} Note that according to the definition in \eqref{eq:zeta_h_k_KL} and \eqref{eq:xi_h_k_KL}, both $\zeta^k_h$ and $\xi^k_h$ are bounded in the range $[0,H]$. As a result, using Azuma-Hoeffding inequality in \Cref{lem:Azuma-Hoeffding}, with probability at least \(1-\delta\),
\begin{align}
\label{eq:Regret_KL_step9}
    \sum_{k=1}^{K} \sum_{h=1}^{H} (\zeta_h^k + \xi_h^k)
\leq c_2 \sqrt{H^3KL},
\end{align}
where \( c_2 > 0 \) is an absolute constant.

\item \textbf{Step 3: Upper bound on Term (ii).} To proceed, it is sufficient to upper bound the right-hand side of \eqref{eq:Term_ii_KL}. By applying the Cauchy-Schwarz inequality in Term (ii) in \eqref{eq:Term_ii_KL}, we get
\begin{align}
\label{eq:Regret_KL_step10}
\sum_{k=1}^K \sum_{h=1}^H \sqrt{ \frac{ 1 }{\{N_h^k(s_h^k, a_h^k) \vee 1\}P^{\star}_{\min}}} \le \sqrt{\left( \sum_{k=1}^K \sum_{h=1}^H \frac{1}{P^{\star}_{\min}}\right)\cdot 
\left( \sum_{k=1}^K \sum_{h=1}^H \frac{1}{\{N_h^k(s_h^k, a_h^k) \vee 1\}} \right)}.
\end{align}

According to \Cref{lem:inverse_count_bound}, we have
\begin{align}
\label{eq:Regret_KL_step11}
     \sum_{k=1}^K \sum_{h=1}^H \frac{1}{\{N_h^k(s_h^k, a_h^k) \vee 1\}} \leq c_3HSA\log(K) \leq c_3HSAL,
\end{align}
where $c_3>0$ is an absolute constant, and $L=\log(S^2AH^2K^{3/2}/\delta)$.

Moreover, according to the definition of $P^{\star}_{\min}$ as given in \Cref{ass:KL_P_min}, we have
\begin{align}
\label{eq:Regret_KL_step12}
   \sum_{k=1}^K \sum_{h=1}^H \frac{1}{P^{\star}_{\min}} \leq \frac{KH}{P^{\star}_{min}}.
\end{align}

Combining \eqref{eq:Regret_KL_step11} and \eqref{eq:Regret_KL_step12} in \eqref{eq:Regret_KL_step10}, we get
\begin{align}
\label{eq:Regret_KL_step13}
   \sum_{k=1}^K \sum_{h=1}^H \sqrt{ \frac{ 1 }{\{N_h^k(s_h^k, a_h^k) \vee 1\}P^{\star}_{\min}}} \leq c_4\sqrt{\frac{H^2KSAL}{P^{\star}_{\min}}},
\end{align}
where \( c_4 > 0 \) being another absolute constant.

\item \textbf{Step 4: Conclusion the proof.} Therefore, applying \eqref{eq:Regret_KL_step9}, \eqref{eq:Regret_KL_step13}, and \eqref{eq:Term_iii_KL} in \eqref{eq:Regret_KL_step8}, with probability at least \(1 - 3\delta\), we have
\begin{align}
\label{eq:Regret_KL_step19}
    \text{Regret}(K) &\leq C''d_H \cdot \bigg(\sqrt{H^3KL} + \sqrt{\frac{H^4KSAL^3}{P^{\star}_{\min}\sigma^2}} + \sqrt{H^2K}\bigg) = \mathcal{O} \left(\sqrt{\frac{H^4SAK\exp\big(2H^2\big)\iota}{P^{\star}_{\min}\sigma^2}} \right),
\end{align}
where $C''$ is any absolute constant and $\iota = \bigg(\log\Big(\frac{(SAHK}{\delta}\Big)\bigg)^3$. 
\end{itemize}
This completes the proof of \Cref{thm:Regret_f_bound}. \qedhere
\end{proof}

\subsection{Key Lemmas for {\RMDPKL}}
\label{subsubsec:Key_Lemma_KL}
%%%%%%%%%%%%%%%% Key Lemma for KL-Divergence %%%%%%%%%%%%%%%%%%%

\begin{keylem}[Optimistic and pessimistic estimation of the robust values for {\RMDPKL}]
\label{lem:Optimistic_pessimism_KL}
\textit{By setting the bonus \( B^{KL}_{k,h} \) as in \eqref{eq:Bonus_term_chi}, then under the typical event \( \mathcal{E} \), it holds that}
\begin{align}
\label{eq:Optimistic_pessimism_ineq_KL}
\underline{Q}_h^k(s,a) \leq Q_{h}^{\pi^k,\sigma}(s,a) &\leq Q_{h}^{\star,\sigma}(s,a) \leq \overline{Q}_h^k(s,a),\nonumber\\
\underline{V}_h^k(s) \leq V_{h}^{\pi^k,\sigma}(s) &\leq V_{h}^{\star,\sigma}(s) \leq \overline{V}_h^k(s),
\end{align}
\textit{for any \( (h, s, a, k) \in [H]\times \mathcal{S} \times \mathcal{A} \times [K] \).}
\end{keylem}

\begin{proof}
We will prove \Cref{lem:Optimistic_pessimism_KL} by induction and in three cases, as follows:
\begin{itemize}
    \item \textbf{Ineq. 1:} To prove $Q_{h}^{\star,\sigma}(s,a) \leq \overline{Q}_h^k(s,a)$.
    \item \textbf{Ineq. 2:} To prove $Q_{h}^{\pi^k,\sigma}(s,a) \leq Q_{h}^{\star,\sigma}(s,a)$.
    \item \textbf{Ineq. 3:} To prove $\underline{Q}_h^k(s,a) \leq Q_{h}^{\pi^k,\sigma}(s,a)$.
\end{itemize}
Let us consider that \eqref{eq:Optimistic_pessimism_ineq_KL} holds at step $h+1$. 
\begin{itemize}
    \item \textbf{Proof of Ineq. 1:}
For step $h$, we will first consider the robust $Q$ function part. Specifically, by using the robust Bellman optimal equations (Eq. \eqref{eq:Robust_bellman_Q_fn} and \eqref{eq:Robust_bellman_V_fn}) and \eqref{eq:Upper_estimate_Q}, we have that
\begin{align}
Q_{h}^{\star,\sigma}(s,a) - \overline{Q}_h^k(s,a)
&= \max \left\{ \mathbb{E}_{\mathcal{U}^{\sigma}_h(s,a)} \left[ V_{h+1}^{\star,\sigma} \right]
- \mathbb{E}_{\widehat{\mathcal{U}^{\sigma}_h}(s,a)} \left[\overline{V}^k_{h+1} \right] - B^{KL}_{k,h}(s,a), \,
Q_{h}^{\star,\sigma}(s,a) - H \right\} \nonumber\\
&\leq \max \left\{\mathbb{E}_{\mathcal{U}^{\sigma}_h(s,a)} \left[ V_{h+1}^{\star,\sigma} \right]
- \mathbb{E}_{\widehat{\mathcal{U}^{\sigma}_h}(s,a)} \left[V_{h+1}^{\star,\sigma} \right] - B^{KL}_{k,h}(s,a), 0 \right\},\label{eq:optimism_pessimism_ineq_KL_step1}
\end{align}
where the second inequality follows from the induction of $V_{h+1}^{\star,\sigma} \leq \overline{V}_{h+1}^k$ at step $h+1$ and the fact that $Q_{h}^{\star,\sigma} \leq H$. By Lemma \ref{lem:Bound_KL_optimal_policy} and by the definition of $\widehat{P}^k_{\min,h}(s,a)$ as given in \eqref{eq:hat_p_min_KL}, we have that
\begin{align}
\label{eq:optimism_pessimism_ineq_KL_step2}
\mathbb{E}_{\mathcal{U}^{\sigma}_h(s,a)} \left[ V_{h+1}^{\star,\sigma} \right]
- \mathbb{E}_{\widehat{\mathcal{U}^{\sigma}_h}(s,a)} \left[V_{h+1}^{\star,\sigma} \right] \leq \frac{c_1H}{\sigma}\sqrt{\frac{L}{\{N^k_h(s,a)\vee 1\}\widehat{P}^k_{\min,h}(s,a)}} + \sqrt{\frac{1}{K}}.
\end{align}

Now recollect the choice of $ B^{KL}_{k,h}$ as given in \eqref{eq:Bonus_term_KL}. Therefore, combining \eqref{eq:optimism_pessimism_ineq_KL_step2} and \eqref{eq:Bonus_term_KL} in \eqref{eq:optimism_pessimism_ineq_KL_step1}, we can conclude that
\begin{align}
\label{eq:optimism_pessimism_ineq_KL_bound_case1}
Q_{h}^{\star,\sigma}(s,a) \leq \overline{Q}_h^k(s,a).
\end{align}

\item \textbf{Proof of Ineq. 2:} By the definition of $Q_{h}^{\star,\sigma}(s,a)$, the
\begin{align}
\label{eq:optimism_pessimism_ineq_KL_bound_case2}
    Q_{h}^{\pi^k,\sigma}(s,a) \leq Q_{h}^{\star,\sigma}(s,a).
\end{align}
is trivial.

\item \textbf{Proof of Ineq. 3:} By using the robust Bellman equations (Eq. \eqref{eq:Robust_bellman_Q_fn} and \eqref{eq:Robust_bellman_V_fn}) and \eqref{eq:Lower_estimate_Q}, we have that
\begin{align}
\underline{Q}_h^k(s,a) - Q_{h}^{\pi^k,\sigma}(s,a) &=\max \left\{
\mathbb{E}_{\widehat{\mathcal{U}_h^{\sigma}}(s,a)} \left[ \underline{V}_{h+1}^k \right]
- \mathbb{E}_{\mathcal{U}_h^\sigma(s,a)} \left[ V_{h+1}^{\pi^k,\sigma} \right]- B^{KL}_{k,h}(s,a), \,0 - Q_{h}^{\pi^k,\sigma}(s,a)
\right\} \notag\\
&\leq \max \left\{
\mathbb{E}_{\widehat{\mathcal{U}_h^\sigma}(s,a)} \left[ V_{h+1}^{\pi^k,\sigma} \right]
- \mathbb{E}_{\mathcal{U}_h^\sigma(s,a)} \left[ V_{h+1}^{\pi^k,\sigma} \right]- B^{KL}_{k,h}(s,a), \, 0\right\}, \label{eq:optimism_pessimism_ineq_KL_step4}
\end{align}
where the second inequality follows from the induction of $\underline{V}_{h+1}^k \leq V_{h+1}^{\pi^k,\sigma}$ at step $h+1$ and the fact that $Q_{h}^{\pi^k,\sigma} \geq 0$. By Lemma ~\ref{lem:Bound_KL_policy_k}, we get
\begin{align}
\mathbb{E}_{\widehat{\mathcal{U}_h^\sigma}(s,a)} \left[ V_{h+1}^{\pi^k,\sigma} \right]
- \mathbb{E}_{\mathcal{U}_h^\sigma(s,a)} \left[ V_{h+1}^{\pi^k,\sigma} \right]\leq \frac{c_1H}{\sigma}\sqrt{\frac{L}{\{N^k_h(s,a)\vee 1\}\widehat{P}^k_{\min,h}(s,a)}} + \sqrt{\frac{1}{K}}.\label{eq:optimism_pessimism_ineq_KL_step5}
\end{align}

Thus by combining \eqref{eq:optimism_pessimism_ineq_KL_step5}, \eqref{eq:optimism_pessimism_ineq_KL_step4}, the choice of $B^{KL}_{k,h}(s,a)$ in \eqref{eq:Bonus_term_KL}, we get
\begin{align}
\label{eq:optimism_pessimism_ineq_KL_bound_case3}
    \underline{Q}_h^k(s,a) \leq Q_{h}^{\pi^k,\sigma}(s,a).
\end{align}
\end{itemize}
Therefore, by \eqref{eq:optimism_pessimism_ineq_KL_bound_case1}, \eqref{eq:optimism_pessimism_ineq_KL_bound_case2} and \eqref{eq:optimism_pessimism_ineq_KL_bound_case3}, we have proved that at step $h$, it holds that
\begin{align}
\label{eq:optimism_pessimism_ineq_Qbound_KL_final}
    \underline{Q}_h^k(s,a) \leq Q_{h}^{\pi^k,\sigma}(s,a) \leq Q_{h}^{\star,\sigma}(s,a) \leq \overline{Q}_h^k(s,a).
\end{align}
Finally for the robust $V$ function part, consider that by the robust Bellman equation (Eq. \eqref{eq:Robust_bellman_Q_fn} and \eqref{eq:Robust_bellman_V_fn}) and \eqref{eq:policy_value_epsiode_k}, 
\begin{align}
\label{eq:optimism_pessimism_ineq_KL_step7}
\underline{V}_h^k(s) = \max_{a \in \mathcal{A}}  \underline{Q}_h^k(s,\cdot) \leq \max_{a \in \mathcal{A}} Q_{h}^{\pi^k,\sigma}(s,\cdot) = V_{h}^{\pi^k,\sigma}(s),
\end{align}
and that by the robust Bellman optimality (Corollary ~\ref{cor:Robust_Bellman_Optimal_eq}), the choice of $\pi^k$, $\overline{V}^k_h$ and $\underline{V}^k_h$ defined in \eqref{eq:policy_value_epsiode_k},
\begin{align}
\label{eq:optimism_pessimism_ineq_KL_step8}
    V_{h}^{\star,\sigma}(s) = \max_{a \in \mathcal{A}} Q_{h}^{\star,\sigma}(s,a) \leq \max_{a \in \mathcal{A}} \overline{Q}_h^k(s,a) = \overline{V}_h^k(s),
\end{align}
which proves that
\begin{align}
\label{eq:optimism_pessimism_ineq_Vbound_KL_final}
    \underline{V}_h^k(s) \leq V_{h}^{\pi^k,\sigma}(s) \leq V_{h}^{\star,\sigma}(s) \leq \overline{V}_h^k(s).
\end{align}
Since the conclusion \eqref{eq:Optimistic_pessimism_ineq_KL} holds for the $V$ function part at step $H+1$, an induction proves Lemma ~\ref{lem:Optimistic_pessimism_KL}. \qedhere
\end{proof}
%%%%%%%%%%%%%%%%%%%%%%%

\begin{keylem}[Proper bonus for {\RMDPKL} and optimistic and pessimistic value estimators] 
\label{lem:Proper_bound_optimism_pessimism_bound_KL}
By setting the bonus $ B^{KL}_{k,h}$ as in \eqref{eq:Bonus_term_KL}, then under the typical event $\mathcal{E}_{KL}$, it holds that
\begin{align}
\label{eq:Proper_bound_optimism_pessimism_bound_KL}
\mathbb{E}_{\widehat{\mathcal{U}_h^\sigma}(s,a)}\left[\overline{V}_{h+1}^k\right] -\mathbb{E}_{\mathcal{U}_h^\sigma(s,a)}\left[\overline{V}_{h+1}^k\right] + \mathbb{E}_{\mathcal{U}_h^\sigma(s,a)}\left[\underline{V}_{h+1}^k\right] -\mathbb{E}_{\widehat{\mathcal{U}_h^\sigma}(s,a)}\left[\underline{V}_{h+1}^k\right] \leq 2B^{KL}_{k,h}(s,a).
\end{align}
\end{keylem}

\begin{proof}
    Let us denote
    \begin{align}
    \label{eq:A_KL}
        A:= \mathbb{E}_{\widehat{\mathcal{U}_h^\sigma}(s,a)}\left[\overline{V}_{h+1}^k\right] -\mathbb{E}_{\mathcal{U}_h^\sigma(s,a)}\left[\overline{V}_{h+1}^k\right] + \mathbb{E}_{\mathcal{U}_h^\sigma(s,a)}\left[\underline{V}_{h+1}^k\right] -\mathbb{E}_{\widehat{\mathcal{U}_h^\sigma}(s,a)}\left[\underline{V}_{h+1}^k\right] 
    \end{align}

We upper bound $A$ by using the concentration inequality given in \Cref{lem:Bound_KL_optimism_pessimism},
\begin{align}
A \leq \frac{2c_1H}{\sigma}\sqrt{\frac{L}{\{N^k_h(s,a)\vee 1\}\widehat{P}^k_{\min,h}(s,a)}} + 2\sqrt{\frac{1}{K}}, \label{eq:Proper_bouns_optimism_pessimism_KL_bound_step1}
\end{align}
where \( c_1 > 0 \) are absolute constants. Therefore, by the choice of \(B^{KL}_{k,h}(s,a)\) in \eqref{eq:Bonus_term_KL}, we get \eqref{eq:Proper_bound_optimism_pessimism_bound_KL}. This concludes the proof of \Cref{lem:Proper_bound_optimism_pessimism_bound_KL}.
\end{proof}

\begin{keylem}[Control of the bonus term for {\RMDPKL}]
\label{lem:Control_Bonus_KL}
\textit{Under the typical event \( \mathcal{E}_{KL} \) and \Cref{ass:KL_P_min}, the bonus term \(  B^{KL}_{k,h} \) in \eqref{eq:Bonus_term_KL} is bounded by}
\begin{align}
    B^{KL}_{k,h}(s,a) \leq \frac{2c_fH}{\sigma}\sqrt{\frac{L^2}{\{N^k_h(s,a)\vee 1\}P^{\star}_{\min}}} + \sqrt{\frac{1}{K}},
\end{align}
where $P^{\star}_{\min}$ satisfies \Cref{ass:KL_P_min}, $L = \log\bigg(\frac{S^3AH^2K^{3/2}}{\delta}\bigg)$, and \( c_f>0\).
\end{keylem}

\begin{proof}
   We recall the choice of $B^{KL}_{k,h}$ as given in \eqref{eq:Bonus_term_KL}, i.e.
\begin{align}
\label{eq:Bonus_term_KL_step1}
   B^{KL}_{k,h}(s,a) = \frac{2c_fH}{\sigma}\sqrt{\frac{L}{\{N^k_h(s,a)\vee 1\}\widehat{P}^k_{\min,h}(s,a)}} + \sqrt{\frac{1}{K}},
\end{align}
where \( L = \log(S^3 A H^2 K^{3/2} / \delta) \), $\widehat{P}^k_{\min,h}(s,a)$ is defined in \eqref{eq:hat_p_min_KL}, and $c_f>0$ is an absolute constant. 
   
By \Cref{lem:binomial_rv_bound} and the union bound, it holds that with probability at least $1 - \delta$ that for all $(h,s,a) \in [H]\times\mathcal{S}\times\mathcal{A}$, we get
\begin{align}
\label{eq:binomial_bound_step1}
    \forall s' \in \mathcal{S}:\quad 
P^{\star}_h(s' \mid s,a) \geq \frac{\widehat{P}^k_h(s' \mid s,a)}{e^2} \geq \frac{P^{\star}_h(s' \mid s,a)}{8e^2L}.
\end{align}

To characterize the relation between $P^{\star}_{\min,h}(s,a)$ and $\widehat{P}^k_{\min,h}(s,a)$ for any $(h,s,a) \in [H]\times\mathcal{S}\times\mathcal{A}$, we suppose—without loss of generality—that $P^{\star}_{\min,h}(s,a) = P^{\star}_h(s_1 \mid s,a)$ and $\widehat{P}^k_{\min,h}(s,a) = \widehat{P}^k_h(s_2 \mid s,a)$ for some $s_1, s_2 \in \mathcal{S}$. Then, it follows that
\begin{align}
P^{\star}_{\min,h}(s,a) = P^{\star}_h(s_1 \mid s,a) 
&\overset{\text{(i)}}{\geq} \frac{\widehat{P}^k_h(s_1 \mid s,a)}{e^2} 
\geq \frac{\widehat{P}^k_{\min,h}(s,a)}{e^2} = \frac{\widehat{P}^k_h(s_2 \mid s,a)}{e^2} \overset{\text{(ii)}}{\geq} \frac{P^{\star}_h(s_2 \mid s,a)}{8e^2 L} 
\geq \frac{P^{\star}_{\min,h}(s,a)}{8e^2L} \overset{(iii)}{\geq} \frac{P^{\star}_{\min}}{8e^2L} ,\label{eq:binomial_bound_step2}
\end{align}
where the inequalities (i) and (ii) follow from \eqref{eq:binomial_bound_step1}, and inequality (iii) follows by \eqref{eq:p_star_min_KL}.

By applying \eqref{eq:binomial_bound_step2} in \eqref{eq:Bonus_term_KL_step1}, we get
\begin{align}
      B^{KL}_{k,h}(s,a) \leq \frac{2c_fH}{\sigma}\sqrt{\frac{L^2}{\{N^k_h(s,a)\vee 1\}P^{\star}_{\min}}} + \sqrt{\frac{1}{K}}. 
\end{align}
This concludes the proof of \Cref{lem:Control_Bonus_KL}.
\end{proof}

\subsection{Auxiliary Lemmas for {\RMDPKL}}
\label{subsubsec:aux_Lemma_KL}

\begin{auxlem}[Bound for {\RMDPKL} and Optimal Robust Value function]
\label{lem:Bound_KL_optimal_policy}
Under event \( \mathcal{E}_{KL} \) defined in \eqref{eq:Event_KL}, with probability at least $1-\delta$, it holds that
\begin{align}
\label{eq:Bound_KL_optimal_policy}
\abs{ \mathbb{E}_{\widehat{\mathcal{U}_h^\sigma}(s,a)}\left[V_{h+1}^{\star,\sigma}\right] 
- \mathbb{E}_{\mathcal{U}_h^\sigma(s,a)}\left[V_{h+1}^{\star,\sigma}\right]}\leq \frac{c_1H}{\sigma}\sqrt{\frac{L}{\{N^k_h(s,a)\vee 1\}\widehat{P}^k_{\min,h}(s,a)}} + \frac{1}{\sqrt{K}},
\end{align}
where \(L = \log(S^3 A H^2 K^{3/2} / \delta) \) and \( c_1 \) is an absolute constant.
\end{auxlem}

\begin{proof} 
By the definition of the operator 
\(\mathbb{E}_{\mathcal{U}_h^\sigma(s,a)}\left[V_{h+1}^{\star,\sigma}\right]\) 
in \eqref{eq:dual_KL} and $\widehat{P}^k_{\min,h}(s,a)$ in \eqref{eq:hat_p_min_KL},  we can arrive at
\begin{align}
   &\abs{ \mathbb{E}_{\widehat{\mathcal{U}_h^\sigma}(s,a)}\left[V_{h+1}^{\star,\sigma}\right] 
- \mathbb{E}_{\mathcal{U}_h^\sigma(s,a)}\left[V_{h+1}^{\star,\sigma}\right]} \nonumber\\
&\quad \leq \sup_{\eta \in [\underline{\eta}, H/\sigma]} \eta \abs{\log\bigg(\mathbb{E}_{\widehat{P}_h^k(\cdot|s,a)}\bigg[\exp\bigg\{-\frac{V_{h+1}^{\star,\sigma}}{\eta}\bigg\}\bigg]\bigg)  - \log\bigg(\mathbb{E}_{P^{\star}_h(\cdot|s,a)}\bigg[\exp\bigg\{-\frac{V_{h+1}^{\star,\sigma}}{\eta}\bigg\}\bigg]\bigg)}.
\end{align}
By the definition of $\mathcal{E}_{KL}$ as defined in \eqref{eq:event_II} and by applying \citep[Lemma 16]{NeuRIPS2024_UnifiedPessimismOfflineRL_Yue}, we have 
\begin{align}
&\abs{ \mathbb{E}_{\widehat{\mathcal{U}_h^\sigma}(s,a)}\left[V_{h+1}^{\star,\sigma}\right] 
- \mathbb{E}_{\mathcal{U}_h^\sigma(s,a)}\left[V_{h+1}^{\star,\sigma}\right]} \leq \frac{c_1H}{\sigma}\sqrt{\frac{L}{\{N^k_h(s,a)\vee 1\}\widehat{P}^k_{\min,h}(s,a)}},
\end{align}
for any $\eta \in \mathcal{N}_{\frac{1}{\sigma S\sqrt{K}}}\big([0,H/\sigma]\big)$. Therefore, by a covering argument, for any $\eta \in [0,H/\sigma]$, we get \eqref{eq:Bound_KL_optimal_policy}. This concludes the proof of Lemma \ref{lem:Bound_KL_optimal_policy}.
\end{proof}
%%%%%%%%%%%%%%%%%%%%%%%%%%%%%%%%%%

\begin{auxlem}[Bound for {\RMDPKL} and the robust value function of \( \pi^k \)]
\label{lem:Bound_KL_policy_k}
Under event \( \mathcal{E}^{KL} \) in \eqref{eq:Event_KL}, suppose that the optimism and pessimism ineq. \eqref{eq:Optimistic_pessimism_ineq_KL} hold at \( (h+1, k) \), then it holds that
\begin{align}
\label{eq:Bound_KL_policy_k}
&\abs{\mathbb{E}_{\widehat{\mathcal{U}_h^\sigma}(s,a)} 
\left[ V_{h+1}^{\pi^k,\sigma} \right]-\mathbb{E}_{\mathcal{U}_h^\sigma(s,a)} 
\left[ V_{h+1}^{\pi^k,\sigma} \right]}\leq \frac{c_1H}{\sigma}\sqrt{\frac{L}{\{N^k_h(s,a)\vee 1\}\widehat{P}^k_{\min,h}(s,a)}} + \frac{1}{\sqrt{K}},
\end{align}
where \(L= \log\bigg(\frac{S^3 A H^2 K^{3/2}}{\delta}\bigg) \) and \( c_1 \) is an absolute constant.
\end{auxlem}

\begin{proof}
By our definition of the operator \( \mathbb{E}_{\mathcal{U}_h^\sigma(s,a)}[V_{h+1}^{\pi^k,\sigma}] \) in~\eqref{eq:dual_KL} and $\widehat{P}^k_{\min,h}(s,a)$ in \eqref{eq:hat_p_min_KL},  we can arrive at
\begin{align}
   &\abs{ 
\mathbb{E}_{\widehat{\mathcal{U}_h^\sigma}(s,a)} \left[ V_{h+1}^{\pi^k,\sigma} \right] 
- \mathbb{E}_{\mathcal{U}_h^\sigma(s,a)} \left[ V_{h+1}^{\pi^k,\sigma} \right]} \nonumber\\
&\quad \leq \sup_{\eta \in [\underline{\eta}, H/\sigma]} \eta \abs{\log\bigg(\mathbb{E}_{\widehat{P}_h^k(\cdot|s,a)}\bigg[\exp\bigg\{-\frac{V_{h+1}^{\pi^k,\sigma}}{\eta}\bigg\}\bigg]\bigg)  - \log\bigg(\mathbb{E}_{P^{\star}_h(\cdot|s,a)}\bigg[\exp\bigg\{-\frac{V_{h+1}^{\pi^k,\sigma}}{\eta}\bigg\}\bigg]\bigg)}.
\end{align}
By the definition of $\mathcal{E}_{KL}$ as defined in \eqref{eq:event_II} and by applying \citep[Lemma 17]{NeuRIPS2024_UnifiedPessimismOfflineRL_Yue}, we can arrive at
\begin{align}
&\abs{ 
\mathbb{E}_{\widehat{\mathcal{U}_h^\sigma}(s,a)} \left[ V_{h+1}^{\pi^k,\sigma} \right] 
- \mathbb{E}_{\mathcal{U}_h^\sigma(s,a)} \left[ V_{h+1}^{\pi^k,\sigma} \right]}\leq \frac{c_1H}{\sigma}\sqrt{\frac{L}{\{N^k_h(s,a)\vee 1\}\widehat{P}^k_{\min,h}(s,a)}},
\end{align}
for any $\eta \in \mathcal{N}_{\frac{1}{\sigma S\sqrt{K}}}\big([0,H/\sigma]\big)$. Therefore, by a covering argument, for any $\eta \in [0,H/\sigma]$, we get \eqref{eq:Bound_KL_policy_k}. This concludes the proof of Lemma \ref{lem:Bound_KL_policy_k}.\qedhere
\end{proof}

\begin{auxlem}[Bounds for {\RMDPKL} and optimistic and pessimistic robust value estimators]
\label{lem:Bound_KL_optimism_pessimism}
Under event $\mathcal{E}_{KL}$ in \eqref{eq:Event_KL}, suppose that the optimism and pessimism ineq. \eqref{eq:Optimistic_pessimism_ineq_KL} holds at $(h+1,k)$, it holds that
\begin{align}
&\max \Bigg\{\abs{\mathbb{E}_{\widehat{\mathcal{U}_h^\sigma}(s,a)}\left[\overline{V}_{h+1}^k\right] - \mathbb{E}_{\mathcal{U}_h^\sigma(s,a)}\left[\overline{V}_{h+1}^k\right]},\abs{\mathbb{E}_{\widehat{\mathcal{U}_h^\sigma}(s,a)}\left[\underline{V}_{h+1}^k\right] - \mathbb{E}_{\mathcal{U}_h^\sigma(s,a)}\left[\underline{V}_{h+1}^k\right]}
\Bigg\}\nonumber\\
&\qquad \qquad \leq  \frac{c_1H}{\sigma}\sqrt{\frac{L}{\{N^k_h(s,a)\vee 1\}\widehat{P}^k_{\min,h}(s,a)}} +\sqrt{\frac{1}{K}},
\end{align}
where $L= \log(S^3 A H^2 K^{3/2} / \delta)$ and $c_1$ is an absolute constant.
\end{auxlem}

\begin{proof}
    We follow the same proof lines as Lemma ~\ref{lem:Bound_KL_policy_k}, and thereby we omit it.
\end{proof}

\section{Proofs of Lower Bound of {\Algoname}}
\label{app:thm:regret_lower_bound}
In this section, we present the proof of the minimax regret lower bound stated in \cref{thm:regret_lower_bound}. The proof is inspired by the techniques in \cite{Book2020_Bandit_Lattimore} and builds upon the lower bound framework introduced in \cite{ICML2025_OnlineDRMDPSampleComplexity_He}. While we adopt the same hard instance construction as in \cite{ICML2025_OnlineDRMDPSampleComplexity_He}, our analysis generalizes it to arbitrary horizon length $H$ and number of states $S$. Notably, our lower bound does not rely on the supremal visitation ratio assumption ($C_{vr}$) that is required in their analysis. To this end, we first introduce the construction of hard instances, and then we prove \cref{thm:regret_lower_bound}. We have introduced some notations that will be used throughout the analysis

\subsection{Construction of a collection of hard RMDPs}
\label{app:hard_instance}

\begin{figure}[t]
\centering
\subfloat[The nominal transition kernel $P^\star$]{%
  \includegraphics[scale=0.75]{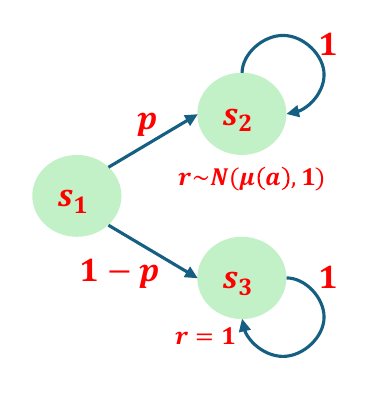}
  \label{fig:nominal_transition_LB}}
  \hspace{10mm}
% \hfill
\subfloat[The worst-case transition kernel $P^\omega$]{%
  \includegraphics[scale=0.75]{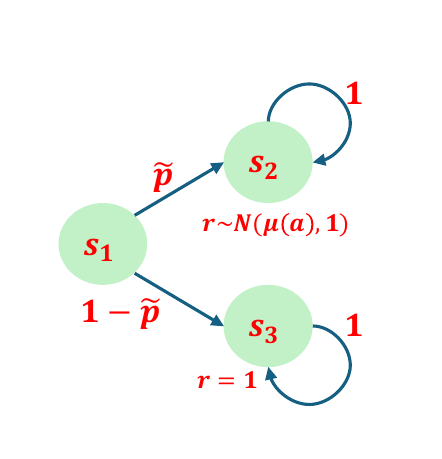}
  \label{fig:perturbed_transition_LB}
}
\label{fig:tranition kernel}
\end{figure}

We consider an episodic finite-horizon RMDP defined by the tuple $\mathcal{M}:= (\mathcal{S}, \mathcal{A}, H, P^\star, \mathcal{U}^{\sigma}(P^\star), r)$, where the state space is $\mathcal{S} = \{s_1, s_2, s_3\}$, the action space is $\mathcal{A} = \{a_1, a_2, \dots, a_A\}$ with $A = |\mathcal{A}|$, and $H$ denotes the horizon length. 

The nominal transition kernel $P^\star$, illustrated in \Cref{fig:nominal_transition_LB}, is defined as follows:
\begin{align}
\label{eq:nominal_P}
    P^\star_h(s' \mid s, a) &=
    \begin{cases}
        p\,\mathbb{I}\{s' = s_2\} + (1 - p)\,\mathbb{I}\{s' = s_3\}, & \text{if } s = s_1,\ h = 1, \\
        \mathbb{I}\{s' = s_2\}, & \text{if } s = s_2,\ h = 2,\dots,H, \\
        \mathbb{I}\{s' = s_3\}, & \text{if } s = s_3,\ h = 2,\dots,H.
    \end{cases}
\end{align}

It is evident from the definition that $s_1$ serves as the initial state in every episode. At step $h = 1$, it transitions to $s_2$ with probability $p$ and to $s_3$ with probability $1 - p$. In subsequent steps $h = 2, \dots, H$, the process remains in $s_2$ or $s_3$, independent of the chosen action. We consider $K$ episodes, and assume $K > A$ and $p \geq A/K$ to support our construction; this lower bound on $p$ becomes looser as $K$ increases. 

The reward structure is as follows: in state $s_2$, the agent receives a reward drawn from a Gaussian distribution $r \sim \mathcal{N}(\mu(a), 1)$, where $\mu(a) \in [0, 1)$ depends on the action taken. In contrast, state $s_3$ yields a fixed reward of $r = 1$. The choice of $\mu(a)$ ensures that the robust value function satisfies $V^{\pi, \sigma}_h(s_2) < V^{\pi, \sigma}_h(s_3)$ for all $h = 2, \dots, H$. Consequently, the worst-case environment will not reduce the transition probability from $s_1$ to $s_2$ relative to the nominal model.

We consider $S$ copies of RMDPs, i.e. ${\bf \Xi}:= \{\mathcal{M}_1, \mathcal{M}_2,\dots, \mathcal{M}_S\}$, where $A\geq cS$,  for $c>1$, and $S\geq 4$. The only difference between any two RMDPs  in ${\bf \Xi}$ is the mean reward at $s_2$. Without loss of generality, we consider the mean reward vector at $s_2$ for RMDP $\mathcal{M}_1$ as 
\begin{align}
\label{eq:mean_M1}
    \mu_1:=(\mu^\star,0,\dots,0)\in \mathbb{R}^A,
\end{align} 
and $\mu*>0$ is a constant to be specified later. Note that $s_1$ serves as the initial state in every episode and for every $\mathcal{M}_i \in {\bf \Xi}$.

%. 
The policy $\pi:=\{\pi^k\}_{k=1}^K$ for any $\mathcal{M}_i \in {\bf \Xi}$ is as follows: The agent follows a learning algorithm $\xi$, and in the $k$-th episode, it selects a policy $\pi^k:=\{\pi^k_h\}_{h=1}^H$ to interact with the environment, resulting in the trajectory $\{s^k_h, a^k_h, r^k_h\}_{h=1}^H$, where for episode $k$
\begin{itemize}
    \item $s^k_h$ denotes the state to which the agent transitions from $s_1$ at step $h = 1$, after taking an arbitrary action;
    \item $a^k_h$ is the action selected by the agent at step $h=2$, i.e., in state $s_2$ or $s_3$;
    \item $r^k_h$ is the reward received at state $s_2$ or $s_3$ in step $h=2$.
\end{itemize}
Let $\big\{s^k_h, a^k_h, r^k_h\big\}_{h=1,k=1}^{H,K}$ denote the collection of trajectories over $K$ episodes. The joint distribution over these trajectories for RMDP $\mathcal{M}_i$ is denoted by $\mathbb{P}^0_i$, and the corresponding random variables are denoted by $S^k_h$, $A^k_h$, and $R^k_h$, respectively. We write $\mathbb{E}^0_i$ to denote the expectation with respect to $\mathbb{P}^0_i$. Following the standard argument from \cite{Book2020_Bandit_Lattimore}, it is sufficient to restrict attention to deterministic learning strategies without loss of generality.

\subsection{Define some events and the mean reward vectors for RMDPs}
We now define two events that will be used in the proof. Let $M_j(K)$ denote the number of episodes in which agent selects action $a_j$ in step $h=2$ if the transit state is $s_2$ according to the learned policy $\pi$, and is given by
\begin{align}
\label{eq:event_M}
    M_j(K)=\sum\limits_{k=1}^K \mathbb{I}\{\pi^k_2(s_2)=a_j\}.
\end{align}
We now denote $N_j(K)$ as the number of episodes in which action $a_j$ is actually taken at state $s_2$ at step $h=2$, and is given by
\begin{align}
    \label{eq:event_N}
    N_j(K)=\sum\limits_{k=1}^K \mathbb{I}\{S^k_2=s_2, A^k_2=a_j\}.
\end{align}

We denote $c_i = \arg\min_{j \notin \{c_1, c_2, \dots, c_{i-1}\}} \mathbb{E}^0_{1}[T_j(K)]$, as the $i^\text{th}$ least chosen action under environment $\mathcal{M}_1$,  after excluding the previously selected least chosen actions $c_1, \dots, c_{i-1}$, where $c_1 = 1$. Since the expected number of times the agent transits from $s_1$ to $s_2$ over $K$ episodes is $Kp$, and given the selection of $c_i$, it follows that
\begin{align}
\label{eq:bound_N}
    \mathbb{E}^0_1[N_{c_i}(K)]\leq \frac{Kp}{A-S}\leq \frac{Kp}{c'A}, \text{ where $c'=\big(1-\frac{1}{c}\big)$.} 
\end{align}

We now define the mean reward vector at state $s_2$ in environment $\mathcal{M}_i \in {\bf \Xi}$ as
\begin{align}
\label{eq:mean_Ms}
\mu_i(j) =
\begin{cases}
\mu^\star, & \text{if } j = 1 \\
\sqrt{S}\mu^\star, & \text{if } j = c_i\\
0, & \text{otherwise}
\end{cases}
\end{align}
This implies that $\mu_i(a_j)=\mu_1(a_j)=\mu^\star$  for all $j\neq c_i$ and $\mu_i(a_j)=\sqrt{S}\mu^\star$  for all $j=c_i$. Therefore, the optimal policy at state $s_2$ selects action $a_1$ under $\mathcal{M}_1$ and action $a_{c_i}$ under $\mathcal{M}_i$.

\subsection{Expected Regret under RMDP $\mathcal{M}_i \in {\bf \Xi}$}

Before we evaluate the robust value functions, we first examine the optimization problem for identifying the worst-case transition from state \( s_1 \) to the two possible next states \( s_2 \) and \( s_3 \) at step $h=1$, formulated as:
\begin{align}
\label{eq:opt_h1}
P^{\omega}_1 &= \arg\min_{p': D(P^{\omega}_1 \parallel P^\star_1) \leq \sigma} \left[ p' \cdot V^{\pi, \sigma}_2(s_2) + (1 - p') \cdot V^{\pi, \sigma}_2(s_3) \right],
\end{align}
where at step $h=1$ the worst-case transition kernel \( P^{\omega}_1 = (p', 1 - p') \) and $P^\star_1 = (p, 1 - p)$.  
By the construction of RMDPs $\mathcal{M}_i$ we have the fact that $V^{\pi, \sigma}_h(s_2) <  V^{\pi, \sigma}_h(s_3)$ by design. Consequently, the expression $p' \cdot V^{\pi, \sigma}_2(s_2) + (1 - p') \cdot V^{\pi, \sigma}_2(s_3)$  is a decreasing function of $p'$. Thus, $P^{\omega}_1$ for step $h=1$ is determined by the supremum value $\tilde{p} = \sup \{ p' : D_f(P^\omega_1, P^\star_1) \leq \sigma \}$, where \( D_f \) denotes \( f \)-divergence (e.g., KL divergence, or \( \chi^2 \)-divergence). For steps $h=2,\dots,H$, the $P^{\omega}_h=1$ for state $s_2$ or $s_3$ according to the definition of $f$-divergence set (\Cref{def:f_divergence_uncertainty}) for {\RMDPchi} and {\RMDPKL}. This formulation ensures $P^{\omega}$ is independent of the policy \( \pi \) and environments \( \mathcal{M}_i \) for $i=1,\dots,S$. Therefore, the worst-transition kernel $P^\omega$, as illustrated in \Cref{fig:perturbed_transition_LB}, is given by
\begin{align}
\label{eq:perturbed_P}
    P^\omega_h(s' \mid s, a) &=
    \begin{cases}
        \tilde{p}\,\mathbb{I}\{s' = s_2\} + (1 - \tilde{p})\,\mathbb{I}\{s' = s_3\}, & \text{if } s = s_1,\ h = 1, \\
        \mathbb{I}\{s' = s_2\}, & \text{if } s = s_2,\ h = 2,\dots,H, \\
        \mathbb{I}\{s' = s_3\}, & \text{if } s = s_3,\ h = 2,\dots,H.
    \end{cases}
\end{align}

\textbf{Robust value functions.} We now construct the robust value functions for any policy $\pi$ and optimal policy $\pi^\star$ at states $s_2$ and $s_3$ for $h=2,\dots,H$ and at state $s_1$ for $h=1$. For $h=2,\dots,H$, we get
\begin{align}
\label{eq:value_pi_h_s3}
    V^{\pi,\sigma}_h(s_3) = (H-h+1).\nonumber\\
    V^{\pi^\star,\sigma}_h(s_3) = (H-h+1).
\end{align}
as the reward at $s_3$ is 1.

\begin{align}
\label{eq:value_pi_h_s2}
    V^{\pi,\sigma}_h(s_2) = (H-h+1)\mu(\pi^k_2).\nonumber\\
    V^{\pi^\star,\sigma}_h(s_2) = (H-h+1)\mu(\pi^\star_2).
\end{align}

For $h=1$, we apply the worst-transition kernel $P_1^\omega$ as given in \eqref{eq:perturbed_P}, and we get
\begin{align}
\label{eq:value_pi_h_s1}
     V^{\pi,\sigma}_h(s_1) = \tilde{p}(H-1) \mu(\pi^k_2) + (1-\tilde{p})(H-1).\nonumber\\
    V^{\pi^\star,\sigma}_h(s_1) = \tilde{p}(H-1) \mu(\pi^\star_2) + (1-\tilde{p})(H-1).
\end{align}
Therefore, regret over $K$ episodes is given by
\begin{align}
\label{eq:regret_LB}
    \text{Regret}(K) = \sum\limits_{k=1}^K V^{\star,\sigma}_1(s_1) - V^{\pi^k,\sigma}_1(s_1)= \sum\limits_{k=1}^K \tilde{p}(H-1)\big(\mu(\pi^\star_2)-\mu(\pi^k_2) \big).
\end{align}

For environment \( \mathcal{M}_1 \), the algorithm $\xi$ incurs the total regret by selecting a suboptimal policy \( \pi^k_2(s_2) \neq a_1 \) in each episode, i.e. $\mu(\pi^\star_2)-\mu(\pi^k_2) = \mu^\star\mathbb{I}\{\pi^k_2(s_2)\neq a_1\}$. Applying this fact in \eqref{eq:regret_LB}, the regret under $\mathcal{M}_1$ cam be bounded as 
\begin{align}
\mathbb{E}[\text{Regret}_{\mathcal{M}_1}(\xi, K)] &= \sum\limits_{k=1}^K \mathbb{E}^0_1\big[V^{\star,\sigma}_1(s_1) - V^{\pi^k,\sigma}_1(s_1)\big]\nonumber\\
&\overset{\text{(i)}}{=} \sum_{k=1}^{K} \mathbb{E}^0_1 \left[ \tilde{p} (H-1)\cdot \mathbb{I}\{\pi^k_2(s_2) \neq a_1\} \mu^\star \right] \nonumber\\
&= \tilde{p} (H-1)\cdot \mu^\star\cdot \mathbb{E}^0_1 \left[ \sum_{k=1}^{K} \mathbb{I}\{\pi^k(s_2) \neq a_1\} \right] \nonumber\\
&\overset{\text{(ii)}}{=}\tilde{p} (H-1)\cdot \mu^\star \cdot \mathbb{E}^0_1 [K - M_1(K)] \notag \\
&\overset{\text{(iii)}}{\geq} \frac{\tilde{p}(H-1) KSH \mu^\star}{2} \mathbb{P}^0_1 \left( M_1(K) \leq \frac{KSH}{2} \right), \label{eq:exp_regret_M1}
\end{align}
where equality (i) follows from the definition of the robust value function under the RMDP setting, using \( \tilde{p} \) as the worst-case transition probability, and the construction of the reward function at state \( s = s_2 \), (ii) uses the definition of \( M_j(K) \) as given in \eqref{eq:event_M}, and (iii) follows from Markov’s inequality.

For environment \( \mathcal{M}_i \in {\bf \Xi}\), the algorithm $\xi$ follows the same analysis, and the total regret incurred by selecting a suboptimal policy \( \pi^k_2(s_2) \neq a_{c_i} \) in each episode can be bounded as
\begin{align}
\mathbb{E}[\text{Regret}_{\mathcal{M}_i}(\xi, K)] 
&= \sum_{k=1}^{K} \mathbb{E}_{i}^0 \big[V^{\star,\sigma}_1(s_1) - V^{\pi^k,\sigma}_1(s_1)\big] \notag \\
&= \sum_{k=1}^{K} \mathbb{E}_i^0 \left[ \tilde{p}(H-1) \cdot \left( \mathbb{I}\{\pi^k_2(s_2) = a_1\} \big(\sqrt{S}-1\big)\mu^\star + \sum_{j>1, j \neq c_i} \mathbb{I}\{\pi^k_2(s_2) = a_j\} \sqrt{S}\mu^\star \right) \right] \notag \\
&\geq \sum_{k=1}^{K} \mathbb{E}_i^0 \left[ \tilde{p}(H-1)\big(\sqrt{S}-1\big) \cdot \mathbb{I}\{\pi^k_2(s_2) = a_1\} \mu^\star \right] \notag \\
&\overset{(i)}{\geq} \tilde{p}(H-1) \cdot \mu^\star \cdot \mathbb{E}_i^0 \left[ \sum_{k=1}^{K} \mathbb{I}\{\pi^k_2(s_2) = a_1\} \right] \notag \\
&= \tilde{p}(H-1) \cdot \mu^\star\cdot \mathbb{E}_i^0 [M_1(K)] \notag \\
&\geq \frac{\tilde{p}(H-1) KSH\mu^\star}{2} \mathbb{P}_i^0 \left( M_1(K) > \frac{KSH}{2} \right), \label{eq:exp_regret_Ms}
\end{align}
where (i) is due to the fact that $S\geq 4$.

Combining \eqref{eq:exp_regret_M1} and \eqref{eq:exp_regret_Ms} and by applying \Cref{lem:Bretagnolle_Huber_inequality}, we get
\begin{align}
    \mathbb{E}[\text{Regret}_{\mathcal{M}_1}(\xi, K) + \text{Regret}_{\mathcal{M}_i}(\xi, K)] &\geq \frac{\tilde{p}(H-1) KSH \mu^\star}{2} \bigg\{\mathbb{P}_1^0 \left( M_1(K) \leq  \frac{KSH}{2}\right) +  \mathbb{P}_i^0 \left( M_1(K) > \frac{KSH}{2}\right)\bigg\}\nonumber\\
    &\geq \frac{\tilde{p}(H-1) KSH \mu^\star}{2} \exp\Big(-\text{KL}(\mathbb{P}^0_1,\mathbb{P}^0_i)\Big).\label{eq:Regret_M1_Mi_bound}
\end{align}

According to \Cref{lem:KL_P1_Pi} and \eqref{eq:bound_N}, we can bound $\text{KL}(\mathbb{P}^0_1,\mathbb{P}^0_i)$ as
\begin{align}
    \text{KL}(\mathbb{P}^0_1,\mathbb{P}^0_i) &= \sum\limits_{j=1}^A
\mathbb{E}^0_1[N_j(K)]\text{KL}\Big(P_{r_{\mathcal{M}_1}(s_2,a_j)}, P_{r_{\mathcal{M}_i}(s_2,a_j)}\Big)\nonumber\\
&= \mathbb{E}^0_1[N_{c_i}(K)]\text{KL}\Big(\mathcal{N}(0,1), \mathcal{N}(\sqrt{i}\mu^*,1)\Big)\nonumber\\
&\leq \frac{S(\mu^\star)^2Kp}{2c'A}. \label{eq:KL_P1_Pi_bound}
\end{align}
 By applying \eqref{eq:KL_P1_Pi_bound} in \eqref{eq:Regret_M1_Mi_bound}, we get
 \begin{align}
      \mathbb{E}[\text{Regret}_{\mathcal{M}_1}(\xi, K) + \text{Regret}_{\mathcal{M}_i}(\xi, K)] &\geq \frac{\tilde{p}(H-1) KSH \mu^\star}{2} \exp\bigg(-\frac{S(\mu^\star)^2Kp}{2c'A}\bigg)\nonumber\\
      &\overset{\text{(i)}}{=} \frac{e^{-1/2}\tilde{p}(H-1)SH\sqrt{Kc'A}}{\sqrt{Sp}}\nonumber\\
      &=\frac{e^{-1/2}\tilde{p}(H-1)H}{\sqrt{p}}\sqrt{Kc'SA}\nonumber\\
      &{=} \Omega\bigg(\frac{\tilde{p}}{\sqrt{p}}\sqrt{H^4KSA}\bigg)\label{eq:Regret_M1_Mi_final_bound},
 \end{align}
where inequality (i) is obtained by $\mu^\star = \sqrt{\frac{c'A}{SKp}}$.

\subsection{Total expected regret under environment ${\bf \Xi}$}
To proceed, we present the lower bound on the sum of the total regret under environment ${\bf \Xi}$. We define the following notations
\begin{align}
    I &:= \sup_{\mathcal{M}_i \in {\bf \Xi}} \mathbb{E}[\text{Regret}_{\mathcal{M}_i}(\xi, K)].\label{eq:event_I}\\
        II &:= \sum\limits_{i=2}^S \mathbb{E}[\text{Regret}_{\mathcal{M}_1}(\xi, K)]+ \mathbb{E}[\text{Regret}_{\mathcal{M}_i}(\xi, K)].\label{eq:event_II}
\end{align}
By \eqref{eq:event_I}, \eqref{eq:event_II} and \eqref{eq:Regret_M1_Mi_final_bound}, we can verify that
\begin{align}
\label{eq:event_I_bound}
    I &\geq 2(S-1)II\nonumber\\
    &=\frac{1}{2(S-1)}\sum\limits_{i=2}^S \bigg[\mathbb{E}[\text{Regret}_{\mathcal{M}_1}(\xi, K)]+ \mathbb{E}[\text{Regret}_{\mathcal{M}_i}(\xi, K)]\bigg]\nonumber\\
    &=\Omega\bigg(\frac{\tilde{p}}{\sqrt{p}}\sqrt{H^4KSA}\bigg).
\end{align}

\subsection{Proof of \Cref{thm:regret_lower_bound}}
We will now find the lower bound for {\RMDPchi} and {\RMDPKL}.

\textbf{Lower bound for {\RMDPchi}.}
Following the definition of $\chi^2$-divergence, the worst-case transition for step $h=1$ is given by
\begin{align}
    \tilde{p}&=\argmin_{p'}  p'V^{\pi,\sigma}_2(s_2) + (1-p')V^{\pi,\sigma}_2(s_3) \label{eq:tilde_p_chi}\\
    &\text{s.t. } p\bigg(\frac{p'}{p}-1\bigg)^2 + (1-p)\bigg(\frac{1-p'}{1-p}-1\bigg)^2 \leq \sigma\nonumber
\end{align}
Since \eqref{eq:tilde_p_chi} decrease monotonically with $\tilde{p}$, the solution of the optimization problem is $\tilde{p}=p+\sqrt{\sigma p(1-p)}$ \cite[Lemma 5.14]{ICML2025_OnlineDRMDPSampleComplexity_He}. By choosing $p=\frac{1}{2}$, we get $\tilde{p}= \frac{1}{2}\big(1+\sqrt{\sigma}\big)\geq \frac{1}{2}\sqrt{(1+\sigma)}$, as $\sqrt{a}+\sqrt{b}\geq \sqrt{a+b}$. Now by applying the choice of $\tilde{p}$ for {\RMDPchi}, we get
\begin{align}
\label{eq:I_bound_RMDP_chi}
    \sup_{\mathcal{M}_i \in {\bf \Xi}} \mathbb{E}[\text{Regret}_{\mathcal{M}_i}(\xi,K)]\geq \mathcal{O}\bigg(\sqrt{H^4(1+\sigma)SAK}\bigg).
\end{align}

\textbf{Lower bound for {\RMDPKL}.} The proof of the lower bound for {\RMDPKL} is inspired by the proof of Theorem 4 in \cite{JMLR2024_DROfflineRLNearOptimalSampelComplexity_Shi}, and for our case we consider the case when the uncertainty level $\sigma$ is relatively large. For {\RMDPKL}, we set 
\begin{align}
    p=1-\alpha, \text{ such that $0< \alpha \leq \frac{3}{2H} \leq \frac{3}{4e^8}\leq \frac{1}{2}$ for some $H\geq 2e^8$}.\label{eq:P_value_KL}
\end{align}
The assumption \eqref{eq:P_value_KL} satisfies that $p\in [1/2,1)$. We now set 
\begin{align}
    \beta := \frac{\log\big(\frac{1}{\alpha}\big)}{2}\geq \frac{\log\big(\frac{2H}{3}\big)}{2}\geq 4\label{eq:beta_KL}
\end{align}
We now consider the radius of the KL-divergence uncertainty set $\sigma$ as
\begin{align}
    \label{eq:sigma_KL_range}
    \bigg(1-\frac{3}{\beta}\bigg) \log\bigg(\frac{1}{\alpha}\bigg) \leq \sigma \leq  \bigg(1-\frac{2}{\beta}\bigg) \log\bigg(\frac{1}{\alpha}\bigg).
\end{align}

According to $\beta$ defined in \eqref{eq:beta_KL} and $\sigma$ satisfying \eqref{eq:sigma_KL_range}, we can verify from \Cref{lem:tilde_p_bound_KL}, that $\tilde{p}\geq \frac{1}{\beta}$.\\ 

Recalling the definition of $P^\star_{\min}$ as given in \eqref{eq:p_star_min_KL}, the smallest positive state transition probability of
the optimal policy $\pi^\star$ under any RMDP $\mathcal{M}_i$ is given by
\begin{align}
\label{eq:P_star_min_KL}
    P^\star_{\min} = 1-p, \text{ where by \eqref{eq:P_value_KL} obeys } \alpha = P^\star_{\min} \in (0,1/H].
\end{align}
Therefore,  by \eqref{eq:P_star_min_KL} and \Cref{lem:tilde_p_bound_KL} in \eqref{eq:event_I_bound}, the lower bound for regret for {\RMDPKL} is given by
\begin{align}
       \sup_{\mathcal{M}_i \in {\bf \Xi}} \mathbb{E}[\text{Regret}_{\mathcal{M}_i}(\xi,K)] &\overset{(i)}{\geq} \Omega\bigg(\frac{1}{\beta\sqrt{1-P^\star_{\min}}}\sqrt{H^5SAK}\bigg)\nonumber\\
       &\overset{(ii)}{=} \Omega\bigg(\frac{2}{\log(1/\alpha)} \sqrt{\frac{H^5SAK}{1-P^\star_{\min}}}\bigg)\nonumber\\
       &\overset{(iii)}{\geq} \Omega\bigg(\frac{(1-3/\beta)}{\sigma} \sqrt{\frac{H^5SAK}{1-P^\star_{\min}}}\bigg)\nonumber\\
       &\overset{(iv)}{=}\Omega\bigg(\sqrt{\frac{H^5SAK}{\sigma^2(1-P^\star_{\min})}}\bigg),
\end{align}
where inequality (i) is obtained by $\tilde{p}\geq \frac{1}{\beta}$ and $p=1-P^\star_{\min}$, (ii) is by putting the value $\beta$ given in \eqref{eq:beta_KL}, (iii) is by the fact that the range of uncertainty level $\sigma$ satisfies \eqref{eq:sigma_KL_range}, finally by using fact $\beta\geq 4$ in inequality (iv) we get the final bound.

%%%%%%%%%%%%%%%%%%%%%% Technical Lemma %%%%%%%%%%%%%%%%%%%%%%%%%%

\section{Technical Lemmas}
Here, we present some other technical lemmas which are useful in the proof.

\begin{techlem}[Azuma Hoeffding's Inequality]
    \label{lem:Azuma-Hoeffding}
    Let $\{Z_t\}_{t \in \mathbb{Z}_+}$ be a martingale with respect to the filtration $\{\mathcal{F}_t\}_{t\in \mathbb{Z}_+}$. Assume that there are predictable processes $\{A_t\}_{t \in \mathbb{Z}_+}$ and $\{B_t\}_{t \in \mathbb{Z}_+}$ with respect to $\{\mathcal{F}_t\}_{t\in \mathbb{Z}_+}$, i.e., for all $t$, $A_t$ and $B_t$ are $\mathcal{F}_{t-1}$-measurable, and constants $0 < c_1, c_2, \dots < +\infty$ such that $A_t \leq Z_t - Z_{t-1} \leq B_t$ and $B_t - A_t \leq c_t$ almost surely. Then, for all $\beta > 0$
    \begin{align}
\mathbb{P}\bigg(\abs{Z_t - Z_0} \geq \beta\bigg) \leq \exp\Bigg\{-\frac{2\beta^2}{\sum\limits_{i \leq t}c^2_t}\Bigg\}.
\end{align}
\end{techlem}
\begin{proof}
    Refer to the proof of Theorem 5.1 of \cite{Book2009_ConcIneq_Dubhashi}.
\end{proof}

\begin{techlem}[Self-bounding variance inequality {\citep[Theorem 10]{Arxiv2009_EmpBernsteinBounds_Maurer}}]
\label{lem:self_bound_variance}
Let $X_1, \ldots, X_T$ be independent and identically distributed random variables with finite variance, that is, $\operatorname{Var}(X_1) < \infty$. Assume that $X_t \in [0, M]$ for every $t$ with $M > 0$, and let
\[
S_T^2 = \frac{1}{T} \sum_{t=1}^T X_t^2 - \left( \frac{1}{T} \sum_{t=1}^T X_t \right)^2.
\]
Then, for any $\varepsilon > 0$, we have
\[
\mathbb{P} \left( \left| S_T - \sqrt{\operatorname{Var}(X_1)} \right| \geq \varepsilon \right)
\leq 2 \exp\left( - \frac{T \varepsilon^2}{2M^2} \right).
\]
\end{techlem}

\begin{proof}
    Refer to the proof of Lemma 7 of \cite{PMLR2022_SampleComplexityRORL_Panaganti}.
\end{proof}

\begin{techlem}[Lemma 7.5 in \cite{Book2019_RL_Agarwal}]
\label{lem:inverse_count_bound}
    For the sequences of $\{s_h^k, a_h^k\}_{h,k=1}^{H,K}$, it holds that
\begin{align}
  \sum_{k=1}^{K} \sum_{h=1}^{H} \frac{1}{\{N_h^k(s_h^k, a_h^k) \vee 1\}} \leq c'HSA \log(K).  
\end{align}
where $c' > 0$ is an absolute constant.
\end{techlem}
\begin{proof}
     Refer to the proof of Lemma 7.5 in \cite{Book2019_RL_Agarwal}. \qedhere
\end{proof}

\begin{techlem}[Bound on Binomial random variable]
\label{lem:binomial_rv_bound}
    Suppose $X \sim \text{Binomial}(n, p)$, where $n \geq 1$ and $p \in [0, 1]$. For any $\delta \in (0,1)$, we have
\begin{align}
X &\geq \frac{np}{8 \log\left( \frac{1}{\delta} \right)} \hspace{9mm} \text{if } np \geq 8 \log\left( \frac{1}{\delta} \right), \label{eq:binomial_lower_bound}\\
X &\leq 
\begin{cases}
e^2 np & \text{if } np \geq \log\left( \frac{1}{\delta} \right), \\
2e^2 \log\left( \frac{1}{\delta} \right) & \text{if } np \leq 2 \log\left( \frac{1}{\delta} \right),
\end{cases} \label{eq:binomial_upper_bound}
\end{align}
hold with probability at least $1 - 4\delta$.
\end{techlem}

\begin{proof}
    Refer to \citep[Lemma 8]{NeuRIPS2023_CuriousPriceDRRLGenerativeMdel_Shi} for details.
\end{proof}

\begin{techlem}[\cite{Book2020_Bandit_Lattimore}]
    \label{lem:Bretagnolle_Huber_inequality}
    Let \( P \) and \( Q \) be probability measures on the same measurable space \( (\Omega, \mathcal{F}) \), and let \( A \in \mathcal{F} \) be an arbitrary event. Then
\begin{align*}
    P(A) + Q(A^c) \geq \frac{1}{2} \exp(-\mathrm{KL}(P, Q)).
\end{align*}
\end{techlem}

\begin{techlem}[\cite{Book2020_Bandit_Lattimore}]
\label{lem:normal_KL}
The KL-divergence between two Gaussian distributions with means \( \mu_1, \mu_2 \) and common variance \( \sigma^2 \) is
\begin{align*}
    \mathrm{KL}(\mathcal{N}(\mu_1, \sigma^2), \mathcal{N}(\mu_2, \sigma^2)) = \frac{(\mu_1 - \mu_2)^2}{2\sigma^2}.
\end{align*}
\end{techlem}
  
\begin{techlem}[KL-Divergence for $\mathbb{P}^0_1$ and $\mathbb{P}^0_i$]
\label{lem:KL_P1_Pi}   
For the distributions $\mathbb{P}^0_1$ and $\mathbb{P}^0_i$, we get the following
    \begin{align}
        \text{KL}(\mathbb{P}^0_1,\mathbb{P}^0_i) = \sum\limits_{j=1}^A
\mathbb{E}^0_1[N_j(K)]\text{KL}\Big(P_{r_{\mathcal{M}_1}(s_2,a_j)}, P_{r_{\mathcal{M}_i}(s_2,a_j)}\Big),
\end{align}
where $N_j(K)$ is defined in \eqref{eq:event_N}.
\end{techlem}
\begin{proof}
    We refer to the proof lines of Lemma E.3 in \cite{ICML2025_OnlineDRMDPSampleComplexity_He}, where we have $H$ horizons for each episode $k$. First, by the definition of KL-divergence, we have that
    \begin{align}
    \label{eq:KL_formualtion}
        \text{KL}(\mathbb{P}^0_1,\mathbb{P}^0_i) = \mathbb{E}^0_1\left[\log\bigg(\frac{d\mathbb{P}^0_1}{d\mathbb{P}^0_2}\bigg)\right].
    \end{align}
    We will now evaluate the Radon-Nikodym derivative of $\mathbb{P}^0_1$ as follows
    \begin{align}
      \label{eq:prob_density_M1}  &p^0_1\big(s^1_1,a^1_1,r^1_1,\dots,s^1_H,a^1_H,r^1_H,\dots, s^K_1,a^K_1,r^K_1,\dots,s^K_H,a^K_H,r^K_H\big) \nonumber\\
        &= \prod_{k=1}^K\prod_{h=1}^H \Pr\Big(s^k_h|s_1\Big)\pi^k_h\Big(a^k_h\mid s^1_1,a^1_1,r^1_1,\dots,s^{k-1}_{h-1},a^{k-1}_{h-1},r^{k-1}_{h-1},s^k_h\Big)\Pr\Big(r_{\mathcal{M}_1}(s^k_h,a^k_h)=r^k_h\Big).
    \end{align}
    The density of any $\mathbb{P}^0_i$ is exactly identical as in \eqref{eq:prob_density_M1}except that $r_{\mathcal{M}_1}$ is replaced by $r_{\mathcal{M}_2}$. This gives rise to
    \begin{align}
        \label{eq:log_P1_Pi_KL}
        \log\bigg(\frac{d\mathbb{P}^0_1}{d\mathbb{P}^0_2}\bigg) &= \sum\limits_{k=1}^K\sum_{h=1}^H \log\bigg(\frac{r_{\mathcal{M}_1}(s^k_h,a^k_h)=r^k_h}{r_{\mathcal{M}_i}(s^k_h,a^k_h)=r^k_h}\bigg).
    \end{align}
Note that in both $\mathcal{M}_1$ and $\mathcal{M}_i$, the agent does not observe any reward at step $h = 1$. At step $h = 2$, the agent transitions to either state $s_2$ or $s_3$. If it reaches state $s_3$, it receives a fixed reward of $r^k_2 = 1$. After receiving this reward at step $h = 2$, the agent remains in the same state for all subsequent steps $h = 3, \dots, H$, continuing to receive the same reward. Specifically, the reward at each step $h \geq 3$ is $r^k_h = 1$ if the agent is at state $s_3$, and $r^k_h = r^k_2$ if it is at state $s_2$. Using this structure, the expression in \eqref{eq:log_P1_Pi_KL} can be rewritten accordingly. \begin{align}
\label{eq:log_P1_Pi_KL_final}     \log\bigg(\frac{d\mathbb{P}^0_1}{d\mathbb{P}^0_2}\bigg) &= \sum\limits_{k=1}^K\mathbb{I}\{s^k_2=s_2\} \log\bigg(\frac{r_{\mathcal{M}_1}(s^k_2,a^k_2)=r^k_2}{r_{\mathcal{M}_i}(s^k_2,a^k_2)=r^k_2}\bigg).
 \end{align} 
 Taking expectation on both sides of \eqref{eq:log_P1_Pi_KL_final}, we get
 \begin{align}
     \label{eq:Exp_KL_P1_Pi}
     \mathbb{E}^0_1\left[\log\Bigg(\frac{d\mathbb{P}^0_1}{d\mathbb{P}^0_2}\bigg\{\bigg\{S^k_h,A^k_h,R^k_h\bigg\}_{h=1}^H\bigg\}_{k=1}^K\Bigg)\right] &= \sum\limits_{k=1}^K\mathbb{E}^0_1\left[\mathbb{I}\{S^k_2=s_2\} \log\bigg(\frac{\Pr_{r_{\mathcal{M}_1}(s_2,A^k_2)}(r^k_2)}{\Pr_{r_{\mathcal{M}_i}(s_2,A^k_2)}(r^k_2)}\bigg)\right]\nonumber\\
     &= \sum\limits_{k=1}^K\mathbb{E}^0_1\left[\mathbb{E}^0_1\left[\mathbb{I}\{S^k_2=s_2\} \log\bigg(\frac{\Pr_{r_{\mathcal{M}_1}(s_2,A^k_2)}(r^k_2)}{\Pr_{r_{\mathcal{M}_i}(s_2,A^k_2)}(r^k_2)}\bigg)\Bigg| S^k_2,A^k_2\right]\right]\nonumber\\
     &\overset{(i)}{=} \sum\limits_{k=1}^K\mathbb{E}^0_1\left[\mathbb{I}\{S^k_2=s_2\} \mathbb{E}^0_1\left[\log\bigg(\frac{\Pr_{r_{\mathcal{M}_1}(s_2,A^k_2)}(r^k_2)}{\Pr_{r_{\mathcal{M}_i}(s_2,A^k_2)}(r^k_2)}\bigg)\Bigg| S^k_2,A^k_2\right]\right]\nonumber\\
     &\overset{(ii)}{=} \sum\limits_{k=1}^K\mathbb{E}^0_1\left[\mathbb{I}\{S^k_2=s_2\}\text{KL}\bigg( P_{r_{\mathcal{M}_1}(s_2,A^k_2)},P_{r_{\mathcal{M}_i}(s_2,A^k_2)}\bigg)\right]\nonumber\\
     &= \sum\limits_{j=1}^A\mathbb{E}^0_1\left[\sum\limits_{k=1}^K\mathbb{I}\{S^k_2=s_2, A^k_2=a_j\}\text{KL}\bigg( P_{r_{\mathcal{M}_1}(s_2,a_j)},P_{r_{\mathcal{M}_i}(s_2,a_j)}\bigg)\right]\nonumber\\
      &= \sum\limits_{j=1}^A\mathbb{E}^0_1\left[N_j(K)\right]\text{KL}\bigg( P_{r_{\mathcal{M}_1}(s_2,a_j)},P_{r_{\mathcal{M}_i}(s_2,a_j)}\bigg),
 \end{align}
 where (i) is because $\mathbb{I}\{S^k_2=s_2\}$ is measurable with respect to the $\sigma$-field generated by $S^k_h$ and $A^k_h$, (ii) follows
from the definition of KL divergence, (iii) follows from the definition of $N_j(K)$ in \eqref{eq:event_N}. Combining \eqref{eq:Exp_KL_P1_Pi} with \eqref{eq:KL_formualtion}, we
conclude the proof.
\end{proof}

\begin{techlem}
\label{lem:KL_Ber_ineq}
For any $p,q\in [\frac{1}{2},1)$ and $p>q$, it holds that 
\begin{align*}
\text{KL}\bigg(Ber(p)\parallel Ber(q)\bigg) \leq \text{KL}\bigg(Ber(q)\parallel Ber(p)\bigg) \leq \frac{(p-q)^2}{p(1-p)}.
\end{align*}
Moreover, for any $0\leq x <y < q$, it holds
\begin{align*}
\text{KL}\bigg(Ber(x)\parallel Ber(q)\bigg) > \text{KL}\bigg(Ber(y)\parallel Ber(q)\bigg),
\end{align*}
where $Ber(p)$ denotes the Bernoulli distribution with mean $p$ and Kullback-Leibler (KL) divergence for $Ber(p)$ and $Ber(q)$ is defined as $\text{KL}\bigg(Ber(p)\parallel Ber(q)\bigg) := p\log\Big( \frac{p}{q}\Big)p + (1-p)\log\Big(\frac{1-p}{1-q}\Big)$
\end{techlem}
\begin{proof}
    The detailed proof is given in \cite[Lemma 16]{ICML2025_OnlineDRMDPSampleComplexity_He}.
\end{proof}

\begin{techlem}[Revised Lemma 20 in \cite{JMLR2024_DROfflineRLNearOptimalSampelComplexity_Shi}]
\label{lem:tilde_p_bound_KL}
    When $\beta$ satisfies \eqref{eq:beta_KL} and the uncertainty level $\sigma$ satisfies \eqref{eq:sigma_KL_range}, the worst-case transition kernel obeys
    \begin{align*}
        \tilde{p}\geq \frac{1}{\beta}.
    \end{align*}
\end{techlem}
\begin{proof}
    We follow the proof lines of \cite[Lemma 20]{JMLR2024_DROfflineRLNearOptimalSampelComplexity_Shi}. By the definition of $\text{KL}\bigg(Ber(p)\parallel Ber(q)\bigg)$, and $p=1-\alpha$, let us consider $\gamma$ be the KL divergence between $Ber(1/\beta)$ and $Ber(p)$, defined as follows
    \begin{align}
        \gamma := \text{KL}\bigg(Ber\Big(\frac{1}{\beta}\Big)\parallel Ber(p)\bigg) &= \frac{1}{\beta}\log\bigg(\frac{1/\beta}{p} \bigg) + \bigg(1 - \frac{1}{\beta}\bigg)\log\bigg(\frac{1- \frac{1}{\beta}}{1-p} \bigg)\nonumber\\
        &= \frac{1}{\beta}\log\bigg(\frac{1}{\beta}\bigg) - \frac{1}{\beta}\log\big(p\big)  +  \bigg(1 - \frac{1}{\beta}\bigg)\log\bigg(1 - \frac{1}{\beta}\bigg) - \bigg(1 - \frac{1}{\beta}\bigg)\log\bigg(1 - p\bigg)\nonumber\\
        &= \frac{1}{\beta}\log\bigg(\frac{1}{\beta}\bigg) - \frac{1}{\beta}\log\big(p\big)  +  \bigg(1 - \frac{1}{\beta}\bigg)\log\bigg(1 - \frac{1}{\beta}\bigg) +\bigg(1 - \frac{1}{\beta}\bigg)\log\bigg(\frac{1}{\alpha}\bigg).
    \end{align}
        We will claim that $\gamma$ satifies the following relation with $\sigma$, which will be proven at the end of this proof:
    \begin{align}
    \label{eq:gamma_range}
        0 \leq \sigma \leq \bigg(1 - \frac{2}{\beta}\bigg)\log\bigg(\frac{1}{\alpha}\bigg) \leq \gamma \leq \bigg(1 - \frac{1}{\beta}\bigg)\log\bigg(\frac{1}{\alpha}\bigg).
    \end{align}
Recall the definition of the worst-case transition kernel defined in \eqref{eq:perturbed_P}, where \begin{align}
\label{eq:tilde_p_ineq_KL}
    \tilde{p} &= \argmin\limits_{p':\text{KL}\Big(Ber(p')\parallel Ber(p)\Big)\leq \sigma} p'V^{\pi,\sigma}_2(s_2) + (1-p')V^{\pi,\sigma}_2(s_3)\nonumber\\
    &\overset{(i)}{\geq} \argmin\limits_{p':\text{KL}\Big(Ber(p')\parallel Ber(p)\Big)\leq \sigma} p'V^{\pi,\sigma}_2(s_2) + (1-p')V^{\pi,\sigma}_2(s_3)\overset{(ii)}{=} \frac{1}{\beta},   
\end{align}
where (i) holds by $\sigma \leq \gamma$ by \eqref{eq:gamma_range} and the last equality (ii) follows from applying \Cref{lem:KL_Ber_ineq} and \eqref{eq:gamma_range} to arrive at
\begin{align}
\label{eq:p_prime_ineq_KL}
    \forall 0 \leq p' < \frac{1}{\beta}, \quad \text{ such that } \quad \text{KL}\bigg(Ber(p')\parallel Ber(p)\bigg) > \text{KL}\bigg(Ber\bigg(\frac{1}{\beta}\bigg)\parallel Ber(p)\bigg) = \gamma.
\end{align}
Now we will proof the claim given in \eqref{eq:gamma_range}. To control $\gamma$, we plug the assumptions that $p \in [1/2,1)$ and $\beta \geq 4$. By these assumptions we arrive at the trivial facts that
\begin{align}
\label{eq:part_i}
    \frac{1}{\beta}\log\bigg(\frac{1}{\beta}\bigg) - \frac{1}{\beta}\log \Big(p\Big) < 0 \quad \text{ and } \quad  \bigg(1-\frac{1}{\beta}\bigg)\log\bigg(1-\frac{1}{\beta}\bigg) < 0.
\end{align}
Using \eqref{eq:part_i}, \eqref{eq:gamma_range} directly leads to
\begin{align}
\label{eq:gamma_part_i}
    \gamma \leq  \bigg(1-\frac{1}{\beta}\bigg)\log\bigg(\frac{1}{\alpha}\bigg).
\end{align}
Similarly, we can observe that
\begin{align}
\label{eq:part_ii}
    -1 \leq \frac{1}{\beta}\log\bigg(\frac{1}{\beta}\bigg) + \bigg(1-\frac{1}{\beta}\bigg)\log\bigg(1-\frac{1}{\beta}\bigg) \leq 0, \quad \text{ and } \quad - \frac{1}{\beta}\log \Big(p\Big)\geq 0.
\end{align}
Using \eqref{eq:part_ii}, \eqref{eq:gamma_range} directly leads to
\begin{align}
\label{eq:gamma_part_ii}
    \gamma \leq  -1+ \bigg(1-\frac{1}{\beta}\bigg)\log\bigg(\frac{1}{\alpha}\bigg) \geq \bigg(1-\frac{2}{\beta}\bigg)\log\bigg(\frac{1}{\alpha}\bigg),
\end{align}
as long as $\log\bigg(\frac{1}{\alpha}\bigg) \geq \beta$ by \eqref{eq:beta_KL}
Using \eqref{eq:gamma_part_i} and \eqref{eq:gamma_part_ii}, it is easy to verify that the choice of uncertainty radius $\sigma$ in \eqref{eq:sigma_KL_range} satisfies the bound in \eqref{eq:gamma_range}.
\end{proof}

%%%%%%%%%%%%%%%%%%%%%% Lower Bound %%%%%%%%%%%%%%%%%%%%%%%%%%

\end{document}